\documentclass[twoside,11pt]{article}
\usepackage{jmlr2e}

\usepackage{epsf}
\usepackage{epsfig}
\usepackage{amsmath}
\usepackage{subfigure}
\usepackage{amssymb}
\usepackage{algorithm}
\usepackage{multirow}
\usepackage{multicol}
\usepackage{array}
\usepackage{enumitem}
\usepackage{url}
\usepackage{algorithm}
\usepackage{algorithmic}
\usepackage{color}
\usepackage[dvipsnames]{xcolor}
\usepackage{graphicx} 
\usepackage{ltxtable} 
\usepackage[pagebackref=true,breaklinks=true,letterpaper=true,colorlinks,bookmarks=yes]{hyperref}

\def\A{{\bf A}}
\def\a{{\bf a}}
\def\B{{\bf B}}
\def\bb{{\bf b}}
\def\C{{\bf C}}
\def\c{{\bf c}}
\def\D{{\bf D}}

\def\G{{\bf G}}

\def\I{{\bf I}}
\def\K{{\bf K}}
\def\k{{\bf k}}
\def\LL{{\bf L}}
\def\M{{\bf M}}

\def\N{{\bf N}}

\def\PP{{\bf P}}
\def\Q{{\bf Q}}

\def\R{{\bf R}}
\def\r{{\bf r}}

\def\U{{\bf U}}
\def\u{{\bf u}}
\def\V{{\bf V}}
\def\v{{\bf v}}
\def\W{{\bf W}}

\def\X{{\bf X}}
\def\x{{\bf x}}
\def\Y{{\bf Y}}

\def\Z{{\bf Z}}

\def\0{{\bf 0}}
\def\1{{\bf 1}}

\def\AM{{\mathcal A}}

\def\FM{{\mathcal F}}

\def\JM{{\mathcal J}}

\def\NM{{\mathcal N}}
\def\OM{{\mathcal O}}

\def\XM{{\mathcal X}}

\def\RB{{\mathbb R}}

\def\PB{{\mathbb P}}

\def\ph{\mbox{\boldmath$\phi$\unboldmath}}

\def\Ph{\mbox{\boldmath$\Phi$\unboldmath}}

\def\Si{\mbox{\boldmath$\Sigma$\unboldmath}}

\def\Lam{\mbox{\boldmath$\Lambda$\unboldmath}}

\def\Pii{\mbox{\boldmath$\Pi$\unboldmath}}

\def\Ph{\mbox{\boldmath$\Phi$\unboldmath}}

\def\argmin{\mathop{\rm argmin}}

\def\range{\mathrm{range}}
\def\nnz{\mathrm{nnz}}

\def\tr{\mathrm{tr}}
\def\rk{\mathrm{rank}}
\def\diag{\mathsf{diag}}

\newtheorem{assumption}{Assumption}
\newtheorem{defn}{Definition}
\newtheorem{rmk}{Remark}

\usepackage{lastpage}
\jmlrheading{20}{2019}{1-\pageref{LastPage}}{9/17}{12/18}{17-517}{Shusen Wang, Alex Gittens, and Michael W. Mahoney}
\ShortHeadings{Scalable Kernel K-Means Clustering with Nystr\"om Approximation}{Wang, Gittens, and Mahoney}

\begin{document}
	
\title{Scalable Kernel K-Means Clustering with Nystr\"om Approximation: Relative-Error Bounds}

\author{\name Shusen Wang \email shusen.wang@stevens.edu\\ 
	\addr Department of Computer Science \\ 
	Stevens Institute of Technology\\ 
	Hoboken, NJ 07030, USA 
	\AND 
	\name Alex Gittens \email gittea@rpi.edu\\ 
	\addr Computer Science Department\\ 
	Rensselaer Polytechnic Institute\\ 
	Troy, NY 12180, USA
	\AND 
	\name Michael W.\ Mahoney \email mmahoney@stat.berkeley.edu \\ 
	\addr International Computer Science Institute and Department of Statistics\\ 
	University of California at Berkeley\\ 
	Berkeley, CA 94720, USA
	}

\editor{Sanjiv Kumar}

\maketitle

\begin{abstract}%
	Kernel $k$-means clustering can correctly identify and extract a far more varied collection of cluster structures than the linear $k$-means clustering algorithm.  However, kernel $k$-means clustering is computationally expensive when the non-linear feature map is high-dimensional and there are many input points.
	
	Kernel approximation, e.g., the Nystr\"om method, has been applied in previous works to approximately solve kernel learning problems when both of the above conditions are present. This work analyzes the application of this paradigm to kernel $k$-means clustering, and shows that applying the linear $k$-means clustering algorithm to $\frac{k}{\epsilon} (1 + o(1))$ features constructed using a so-called rank-restricted Nystr\"om approximation results in cluster assignments that satisfy a $1 + \epsilon$ approximation ratio in terms of the kernel $k$-means cost function, relative to the guarantee provided by the same algorithm without the use of the Nystr\"om method. As part of the analysis, this work establishes a novel $1 + \epsilon$ relative-error trace norm guarantee for low-rank approximation using the rank-restricted Nystr\"om approximation.
	
	Empirical evaluations on the $8.1$ million instance MNIST8M dataset demonstrate the scalability and usefulness of kernel $k$-means clustering with Nystr\"om approximation. This work argues that spectral clustering using Nystr\"om approximation---a popular and computationally efficient, but theoretically unsound approach to non-linear clustering---should be replaced with the efficient and theoretically sound combination of kernel $k$-means clustering with Nystr\"om approximation. The superior performance of the latter approach is empirically verified. 
\end{abstract}

\begin{keywords}
	kernel $k$-means clustering, the Nystr\"om method, randomized linear algebra
\end{keywords}

\section{Introduction} \label{sec:intro}

Cluster analysis divides a data set into several groups using information found
only in the data points. Clustering can be used in an exploratory manner to
discover meaningful groupings within a data set, or it can serve as the
starting point for more advanced analyses. As such, applications of clustering
abound in machine learning and data analysis, including,  \emph{inter alia}:
genetic expression analysis~\citep{sharan2002cluster}, market
segmentation~\citep{chaturvedi1997feature}, social network
analysis~\citep{handcock2007model}, image
segmentation~\citep{haralick1985segmentation}, anomaly
detection~\citep{chandola2009anomaly}, collaborative
filtering~\citep{ungar1998clustering}, and fast approximate learning of
non-linear models~\citep{si2014memory}.

Linear $k$-means clustering is a standard and well-regarded approach to cluster
analysis that partitions input vectors $\{\a_1, \ldots, \a_n\} \subset \RB^d$
into $k$ clusters, in an unsupervised manner, by assigning each vector to the
cluster with the nearest centroid. Formally, linear $k$-means clustering seeks
to partition the set $[n] = \{1, \ldots, n\}$ into $k$ disjoint sets $\JM_1,
\ldots, \JM_k$ by solving
\begin{small}
	\begin{eqnarray} \label{eq:kmeans}
	\argmin_{\JM_1, \ldots, \JM_k} \;
	\frac{1}{n} \sum_{i=1}^k \sum_{j\in \JM_i} 
	\bigg\| \a_j \: - \: \frac{1}{|\JM_i |} \sum_{l \in \JM_i} \a_l \bigg\|_2^2.
	\end{eqnarray}
\end{small}%
Lloyd's algorithm~\citep{lloyd1982least} is a standard approach for finding local minimizers of \eqref{eq:kmeans}, and is a staple in data mining and machine learning.

Despite its popularity, linear $k$-means clustering is not a universal solution
to all clustering problems. In particular, linear $k$-means clustering strongly biases
the recovered clusters towards isotropy and sphericity.  Applied to the data in
Figure~\ref{fig:kmeans_example:1}, Lloyd's algorithm is perfectly capable of
partitioning the data into three clusters which fit these assumptions.
However, the data in Figure~\ref{fig:kmeans_example:2} do not fit these
assumptions: the clusters are ring-shaped and have coincident centers, so
minimizing the linear $k$-means objective does not recover these clusters.

\begin{figure}[!ht]
	\begin{center}
		\centering
		\subfigure[\,]{\includegraphics[width=0.26\textwidth]{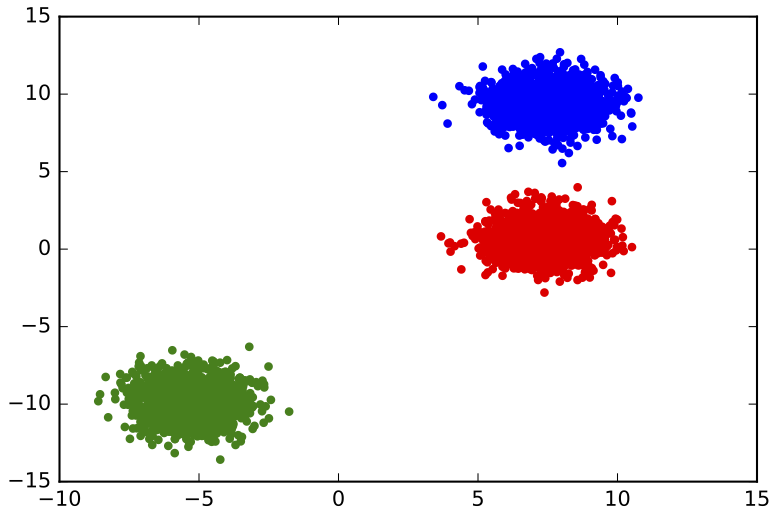} \label{fig:kmeans_example:1}}
		\subfigure[\,]{\includegraphics[width=0.26\textwidth]{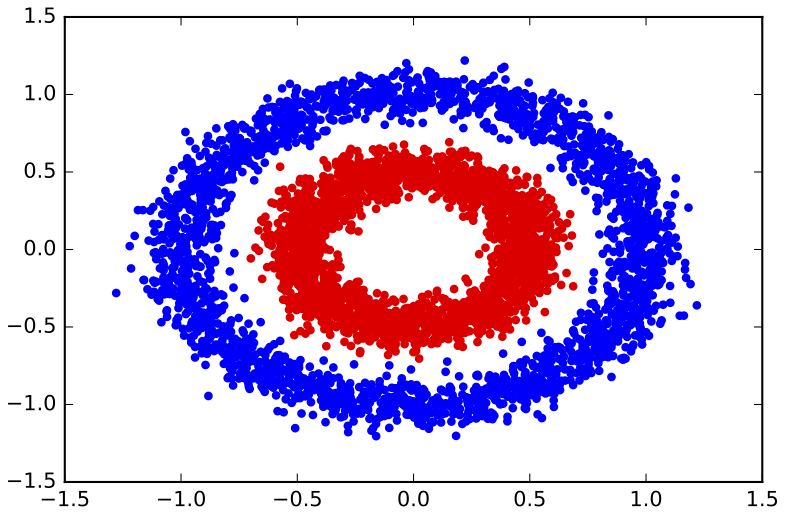} \label{fig:kmeans_example:2}}
		\subfigure[\,]{\includegraphics[width=0.26\textwidth]{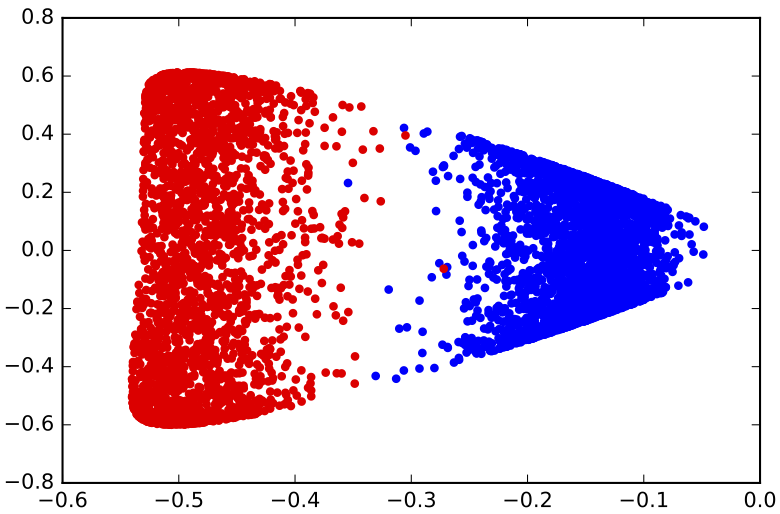} \label{fig:kmeans_example:3}}
	\end{center}
	\vspace{-3mm}
	\caption{Figures~\ref{fig:kmeans_example:1} and~\ref{fig:kmeans_example:2} show two sets
		of two-dimensional points,
		with different colors indicating different clusters.
		Figure~\ref{fig:kmeans_example:1} is separable using linear $k$-means;
		Figures~\ref{fig:kmeans_example:2} is inseparable using linear $k$-means.
		Figure~\ref{fig:kmeans_example:3} shows the first two dimensions of the feature vectors 
		$\k_1, \ldots, \k_n \in \RB^n$ derived from a kernel.}
	\label{fig:kmeans_example}
	\vspace{-3mm}
\end{figure}

To extend the scope of $k$-means clustering to include anistotropic,
non-spherical clusters such as those depicted in
Figure~\ref{fig:kmeans_example:2}, \citet{scholkopf1998nonlinear} proposed to
perform linear $k$-means clustering in a nonlinear feature space instead of the input space.
After choosing a feature function $\ph : \RB^d \mapsto \FM$ to map the input
vectors non-linearly into feature vectors, they propose minimizing
the objective function
\begin{small}
	\begin{eqnarray} \label{eq:kernel_kmeans}
	\argmin_{\JM_1, \ldots, \JM_k} \;
	\frac{1}{n}\sum_{i=1}^k \sum_{j\in \JM_i} 
	\bigg\| \ph (\a_j ) \: - \: \frac{1}{|\JM_i |} \sum_{l \in \JM_i} \ph (\a_l ) \bigg\|_2^2 ,
	\end{eqnarray}
\end{small}%
where $\{\JM_1, \ldots, \JM_k\}$ denotes a $k$-partition of $[n]$.
The ``kernel trick'' enables us to minimize this objective without explicitly
computing the potentially high-dimensional features, as inner products in
feature space can be computed implicitly by evaluating the kernel function
\[
\kappa \big(\a_i, \, \a_j \big) \; = \; \big\langle \ph (\a_i ) , \, \ph (\a_j ) \big\rangle.
\]
Thus the information required to solve the kernel $k$-means problem~\eqref{eq:kernel_kmeans}, 
is present in the kernel matrix
$\K  = [ \kappa (\a_i, \a_j ) ]_{ij} \in \RB^{n\times n}$.

Let $\K = \V \Lam \V^T$ be the full eigenvalue decomposition (EVD) of the kernel matrix
and $\k_1, \ldots, \k_n \in \RB^{n}$ be the rows of $\V \Lam^{1/2} \in \RB^{n\times n}$.
It can be shown (see Appendix~\ref{sec:proof:kernel})
that the solution of~\eqref{eq:kernel_kmeans} is identical to the solution of
\begin{small}
	\begin{eqnarray}
	\label{eqn:kernel_kmeans_equivalent}
	\argmin_{\JM_1, \ldots, \JM_k} \;
	\frac{1}{n}\sum_{i=1}^k \sum_{j\in \JM_i} 
	\bigg\| \k_j \: - \: \frac{1}{|\JM_i |} \sum_{l \in \JM_i} \k_l \bigg\|_2^2.
	\end{eqnarray}
\end{small}%
To demonstrate the power of kernel $k$-means clustering, consider the dataset in Figure~\ref{fig:kmeans_example:2}.
We use the Gaussian RBF kernel 
\[
\kappa (\a, \a' ) 
= \exp \big(- \tfrac{1}{2\sigma^2} \|\a - \a' \|_2^2 \big) 
\]
with $\sigma = 0.3$, and form the corresponding kernel matrix of the data in Figure~\ref{fig:kmeans_example:2}.
Figure~\ref{fig:kmeans_example:3} scatterplots the first two coordinates of the feature vectors $\k_1 , \ldots , \k_n$.
Clearly, the first coordinate of the feature vectors already separates the two classes well, so 
$k$-means clustering using the non-linear features $\k_1 , \cdots , \k_n$ has a better chance of separating the two classes.

Although it is more generally applicable than linear $k$-means clustering, 
kernel $k$-means clustering is computationally expensive. As a baseline, we consider the 
cost of optimizing~\eqref{eqn:kernel_kmeans_equivalent}.
The formation of the kernel matrix $\K$ given the 
input vectors $\a_1, \ldots, \a_n \in \RB^d$ costs $\OM (n^2 d)$ time.
The objective in~\eqref{eqn:kernel_kmeans_equivalent} can then be (approximately) minimized using Lloyd's algorithm at a cost of $\OM(n^2 k)$ time per iteration.
This requires the $n$-dimensional non-linear feature vectors obtained from the full EVD of $\K$; computing these feature vectors takes $\OM (n^3)$ time, because $\K$ is, in general, full-rank.
Thus, approximately solving the kernel $k$-means clustering problem by optimizing~\eqref{eqn:kernel_kmeans_equivalent} costs $\OM(n^3 + n^2 d + T n^2 k)$ time, where $T$ is the number of iterations of Lloyd's algorithm. 

Kernel approximation techniques, including the Nystr\"om method \citep{nystrom1930praktische,williams2001using,gittens2013revisiting}
and random feature maps \citep{rahimi2007random},
have been applied to decrease the cost of solving kernelized machine learning problems: the idea is to replace 
$\K$ with a low-rank approximation, which allows for more efficient computations.
\citet{chitta2011approximate,chitta2012efficient} proposed to apply 
kernel approximations to efficiently approximate kernel $k$-means clustering.
Although kernel approximations mitigate the computational challenges of kernel $k$-means clustering,
the aforementioned works do not provide guarantees on the clustering performance: 
how accurate must the low-rank approximation of $\K$ be to ensure near optimality of the approximate clustering obtained via this method?

We propose a provable approximate solution to the kernel $k$-means problem based on the Nystr\"om approximation.
Our method has three steps:
first, extract $c$ ($c \ll n$) features using the Nystr\"om method;
second, reduce the features to $s$ dimensions ($k \leq s < c$) using the truncated SVD;\footnote{This is why
	our variant is called the rank-restricted Nystr\"om approximation. 
	The rank-restriction serves two purposes.
	First, although we do not know whether the rank-restriction is necessary for the $1+\epsilon$ bound, we are unable to establish the bound without it.
	Second, the rank-restriction makes the third step, linear $k$-means clustering, much less costly.
	For the computational benefit, previous works \citep{boutisdis2009unsupervised,boutsidis2010random,boutsidis2015randomized,cohen2015dimensionality,feldman2013turning} 
	have considered dimensionality reduction for linear $k$-means clustering.\cite{}
}
third, apply any off-the-shelf linear $k$-means clustering algorithm upon the $s$-dimensional features
to obtain the final clusters. 
The total time complexity of the first two steps is $\OM (ndc + nc^2)$.
The time complexity of the third step depends on the specific linear $k$-means algorithm;
for example, using Lloyd's algorithm, the per-iteration complexity is $\OM (n s k)$, and the number of iterations may depend on $s$.\footnote{Without the rank-restriction, the per-iteration cost would be $\OM (n c k)$, and the number of iterations may depend on $c$.}

Our method comes with a strong approximation ratio guarantee.
Suppose we set $s = {k}/{\epsilon}$ and $c = \tilde{\OM} (\mu s / \epsilon )$ for any $\epsilon \in (0, 1)$,
where $\mu \in [1, \tfrac{n}{s}]$ is the coherence parameter of the dominant $s$-dimensional singular space of the kernel matrix $\K$.
Also suppose the standard kernel $k$-means and our approximate method use
the same linear $k$-means clustering algorithm, e.g., Lloyd's algorithm or some other algorithm that comes with different provable approximation guarantees.
As guaranteed by Theorem~\ref{thm:kkmeans:nystrom2},
when the quality of the clustering is measured by the cost function defined in \eqref{eq:kernel_kmeans}, with probability at least $0.9$,
our algorithm returns a clustering that is at most $ 1 + \OM(\epsilon)$ times worse than the standard kernel $k$-means clustering.
Our theory makes explicit the trade-off between accuracy and computation:
larger $s$ and $c$ lead to high accuracy and also high computational cost.

Spectral clustering \citep{shi2000normalized,ng2002spectral}
is a popular alternative to kernel $k$-means clustering
that can also partition non-linearly separable data such as those in Figure~\ref{fig:kmeans_example:2}.
Unfortunately, because it requires computing an $n\times n$ affinity matrix 
and the top $k$ eigenvectors of the corresponding graph Laplacian,
spectral clustering is inefficient for large $n$.
\citet{fowlkes2004spectral} applied the
Nystr\"om approximation to increase the scalability of spectral clustering.  Since then, spectral clustering with Nystr\"om approximation
has been used in many works, e.g.,
\citep{arikan2006compression,berg2004s,chen2011parallel,wang2015towards,weiss2009spectral,zhang2010clustered}.
Despite its popularity in practice, this approach does not come with guarantees on the approximation ratio for the obtained clustering.
Our algorithm, which combines kernel $k$-means with Nystr\"om approximation, is an
equally computationally efficient alternative that comes with strong bounds on the approximation ratio, and 
can be used wherever spectral clustering is applied.

\subsection{Contributions}

Using tools developed
in~\citep{boutsidis2015randomized,cohen2015dimensionality,feldman2013turning},
we rigorously analyze the performance of approximate kernel
$k$-means clustering with the Nystr\"om approximation, and show that a rank-$\tfrac{k}{\epsilon}$
Nystr\"om approximation delivers a $1 + \OM(\epsilon)$ approximation ratio guarantee, relative to the guarantee provided by the same algorithm without the use of the Nystr\"om method.

As part of the analysis of kernel $k$-means with Nystr\"om approximation,
we establish the first relative-error bound for rank-restricted Nystr\"om approximation,\footnote{Similar 
	relative-error bounds were independently developed by contemporaneous work of~\citet{tropp2017fixed}, 
	in service of the analysis of a novel streaming algorithm for fixed-rank approximation of positive semidefinite
	matrices. Preliminary versions of this work and theirs were simultaneously submitted to arXiv in June 2017.}
which has independent interest.

Kernel $k$-means clustering and spectral clustering are competing solutions to 
the nonlinear clustering problem, neither of which scales well with $n$.
\citet{fowlkes2004spectral} introduced the use of Nystr\"om approximations
to make spectral clustering scalable;
this approach has become popular in machine learning.
We identify fundamental mathematical problems with this heuristic.
These concerns and an empirical comparison establish that our
proposed combination of kernel $k$-means with rank-restricted Nystr\"om approximation is a 
theoretically well-founded and empirically competitive alternative to 
spectral clustering with Nystr\"om approximation.

Finally, we demonstrate the scalability of this approach 
by measuring the performance of an Apache Spark
implementation of a distributed version of our approximate kernel $k$-means
clustering algorithm using the MNIST8M data set, which has $8.1$ million
instances and $10$ classes. 

\subsection{Relation to Prior Work}

The key to our analysis of the proposed approximate kernel $k$-means clustering algorithm is 
a novel relative-error trace norm bound for a {\it rank-restricted} Nystr\"om approximation.
We restrict the rank of the Nystr\"om approximation in a non-standard manner
(see Remark~\ref{remark:rank_restrict}).
Our relative-error trace norm bound is not a simple consequence
of the existing bounds for non-rank-restricted Nystr\"om approximation
such as the ones provided by~\citet{gittens2013revisiting}. 
The relative-error bound which we provide for the
rank-restricted Nystr\"om approximation is potentially useful in other applications
involving the Nystr\"om method.

The projection-cost preservation (PCP) property
\citep{cohen2015dimensionality,feldman2013turning} is an important tool for
analyzing approximate linear $k$-means clustering.
We apply our novel relative-error trace norm bound as well as 
existing tools in \citep{cohen2015dimensionality}
to prove that the rank-restricted Nystr\"om approximation enjoys 
the PCP property. We do not rule out the possibility
that the non-rank-restricted (rank-$c$) Nystr\"om
approximation satisfies the PCP property and/or also enjoys a $1 +\epsilon$
approximation ratio guarantee when applied to kernel $k$-means clustering. However, the
cost of the linear $k$-means clustering step in the algorithm is proportional to the dimensionality of
the feature vectors, so the rank-restricted Nystr\"om approximation, which
produces $s$-dimensional feature vectors, where $s < c$, is more
computationally desirable.

\citet{musco2016recursive} similarly establishes a $1 + \epsilon$ approximation ratio for the kernel $k$-means
objective when a Nystr\"om approximation is used in place of the full kernel matrix. 
Specifically, \citet{musco2016recursive} shows that when
$c = \OM (\tfrac{k}{\epsilon} \log \tfrac{k}{\epsilon})$ columns of $\K$ are sampled using
{\it ridge leverage score (RLS) sampling} \citep{alaoui2015fast,cohen2017input,musco2016recursive}
and are used to form a Nystr\"om approximation, then
applying linear $k$-means clustering to the $c$-dimensional Nystr\"om features returns
a clustering that has objective value at most $1 + \epsilon$ times as large as the
objective value of the best clustering. 
Our theory is independent of that in \citet{musco2016recursive}, and differs in that
(1) \citet{musco2016recursive} applies specifically to Nystr\"om approximations formed using RLS sampling,
whereas our guarantees apply to any sketching method that satisfies the ``subspace embedding''
and ``matrix multiplication'' properties (see Lemma~\ref{lem:property} for definitions of these two properties);
(2) \citet{musco2016recursive} establishes a $1 + \epsilon$ approximation ratio for the non-rank-restricted RLS-Nystr\"om approximation,
whereas we establish a $1 + \epsilon$ approximation ratio for the (more computationally efficient) rank-restricted Nystr\"om approximation.

\subsection{Paper Organization}

In Section~\ref{sec:notation}, we start with a definition of the notation used throughout this paper as well as a background on matrix sketching methods.
Then, in Section~\ref{sec:main}, we present our main theoretical results:
Section~\ref{sec:main:nystrom} presents an improved relative-error rank-restricted Nystr\"om approximation;
Section~\ref{sec:main:kkmeans} presents the main theoretical results on kernel $k$-means with Nystr\"om approximation;
and Section~\ref{sec:main:kpca} studies kernel $k$-means with kernel principal component analysis.
Section~\ref{sec:sc} discusses and evaluates the theoretical and empirical merits of kernel $k$-means clustering versus spectral clustering, when each is approximated using Nystr\"om approximation.
Section~\ref{sec:medium} empirically compares the Nystr\"om method and the random feature maps for the kernel $k$-means clustering on medium-scale data.
Section~\ref{sxn:implementation} presents a large-scale distributed implementation in Apache Spark and its empirical evaluation on a data set with $8.1$ million points.
Section~\ref{snx:conclusion} provides a brief conclusion.
Proofs are provided in the Appendices.

\section{Notation}
\label{sec:notation}

This section defines the notation used throughout this paper.
A set of commonly used parameters is summarized in Table~\ref{tab:notation}.

\begin{table}[t]\setlength{\tabcolsep}{0.3pt}
	\def\arraystretch{1.1}
	\caption{Commonly used parameters. It holds that $k \leq s \leq c \leq n$.}
	\label{tab:notation}
	\begin{center}
		\begin{small}
			\begin{tabular}{c l}
				\hline
				~~~{\bf Notation}~~~&~~~{\bf Definition}~~~\\
				\hline
				~~~$n$~~~ & ~~~number of samples~~~\\
				~~~$d$~~~ & ~~~number of features (attributes)~~~\\
				~~~$k$~~~ & ~~~number of clusters~~~\\
				~~~$s$~~~ & ~~~target rank of the Nystr\"om approximation~~~\\
				~~~$c$~~~ & ~~~sketch size of the Nystr\"om approximation~~~ \\
				\hline
			\end{tabular}
		\end{small}
	\end{center}
\end{table}

\paragraph{Matrices and vectors.}
We take $\I_n$ to be the $n\times n$ identity matrix, 
$\0$ to be a vector or matrix of all zeros of the appropriate size,
and $\1_n$ to be the $n$-dimensional vector of all ones.
\paragraph{Sets.}
The set $\{1, 2, \cdots, n\}$ is written as $[n]$.
We call $\{\JM_1 , \cdots , \JM_k \}$ a $k$-partition of $[n]$ 
if $\JM_1 \cup \cdots \cup \JM_k = [n]$ and $\JM_p \cap \JM_q = \emptyset$ when $p \neq q$. 
Let $|\JM |$ denote the cardinality of the set $\JM$.
\paragraph{Singular value decomposition (SVD).}
Let $\A \in \RB^{n\times d}$ and $\rho = \rk (\A)$.
A (compact) singular value decomposition (SVD) is defined by 
\begin{eqnarray*}
	\textstyle{\A \; = \; \U \Si \V^T
		\; = \; \sum_{i=1}^\rho  \sigma_i (\A) \u_i \v_i^T ,}
\end{eqnarray*}
where 
$\U$, $\Si$, $\V$ are an $n\times \rho$ column-orthogonal matrix,
a $\rho\times \rho$ diagonal matrix with nonnegative entries, and a $d\times \rho$ column-orthogonal matrix, respectively.
If $\A$ is symmetric positive semi-definite (SPSD), then $\U = \V$, and this decomposition is also called the (reduced) eigenvalue decomposition (EVD). By convention,
we take $\sigma_1 (\A) \geq \cdots \geq \sigma_\rho (\A) $.
\paragraph{Truncated SVD.}
The matrix $\A_s = \sum_{i=1}^s \sigma_i (\A) \u_i \v_i^T$ is a rank-$s$ truncated SVD of $\A$, and is an optimal rank-$s$ approximation to $\A$ when
the approximation error is measured in a unitarily invariant norm.
\paragraph{The Moore-Penrose inverse} of $\A$ is defined by
$\A^\dag = \V \Si^{-1} \U^T$.
\paragraph{Leverage score and coherence.}
Let $\U \in \RB^{n\times \rho}$ be defined in the above
and $\u_i$ be the $i$-th row of $\U$.
The row leverage scores of $\A$ are $l_i = \|\u_{i}\|_2^2$ for $i \in [n]$.
The row coherence of $\A$ is $\mu (\A) = \frac{n}{\rho} \max_{i } \|\u_{i}\|_2^2$. 
The leverage scores for a matrix $\A$ can be computed exactly in the time it takes to compute the matrix $\U$; and
the leverage scores can be approximated (in theory~\citep{drineas2012fast} and in practice~\citep{gittens2013revisiting}) in roughly the time it takes to apply a random projection matrix to the matrix $\A$.
\paragraph{Matrix norms.}
We use three matrix norms in this paper:
\begin{small}
	\begin{align*}
	\textrm{Frobenius Norm: } \quad &
	\|\A \|_F \; = \; \sqrt{ {\textstyle \sum_{i,j}} a_{ij}^2 }
	\; = \; \sqrt{ {\textstyle \sum_{i} } \sigma_i^2 (\A )  };\\
	\textrm{Spectral Norm: } \quad &
	\|\A \|_2 \; = \; \max_{\|\x\|_2=1} \big\| \A \x \big\|_2
	\; = \; \sigma_1 (\A ) ;\\
	\textrm{Trace Norm: } \quad &
	\|\A \|_* \; = \; {\textstyle \sum_{i} } \sigma_i (\A )  .
	\end{align*}
\end{small}%
Any square matrix satisfies $\tr (\A ) \leq \|\A \|_*$. If additionally $\A$ is SPSD, 
then $\tr (\A ) = \|\A \|_*$.

\paragraph{Matrix sketching}
Here, we briefly review matrix sketching methods that are commonly used within randomized linear algebra (RLA)~\citep{mahoney2011ramdomized}.

Given a matrix $\A \in \RB^{m \times n}$, we call $\C = \A \PP \in \RB^{m \times c}$ (typically $c \ll n$) a \emph{sketch} of $\A$ and $\PP \in \RB^{n\times c}$ a \emph{sketching matrix}.
Within RLA, sketching has emerged as a powerful primitive, where one is primarily interested in using random projections and random sampling to construct randomzed sketches~\citep{mahoney2011ramdomized,drineas2016randnla}.
In particular, sketching is useful as it allows large matrices to be replaced with smaller matrices which are more amenable to efficient computation, but provably retain almost optimal accuracy in many computations~\citep{mahoney2011ramdomized,woodruff2014sketching}.
The columns of $\C$ typically comprise a rescaled subset of the columns of $\A$, or random linear combinations of the columns of $\A$;
the former type of sketching is called \emph{column selection} or \emph{random sampling}, and the latter is referred to as \emph{random projection}.

\emph{Column selection} forms $\C \in \RB^{m \times c}$ using a randomly sampled and rescaled subset of the columns of $\A \in \RB^{m \times n}$.
Let $p_1 , \cdots , p_n \in (0, 1)$ be the sampling probabilities associated with the columns of $\A$ (so that, in particular,
$\sum_{i=1}^n p_i = 1$). The columns of the sketch are selected identically and independently as follows:
each column of $\C$ is randomly sampled from the columns of
$\A$ according to the sampling probabilities and rescaled by $\frac{1}{\sqrt{c p_i}},$ where $i$ is the index of the column of $\A$ that was selected.
In our matrix multiplication formulation for sketching, column selection corresponds to a sketching matrix $\PP \in \RB^{n \times c}$
that has exactly one non-zero entry in each column, whose position and magnitude correspond to the index of the column
selected from $\A$. \emph{Uniform sampling} is column sampling with $p_1 = \cdots = p_n = \frac{1}{n}$, and \emph{leverage score sampling}
takes $p_i = \tfrac{l_i}{\sum_j l_j}$ for $i\in[n]$, where $l_i$ is the $i$-th leverage score of some matrix (typically $\A$, $\A_k$, or
an randomized approximation thereto)~\citep{drineas2012fast}.

\emph{Gaussian projection} is a type of random projection where the sketching matrix is taken to be $\PP = \frac{1}{\sqrt{c}} \G \in \RB^{n\times c}$; here the
entries of $\G$ are i.i.d.\ $\NM (0, 1)$ random variables.
Gaussian projection is inefficient relative to column sampling:
the formation of a Gaussian sketch of a dense $m \times n$ matrix requires $\OM (mnc)$ time.
The \emph{Subsampled Randomized Hadamard Transform (SRHT)} is a more efficient alternative
that enjoys similar properties to the Gaussian projection~\citep{drineas2011faster,lu2013faster,tropp2011improved},
and can be applied to a dense $m\times n$ matrix in only $\OM (mn \log c)$ time.
The \emph{CountSketch} is even more efficient: it can be applied to any matrix $\A$ in $\OM (\nnz (\A))$ time~\citep{clarkson2013low,meng2013low,nelson2013osnap}, where $\nnz (\cdot )$ denotes the number of nonzero entries in a matrix.

\section{Our Main Results: Improved SPSD Matrix Approximation and Kernel $k$-means Approximation}
\label{sec:main}

In this section, we present our main theoretical results.
We start, in Section~\ref{sec:main:nystrom}, by presenting Theorem~\ref{thm:nystrom}, a novel result on SPSD matrix approximation with the rank-restricted Nystr\"om method.
This result is of independent interest, and so we present it in detail, but in this paper we will use it to establish our main result.
Then, in Section~\ref{sec:main:kkmeans}, we present Theorem~\ref{thm:kkmeans:nystrom2}, which is our main result for approximate kernel $k$-means with the Nystr\"om approximation.
In Section~\ref{sec:main:kpca}, we establish novel guarantees on kernel $k$-means with dimensionality reduction.

\subsection{The Nystr\"om Method}
\label{sec:main:nystrom}

The Nystr\"om method~\citep{nystrom1930praktische} is the most popular kernel approximation method in the machine learning community.
Let $\K \in \RB^{n\times n}$ be an SPSD matrix and $\PP \in \RB^{n\times c}$ be a sketching matrix.
The Nystr\"om method approximates $\K$ with $\C \W^\dag \C^T$, 
where $\C = \K \PP$ and $\W = \PP^T \K \PP$.
The Nystr\"om method was introduced to the machine learning community by~\citet{williams2001using};
since then, numerous works have studied its theoretical properties, e.g.,~\citep{drineas2005nystrom,gittens2013revisiting,jin2012improved,kumar2012sampling,wang2013improving,yang2012nystrom}.

Empirical results in~\citep{gittens2013revisiting,wang2015towards,yang2012nystrom}
demonstrated that the accuracy of the Nystr\"om method significantly increases
when the spectrum of $\K$ decays fast. This suggests that the Nystr\"om approximation
captures the dominant eigenspaces of $\K$, and that error bounds comparing the
accuracy of the Nystr\"om approximation of $\K$ to that of the best rank-$s$
approximation $\K_s$ (for $s < c$) would provide a meaningful measure of the
performance of the Nystr\"om kernel approximation.
\citet{gittens2013revisiting}~established the first relative-error bounds
showing that for sufficiently large $c$,
the trace norm error $\|\K - \C \W^\dag \C^T\|_*$ is
comparable to $\|\K - \K_s\|_*$.
Such results quantify the benefits of spectral decay
to the performance of the Nystr\"om method, and are sufficient to analyze the performance of 
Nystr\"om approximations in applications such as kernel ridge regression~\citep{alaoui2015fast,bach2013sharp} and kernel support vector machines~\citep{cortes2010impact}.

However, \citet{gittens2013revisiting} did not analyze the performance of rank-restricted Nystr\"om approximations;
they compared the approximation accuracies of the rank-$c$ matrix $\C \W^\dag \C^T$ and the rank-$s$ matrix $\K_s$ (recall $s < c$).
In our application to approximate kernel $k$-means clustering, 
it is the rank-$s$ matrix $(\C \W^{\dagger} \C^T )_s$ that is of relevance.
Given $\C$ and $\W$, the truncated SVD 
$\big((\W^\dag)^{1/2} \C^T \big)_s = \tilde{\U}_s \tilde{\Si}_s \tilde{\V}_s^T$ can be found using $\OM (n c^2)$ time.
Then the rank-$s$ Nystr\"om approximation can be written as
\begin{small}
	\begin{eqnarray}
	\label{eqn:rrnystrom}
	(\C \W^\dag \C^T )_s
	\; = \; \big(\C (\W^\dag)^{1/2} \big)_s \big(\C (\W^\dag)^{1/2} \big)_s^T 
	\; = \; \tilde{\V}_s \tilde{\Si}_s^2 \tilde{\V}_s^T.
	\end{eqnarray}
\end{small}%
Theorem~\ref{thm:nystrom} provides a relative-error trace norm approximation guarantee for the sketch \eqref{eqn:rrnystrom} and is novel; a proof is provided in Appendix~\ref{sec:proof:nystrom}.

\begin{theorem} [Relative-Error Rank-Restricted Nystr\"om Approximation] \label{thm:nystrom}
	Let $\K \in \RB^{n\times n}$ be an SPSD matrix,
	$s$ be the target rank, and $\epsilon \in (0, 1)$ be an error parameter.
	Let $\PP \in \RB^{n\times c}$ be a sketching matrix corresponding to one of the sketching methods listed in Table~\ref{tab:sketch}.
	Let $\C = \K \PP$ and $\W = \PP^T \K \PP$.
	Then
	\begin{eqnarray*}
		\big\| \K - ( \C \W^\dag \C^T )_s \big\|_*
		& \leq & (1+ \epsilon ) \, \big\| \K - \K_s \big\|_*
	\end{eqnarray*}
	holds with probability at least $0.9$.
	In addition, there exists an $n\times s$ column orthogonal matrix $\Q$
	such that $(\C \W^\dag \C^T)_s = \K^{1/2} \Q \Q^T \K^{1/2}$.
\end{theorem}

\begin{table}[t]\setlength{\tabcolsep}{0.3pt}
	\caption{Sufficient sketch sizes for several sketching methods used to construct Nystr\"om approximations to 
		the matrix $\K$. Let $\K_s = \V_s \Si_s \V_s^T$ be the truncated SVD of $\K$, then $\mu$ denotes the row coherence of $\V_s \in \RB^{n\times s}$ and 
		the leverage score sampling is done using the row leverage scores of $\V_s$.
	}
	\label{tab:sketch}
	\vspace{-4mm}
	\begin{center}
		\begin{small}
			\begin{tabular}{c c c}
				\hline
				~~~{\bf sketching}~~~&~~~{\bf sketch size ($c$)}~~~
				&~~~{\bf time complexity ($T$)}~~\\
				\hline
				~~~uniform sampling~~~
				& ~~~$\OM (\mu {s} / {\epsilon} + \mu s \log  s)$~~~ 
				& ~~~$\OM (nc)$~~~ \\
				~~~leverage sampling~~~
				& ~~~$\OM ({s} / {\epsilon} + s \log s) $~~~ 
				& ~~~$\tilde\OM (n^2 s)$\\
				~~~Gaussian projection~~~
				& ~~~$\OM ({s} / {\epsilon})  $~~~
				& ~~~$\OM (n^2 c)$ \\
				~~~SRHT~~~
				& ~~~$\OM ((s + \log n) (\epsilon^{-1} + \log s))  $~~~ 
				& ~~~$\OM (n^2 \log c)$ \\
				~~~CountSketch~~~
				& ~~~$\OM (s/\epsilon + s^2)  $~~~
				& ~~~$\OM (n^2)$ \\
				\hline
			\end{tabular}
		\end{small}
	\end{center}
	\vspace{-5mm}
\end{table}

\begin{rmk}[Rank Restrictions] \label{remark:rank_restrict}
	The traditional rank-restricted Nystr\"om approximation, $\C (\W_s)^\dagger \C^T$,
	\citep{drineas2005nystrom,fowlkes2004spectral,gittens2013revisiting,li2015large}
	is not known to satisfy a relative-error bound of the form guaranteed in Theorem~\ref{thm:nystrom}.
	\citet{pourkamali-anaraki2016randomized}~pointed out the drawbacks of the traditional rank-restricted
	Nystr\"om approximation and proposed the use of the rank-restricted Nystr\"om approximation $(\C \W^\dagger \C^T)_s$ in applications
	requiring kernel approximations, but provided
	only empirical evidence of its performance.
	This work provides guarantees on the approximation error of 
	the rank-restricted Nystr\"om approximation $(\C \W^\dagger \C^T)_s$, and applies this 
	approximation to the kernel $k$-means clustering problem.
	The contemporaneous work~\citet{tropp2017fixed} provides similar guarantees on the approximation error of 
	$(\C \W^\dagger \C^T)_s$, and uses this Nystr\"om approximation as the basis of a streaming algorithm for 
	fixed-rank approximation of positive-semidefinite matrices.
\end{rmk}


\subsection{Main Result for Approximate Kernel $k$-means}
\label{sec:main:kkmeans}

In this section we establish the approximation ratio guarantees for the objective function of kernel $k$-means clustering.
We first define $\gamma$-approximate $k$-means algorithms (where $\gamma \geq 1$),
then present our main result in Theorem~\ref{thm:kkmeans:nystrom2}.

Let $\A$ be a matrix with $n$ rows $\a_1 , \cdots , \a_n$.
The objective function for linear $k$-means clustering over the rows of $\A$ is
\begin{small}
	\begin{eqnarray*}
		f \big(\JM_1 , \cdots , \JM_k \, ; \, \A \big)
		\; = \; \frac{1}{n} \sum_{i=1}^k \sum_{j\in \JM_i} 
		\bigg\| \a_j \: - \: \frac{1}{|\JM_i |} \sum_{l \in \JM_i} \a_l \bigg\|_2^2 .
	\end{eqnarray*}
\end{small}%
The minimization of $f$ w.r.t.\ the $k$-partition $\{\JM_1 , \cdots , \JM_k\}$ is NP-hard \citep{garey1982complexity,aloise2009np,dasgupta2009random,mahajan2009planar,awasthi2015hardness}, 
but approximate solutions can be obtained in polynomial time.
$\gamma$-approximate algorithms capture one useful notion of approximation.

\begin{defn}[$\gamma$-Approximate Algorithms] \label{def:gamma0}
	A linear $k$-means clustering algorithm $\AM_\gamma$ takes as input a matrix $\Z$ with $n$ rows and 
	outputs $\{ {\JM}_1', \cdots , {\JM}_k' \}$.
	We call $\AM_\gamma$ a {$\gamma$-approximate algorithm} if, for any such matrix $\Z$,
	\begin{small}
		\begin{eqnarray*}
			f \big(\JM_1' , \cdots , \JM_k' \, ; \, \Z \big)
			& \leq & \gamma \, \cdot \,
			\min_{\JM_1, \cdots , \JM_k} \;
			f \big(\JM_1 , \cdots , \JM_k \, ; \, \Z \big) .
		\end{eqnarray*}
	\end{small}%
	Here $\{ {\JM}_1, \cdots , {\JM}_k \}$ and $\{ {\JM}_1', \cdots , {\JM}_k' \}$
	are $k$-partitions of $[n]$.
\end{defn}

Many $(1+\epsilon)$-approximation algorithms have been proposed, 
but they are computationally expensive
\citep{chen2009coresets,har2004coresets,kumar2004simple,matouvsek2000approximate}.
There are also relatively efficient constant factor approximate algorithms,
e.g., \citep{arthur2007kmeans,kanungo2002local,song2010fast}.

Let $\ph$ be a feature map, $\Ph$ be the matrix with rows $\phi(\a_1), \ldots, \phi(\a_n)$,
and $\K = \Ph \Ph^T \in \R^{n \times n}$ be the associated kernel matrix.
Analogously, we denote the objective function for kernel $k$-means clustering by
\begin{small}
	\begin{eqnarray*}
		f \big(\JM_1 , \cdots , \JM_k \; ; \; \Ph \big)
		\; = \; \frac{1}{n} \sum_{i=1}^k \sum_{j\in \JM_i} 
		\bigg\| \ph (\a_j ) \: - \: \frac{1}{|\JM_i |} \sum_{l \in \JM_i} \ph (\a_l ) \bigg\|_2^2 ,
	\end{eqnarray*}
\end{small}%
where $\{\JM_1 , \cdots , \JM_k \}$ is a $k$-partition of $[n]$.

\begin{theorem} [Kernel $k$-Means with Nystr\"om Approximation] \label{thm:kkmeans:nystrom2}
	Choose a sketching matrix $\PP \in \RB^{n\times c}$ and sketch size $c$ consistent with Table~\ref{tab:sketch}.
	Let $\C \W^\dag \C^T$ be the previously defined Nystr\"om approximation of $\K$.
	Let $\B \in \RB^{n\times s}$ be any matrix satisfying $\B \B^T = (\C \W^\dag \C^T )_s$.
	Let the $k$-partition $\{ \tilde{\JM}_1, \cdots , \tilde{\JM}_k \}$
	be the output of a $\gamma$-approximate algorithm applied to the rows of $\B$. 
	With probability at least $0.9$, 
	\begin{small}
		\begin{eqnarray*}
			f \big(\tilde{\JM}_1, \cdots , \tilde{\JM}_k \; ; \; \Ph \big)
			& \leq & \gamma \big( 1 + \epsilon + \tfrac{k}{s} \big)  \, \cdot \,
			\min_{\JM_1, \cdots , \JM_k} \;
			f \big(\JM_1 , \cdots , \JM_k \; ; \; \Ph \big) .
		\end{eqnarray*}
	\end{small}%
\end{theorem}

\begin{rmk}
	Kernel $k$-means clustering is an NP-hard problem.
	Therefore, instead of comparing with $\min f$,
	we compare with clusterings obtained using $\gamma$-approximate algorithms.
	Theorem~\ref{thm:kkmeans:nystrom2} shows that, when uniform sampling to form the Nystr\"om approximation, if $s = \OM(\frac{k}{\epsilon})$
	and $c = \tilde{\OM} (\frac{\mu s}{\epsilon} )$, then
	the returned clustering has an objective value that is at most a factor of $\epsilon$ larger
	than the objective value of the kernel $k$-means clustering returned by the $\gamma$-approximate algorithm.
\end{rmk}

\begin{rmk}
	\label{remark:twoerrors}
	Assume we are in a practical setting where $c$, the budget of column samples one can use to form a Nystr\"om approximation, and $k$, the number of desired cluster centers, are fixed. 
	The pertinent question is how to choose $s$ to produce a high-quality approximate clustering. 
	Theorem~\ref{thm:kkmeans:nystrom2} shows that
	for uniform sampling, the error ratio is
	\[
	1 + \epsilon + \tfrac{k}{s}
	\; = \;
	1 + \tilde{\OM} (\tfrac{s \mu}{c}) + \tfrac{k}{s} .
	\]
	To balance the two sources of error, $s$ must be larger than $k$, but not too large a fraction of $c$. 
	To minimize the above error ratio, $s$ should be selected on the order of $\sqrt{{k c} / {\mu }}$.
	Since the matrix coherence $\mu $ ($\geq 1$) is unknown, it can be heuristically treated as a constant.
\end{rmk}

We empirically study the effect of the values of $c$ and $s$ using a data set comprising $8.1$ million samples.
Note that computing the kernel $k$-means clustering objective function requires the formation of the entire kernel matrix $\K $, which is infeasible for a
data set of this size;
instead, we use normalized mutual information (NMI)~\citep{strehl2002cluster}---a standard measure of the performance of clustering algorithms---
to measure the quality of the clustering obtained by approximating kernel $k$-means clustering using Nystr\"om approximations formed through uniform sampling.
NMI scores range from zero to one, with a larger score indicating better performance.
We report the results in Figure~\ref{fig:spark}.
The complete details of the experiments, including the experimental setting and time costs, are given in Section~\ref{sxn:implementation}.

\begin{figure}[!ht]
	\begin{center}
		\centering
		\subfigure[\textsf{Fix $s=20$ and vary $c$.}]{\includegraphics[width=0.46\textwidth]{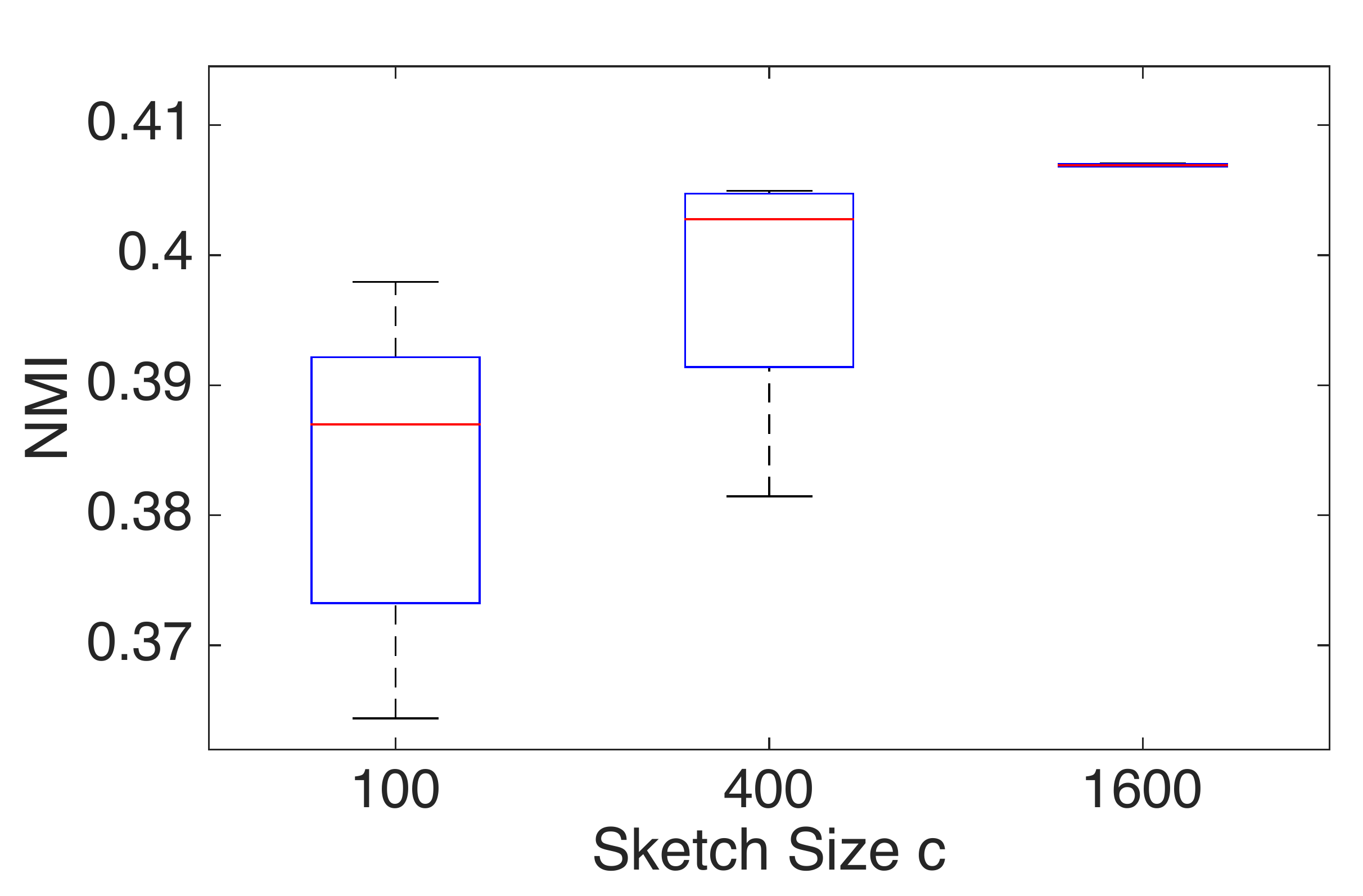} 
			\label{fig:spark:c}}
		~~~
		\subfigure[\textsf{Fix $c=1600$ and vary $s$.}]{\includegraphics[width=0.46\textwidth]{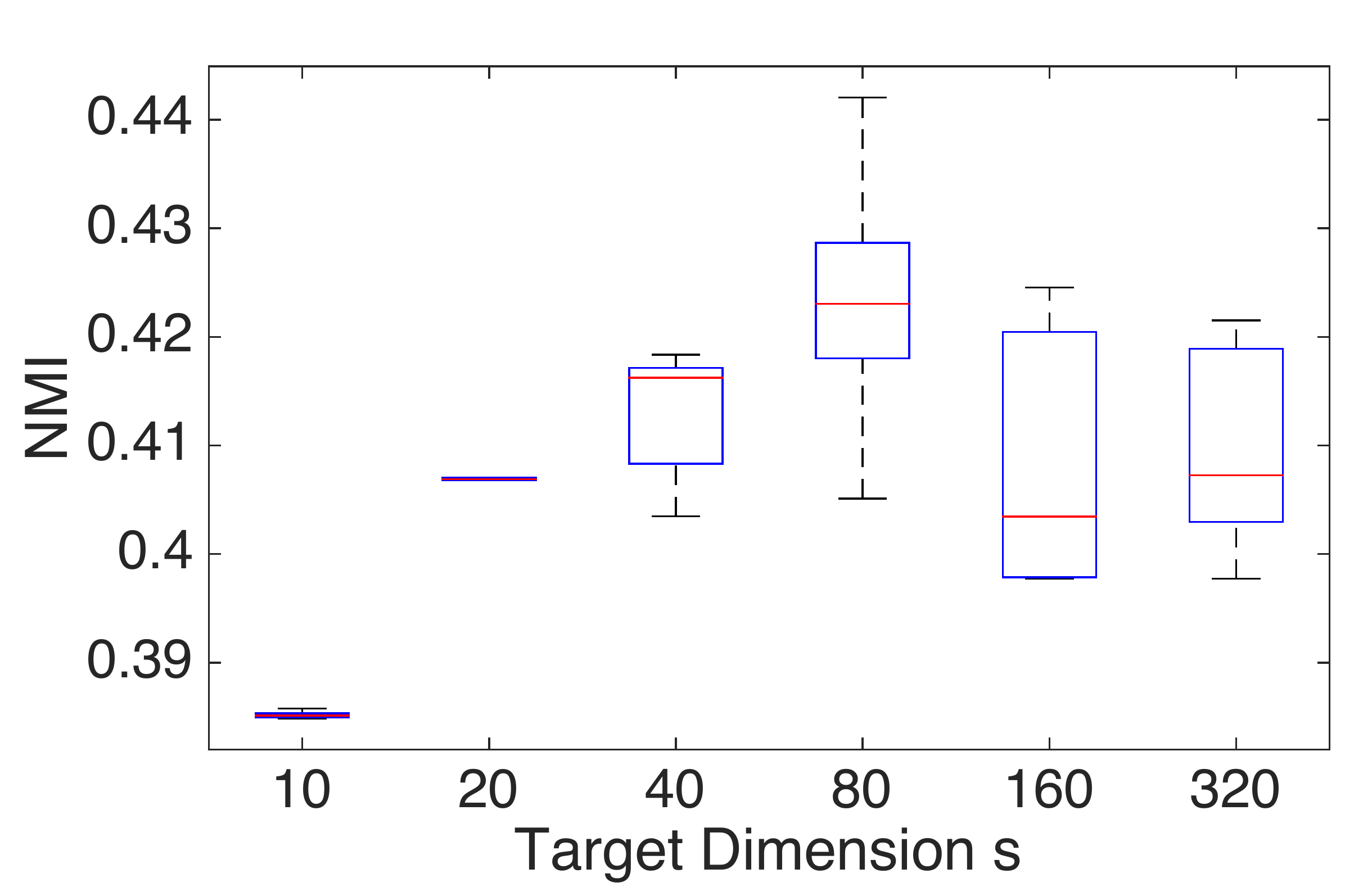} 
			\label{fig:spark:s}}
	\end{center}
	\vspace{-3mm}
	\caption{Performance of approximate kernel $k$-means clustering on the 
		MNIST8M data set, for which $n=8.1\times 10^6$, $d=784$, and $k=10$. We use
		Nystr\"om approximations formed using uniform sampling, and 
		the built-in $k$-means algorithm of Spark MLlib 
		(a parallelized variant of the $k$-means++ method)
		with the setting $\textsf{MaxIter}=100$.}
	\label{fig:spark}
	\vspace{-3mm}
\end{figure}

From Figure~\ref{fig:spark:c} we observe that larger values of $c$ lead to better and more stable clusterings: the mean of the NMI increases and its standard deviation decreases.
This is reasonable and in accordance with our theory. However, larger values of $c$ incur more computations, so
one should choose $c$ to trade off computation and accuracy.

Figure~\ref{fig:spark:s} shows that for fixed $k$ and $c$,
the clustering performance is not monotonic in $s$,
which matches Theorem~\ref{thm:kkmeans:nystrom2} (see the discussion in Remark~\ref{remark:twoerrors}).
Setting $s$ as small as $k$ results in poor performance.
Setting $s$ over-large not only incurs more computations, but also negatively affects clustering performance;
this may suggest the necessity of rank-restriction.
Furthermore, 
in this example, $\sqrt{k c} = 126.5$,
which corroborates the suggestion made in Remark~\ref{remark:twoerrors}
that setting $s$ around $\sqrt{k c / \mu}$ (where $\mu$ is unknown but can be treated as a constant larger than $1$) 
can be a good choice.

\begin{rmk}
	\citet{musco2016recursive} established a $1+\epsilon$ approximation ratio
	for the kernel $k$-means objective value when a non-rank-restricted Nystr\"om approximation is formed 
	using ridge leverage scores (RLS) sampling;
	their analysis is specific to RLS sampling and does not extend to other sketching methods.
	By way of comparison, our analysis covers several popular sampling schemes and applies to rank-restricted Nystr\"om approximations, but does not extend to RLS sampling.
\end{rmk}

\subsection{Approximate Kernel $k$-Means with KPCA and Power Method} \label{sec:main:kpca}

The use of dimensionality reduction to increase the computational efficiency of $k$-means clustering has been widely studied,
e.g. in~\citep{boutsidis2010random,boutsidis2015randomized,cohen2015dimensionality,feldman2013turning,zha2002spectral}.
Kernel principal component analysis (KPCA) is particularly well-suited to this application~\citep{dhillon2004kernel,ding2005equivalence}.
Applying Lloyd's algorithm on the rows of $\Ph $ or $\K^{-1/2}$ 
has an $\OM (n^2 k)$ per-iteration complexity;
if $s$ features are extracted using KPCA and Lloyd's algorithm is applied to the resulting $s$-dimensional feature map, then the per-iteration cost reduces to $\OM(nsk).$
Proposition~\ref{prop:kpca} states that, to obtain a $1+\epsilon$ approximation ratio in terms of the kernel $k$-means objective function, it suffices to use
$s = \frac{k}{\epsilon}$ KPCA features. This proposition is a simple consequence of \citep{cohen2015dimensionality}.

\begin{proposition} [KPCA] \label{prop:kpca}
	Let $\Ph$ be a matrix with $n$ rows, and $\K = \Ph \Ph^T \in \R^{n \times n}$ be the corresponding kernel matrix. Let
	$\K_s = \V_s \Lam_s \V_s^T$ be the truncated SVD of $\K$
	and take $\B = \V_s \Lam_s^{1/2} \in \RB^{n\times s}$.
	Let the $k$-partition $\{ \tilde{\JM}_1, \cdots , \tilde{\JM}_k \}$
	be the output of a $\gamma$-approximate algorithm applied to the rows of $\B$. 
	Then
	\begin{small}
		\begin{eqnarray*}
			f \big(\tilde{\JM}_1, \cdots , \tilde{\JM}_k \; ; \; \Ph \big)
			& \leq & \gamma \, \big( 1 + \tfrac{k}{s} \big)  \, \cdot \,
			\min_{\JM_1, \cdots , \JM_k} \;
			f \big(\JM_1 , \cdots , \JM_k \; ; \; \Ph \big) .
		\end{eqnarray*}
	\end{small}%
\end{proposition}

In practice, the truncated SVD (equivalently EVD) of $\K$ is computed using the power method or Krylov subspace methods.
These numerical methods do not compute the exact decomposition $\K_s = \V_s \Lam_s \V_s^T$, so Proposition~\ref{prop:kpca} is not 
directly applicable. It is useful to have a theory
that captures the effect of realisticly inaccurate estimates like $\K_s \approx \tilde{\V}_s \tilde{\Lam}_s \tilde{\V}_s^T$ on the clustering process. As one 
particular example, consider that
general-purpose implementations of the truncated SVD attempt to mitigate the fact that the computed decompositions are inaccurate by returning
very high-precision solutions, e.g.\ solutions that satisfy $\| (\I_n - \V_s \V_s^T ) \tilde{\V }_s \|_2 \leq 10^{-10}$. 
Understanding the trade-off between the precision of 
the truncated SVD solution and the impact on the approximation ratio of the approximate kernel $k$-means solution allows us to 
more precisely manage the computational complexity of our algorithms. Are such high-precision solutions necessary for kernel $k$-means clustering? 

Theorem~\ref{thm:power} answers this question by establishing 
that highly accurate eigenspaces are not significantly more useful in approximate kernel $k$-means clustering than eigenspace estimates with lower accuracy.
A low-precision solution obtained by running the power iteration for a few rounds
suffices for kernel $k$-means clustering applications.
We prove Theorem~\ref{thm:power} in Appendix~\ref{sec:proof:power}.

\begin{algorithm}[t]
	\caption{Approximate Kernel $k$-Means using the Power Method.}
	\label{alg:powmethod}
	\begin{small}
		\begin{algorithmic}[1]
			\STATE {\bf Input}: kernel matrix $\K \in \RB^{n \times n}$, number of clusters $k$, target dimension $s$ ($\geq k$),
			sketch~size~$c$~($\geq s$), number of iterations $t$ ($\geq 1$)
			\STATE Draw a Gaussian projection matrix $\PP \in \RB^{n \times c}$;
			\FOR{all $j \in [t]$} 
			\STATE $\PP \longleftarrow \K\PP$;
			\ENDFOR
			\STATE Orthogonalize $\PP$ to obtain $\U \in \RB^{n \times c}$;
			\STATE Compute $\C = \K \U$ and $\W = \U^T \K \U$;
			\STATE Compute a $\B \in \RB^{n \times s}$ satisfying $(\C \W^\dagger \C^T)_s = \B\B^T$;
			\RETURN $\B$.
		\end{algorithmic}
	\end{small}
\end{algorithm}

\begin{theorem} [The Power Method]  \label{thm:power}
	Let $\Ph$ be a matrix with $n$ rows, $\K = \Ph \Ph^T \in \RB^{n \times n}$ be the corresponding kernel matrix,
	and $\sigma_i$ be the $i$-th singular value of $\K$.
	Fix an error parameter $\epsilon \in (0, 1)$.
	Run Algorithm~\ref{alg:powmethod} with
	$t = \OM (\frac{ \log (n / \epsilon ) }{ \log ( \sigma_{s} / \sigma_{s+1} ) })$
	to obtain $\B \in \RB^{n\times s}$.
	Let the $k$-partition $\{ \tilde{\JM}_1, \cdots , \tilde{\JM}_k \}$
	be the output of a $\gamma$-approximate algorithm applied to the rows of $\B$.
	If $c = s + \OM (\log \frac{1}{\delta}) $, then
	\begin{small}
		\begin{eqnarray*}
			f \big(\tilde{\JM}_1, \cdots , \tilde{\JM}_k \; ; \; \Ph \big)
			& \leq & \gamma \, \big( 1 + \epsilon + \tfrac{k}{s} \big)  \, \cdot \,
			\min_{\JM_1, \cdots , \JM_k} \;
			f \big(\JM_1 , \cdots , \JM_k \; ; \; \Ph \big) .
		\end{eqnarray*}
	\end{small}%
	holds with probability at least $1-\delta$.
	If $c = s$, then the above inequality holds with probability
	$0.9 - \OM (s^{- \tau})$, where $\tau$ is a positive constant \citep{tao2010random}.
\end{theorem}

Note that the power method requires forming the entire kernel matrix $\K \in \RB^{n\times n}$,
which may not fit in memory even in a distributed setting. Therefore, 
in practice, the power method may not be as efficient as the Nystr\"om approximation with uniform sampling,
which avoids forming $\K$.

Theorem~\ref{thm:kkmeans:nystrom2}, Proposition~\ref{prop:kpca}, and Theorem~\ref{thm:power} are highly interesting from a theoretical perspective.
These results demonstrate that $s = \tfrac{k}{\epsilon } \big( 1 + o (1) \big)$ 
features are sufficient to ensure a $(1+\epsilon)$ approximation ratio. 
Prior work~\citep{dhillon2004kernel,ding2005equivalence} set $s = k$ and did not provide approximation ratio guarantees. Indeed,
a lower bound in the linear $k$-means clustering case due to~\citep{cohen2015dimensionality}
shows that $s = \Omega(\tfrac{k}{\epsilon})$ is necessary to obtain a $1+ \epsilon$ approximation ratio.

%

\section{Comparison to Spectral Clustering with Nystr\"om Approximation}
\label{sec:sc}

In this section, we provide a brief discussion and empirical comparison of our clustering algorithm,
which uses the Nystr\"om method to approximate kernel $k$-means clustering, with the popular
alternative algorithm that uses the Nystr\"om method to approximate spectral clustering.

\subsection{Background}
\label{sec:sc:background}

Spectral clustering is a method with a long history
\citep{Cheeger69_bound,Donath:1972,donath1973lower,fiedler1973algebraic,GM95,ST96}.
Within machine learning, spectral clustering is more widely used than
kernel $k$-means clustering~\citep{ng2002spectral,shi2000normalized}, and
the use of the Nystr\"om method to speed up spectral clustering has been popular
since~\citet{fowlkes2004spectral}. Both spectral clustering and kernel $k$-means clustering can be
approximated in time linear in $n$ by using the Nystr\"om method with uniform sampling.
Practitioners reading this paper may ask: 
\begin{quote}
	{\it How does the approximate kernel $k$-means clustering algorithm presented here, which uses Nystr\"om approximation, compare to 
		the popular heuristic of combining spectral clustering with Nystr\"om approximation?}
\end{quote}
Based on our theoretical results and empirical observations, our answer to this reader is:
\begin{quote}
	{\it Although they have equivalent computational costs, kernel $k$-means clustering
		with Nystr\"om approximation is both more theoretically sound and more
		effective in practice than spectral clustering with Nystr\"om
		approximation.}
\end{quote}
We first formally describe spectral clustering, and then substantiate our claim
regarding the theoretical advantage of our approximate kernel $k$-means method.
Our discussion is limited to the normalized and symmetric graph Laplacians used
in~\citet{fowlkes2004spectral}, but spectral clustering using asymmetric graph
Laplacians encounters similar issues.

\subsection{Spectral Clustering with Nystr\"om Approximation}
\label{sec:sc:discussion}

The input to the spectral clustering algorithm is an affinity matrix $\K \in \RB_+^{n\times n}$ that measures the pairwise similarities between 
the points being clustered; typically $\K$ is a kernel matrix or the adjacency matrix of a weighted graph constructed
using the data points as vertices. Let
$\D = \diag (\K \1_n)$ be the diagonal degree matrix associated with $\K$,
and $\LL = \I_n - \D^{-1/2} \K \D^{-1/2} $ be the associated normalized graph Laplacian matrix.
Let $\V_{k} \in \RB^{n\times k}$ denote the bottom $k$ eigenvectors of $\LL$, or equivalently, the top $k$ eigenvectors of $\D^{-1/2} \K \D^{-1/2}$.
Spectral clustering groups the data points by performing linear $k$-means clustering on the 
normalized rows of $\V_{k}$.
\citet{fowlkes2004spectral} popularized the application of the Nystr\"om approximation to spectral clustering. This algorithm computes
an approximate spectral clustering by:
(1) forming a Nystr\"om approximation to $\K$, denoted by $\tilde{\K}$;
(2) computing the degree matrix $\tilde{\D} = \diag (\tilde\K \1_n)$ of $\tilde{\K}$;
(3) computing the top $k$ singular vectors $\tilde{\V}_{k}$ of $\tilde\D^{-1/2} \tilde\K \tilde\D^{-1/2}$,
which are equivalent to the bottom $k$ eigenvectors of
$\tilde{\LL} = \I_n - \tilde\D^{-1/2} \tilde\K \tilde\D^{-1/2}$;
(4) performing linear $k$-means over the normalized rows of $\tilde{\V}_k \in \RB^{n\times k}$.

To the best of our knowledge, 
spectral clustering with Nystr\"om approximation does not have a bounded approximation ratio relative to exact spectral clustering. 
In fact, it seems unlikely that the approximation ratio could be bounded, as there are fundamental
problems with the application of the Nystr\"om approximation to the affinity matrix.

\begin{itemize}
	\item 
	The affinity matrix $\K$ used in spectral clustering must be elementwise nonnegative.
	However, the Nystr\"om approximation of such a matrix can have numerous negative entries, so
	$\tilde{\K}$ is, in general, not proper input for the spectral clustering algorithm.
	In particular, the approximated degree matrix $\tilde{\D} = \diag (\tilde\K \1_n)$
	may have negative diagonal entries, so $\tilde{\D}^{-1/2}$ is not guaranteed to be a real matrix;
	such exceptions must be handled heuristically. 
	The approximate asymmetric Laplacian $\tilde{\LL} = \I_n - \tilde\K \tilde\D^{-1}$
	does avoid the introduction of complex values; 
	however, the negative entries in $\tilde\D^{-1}$ negate whole columns of $\tilde\K$,
	leading to less meaningful negative similarities/distances.
	\item 
	Even if $\tilde{\D}^{-1/2}$ is real, the matrix
	$\tilde{\LL} = \I_n - \tilde\D^{-1/2} \tilde\K \tilde\D^{-1/2}$ may not be SPSD, much less a 
	Laplacian matrix.
	Thus the bottom eigenvectors of $\tilde{\LL}$ cannot be viewed as useful coordinates for linear $k$-means clustering in the same way that
	the eigenvectors of $\LL$ can be.
	\item
	Such approximation is also problematic in terms of matrix approximation accuracy.
	Even when $\tilde{\K}$ approximates $\K$ well, which can be theoretically guaranteed,
	the approximate Laplacian 
	$\tilde{\LL} = \I_n - \tilde\D^{-1/2} \tilde\K \tilde\D^{-1/2}$ can be far from $\LL$.
	This is because a small perturbation in $\tilde{\D}$ can have an out-sized influence on the eigenvectors of $\tilde{\LL}$.
	\item
	One may propose to approximate $\N = \D^{-1/2} \K \D^{-1/2}$, rather than $\K$, with a Nystr\"om approximation $\tilde{\N}$; this ensures that  
	the approximate normalized graph Laplacian $\tilde{\LL} = \I_n - \tilde{\N }$ is SPSD.
	However, this approach requires forming the entirety of $\K$ in order to compute the degree matrix $\D$, and thus has quadratic (with $n$) time and memory costs.
	Furthermore, although the resulting approximation, $\tilde{\LL}$, is SPSD, it is not a graph Laplacian:
	its off-diagonal entries are not guaranteed to be non-positive, and its smallest eigenvalue may be nonzero.
\end{itemize}
In summary, spectral clustering using the Nystr\"om approximation~\citep{fowlkes2004spectral}, which has proven to be a useful heuristic, and which is composed of theoretically principled parts, is less principled when viewed in its entirety.
Approximate kernel $k$-means clustering using Nystr\"om approximation is an equivalently efficient, but theoretically more principled alternative.  

\subsection{Empirical Comparison with Approximate Spectral Clustering using Nystr\"om Approximation}
\label{sec:sc:exp}

\begin{table}[t]\setlength{\tabcolsep}{0.3pt}
	\caption{Summary of the data sets used in our comparisons.}
	\label{tab:data}
	\begin{center}
		\begin{small}
			\begin{tabular}{c c c c}
				\hline
				~~~~~{\bf dataset}~~~~~ & ~~\#{\bf instances} ($n$)~~ 
				& ~~\#{\bf features} ($d$)~~ & ~~\#{\bf clusters} ($k$)~~ \\
				\hline
				MNIST \citep{lecun1998gradient} 	& 60,000 & 780 & 10 \\
				Mushrooms \citep{uci2010} & 8,124 & 112 & 2 \\
				PenDigits \citep{uci2010} & 7,494 & 16 & 10 \\
				\hline
			\end{tabular}
		\end{small}
	\end{center}
\end{table}

To complement our discussion of the relative merits of the two methods, we empirically compared the performance of our novel method of approximate kernel $k$-means clustering using the Nystr\"om method with the 
popular method of approximate spectral clustering using the Nystr\"om method. We used three classification data sets, described in Table~\ref{tab:data}.
The data sets used are available at 
\url{http://www.csie.ntu.edu.tw/~cjlin/libsvmtools/datasets/}.

\begin{figure}
	\vspace{-5mm}
	\begin{center}
		\centering
		\includegraphics[width=0.83\textwidth]{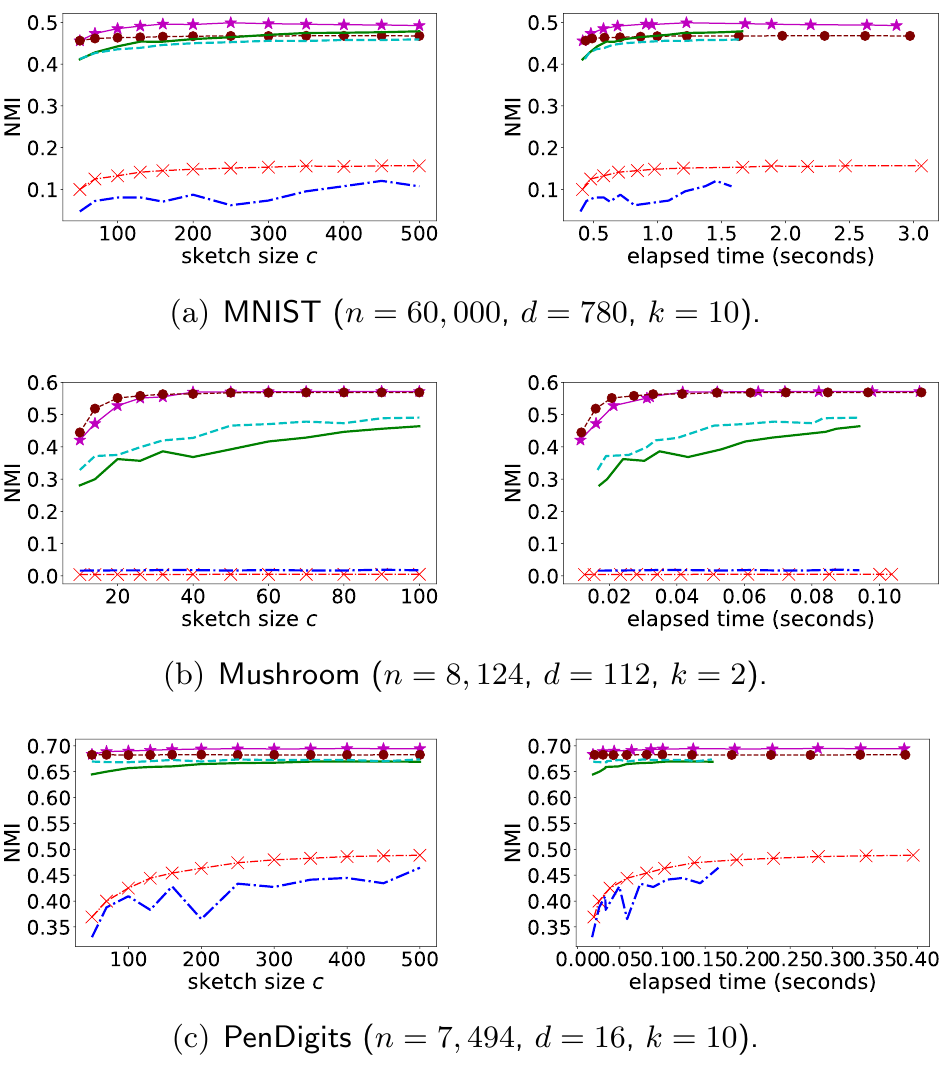}
		\vspace{3mm}
		\includegraphics[width=0.6\textwidth]{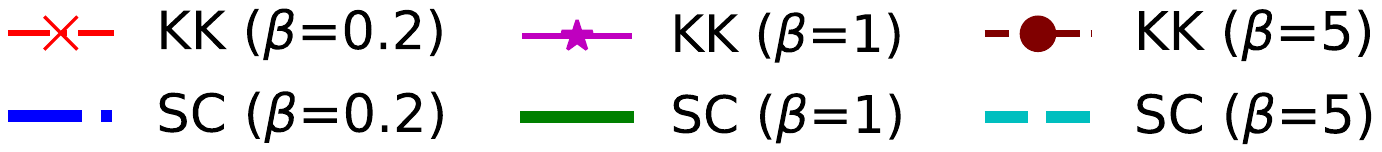}
		\vspace{-5mm}
	\end{center}
	\caption{Comparisons between the runtimes and accuracies of approximate spectral clustering (SC) and approximate kernel $k$-means clustering (KK), 
		both using uniform column sampling Nystr\"om approximation.}
	\label{fig:sc_vs_kk}
\end{figure}

Let $\a_1 , \ldots, \a_n \in \RB^d$ be the input vectors.
We take both the affinity matrix for spectral clustering and the kernel matrix for kernel $k$-means
to be the RBF kernel matrix $\K = [\kappa(\a_i, \a_j)]_{ij} \in \RB^{n\times n}$,
where $\kappa (\a, \a') = \exp \big( - \frac{1}{2\sigma^2} \|\a - \a' \|_2^2 \big)$
and $\sigma$ is the kernel width parameter.
We choose $\sigma$ based on the average interpoint distance in the data sets as
\begin{eqnarray} \label{eq:rbf_param}
\sigma 
& = & \beta \cdot \sqrt{\textstyle \frac{1}{n^2} \sum_{i=1}^n \sum_{j=1}^n \|\a_i - \a_j\|_2^2 },
\end{eqnarray}
where we take $\beta = 0.2$, $1$, or $5$.

The algorithms under comparison are all implemented in Python 3.5.2.
Our implementation of approximate spectral clustering follows the code in~\citep{fowlkes2004spectral}.
To compute linear $k$-means clusterings, we use the function \textsf{sklearn.cluster.KMeans}
present in the scikit-learn package.
Our algorithm for approximate kernel $k$-means clustering is described in more detail in Section~\ref{sec:medium:algs}.
We ran the computations on a MacBook Pro with a 2.5GHz Intel Core i7 CPU and 16GB of RAM.

We compare approximate spectral clustering (SC) with approximate kernel $k$-means clustering (KK), with both using the rank-restricted Nystr\"om method with uniform sampling.%
\footnote{Uniform sampling is appropriate for the value of $\sigma$ used in Eqn.~(\ref{eq:rbf_param}); see~\citet{gittens2013revisiting} for a detailed discussion of the effect of varying $\sigma$.}
We used normalized mutual information (NMI) \citep{strehl2002cluster} to evaluate clustering performance:
the NMI falls between 0 (representing no mutual information between the true and approximate clusterings) and 1 (perfect correlation of the two clusterings), so larger NMI indicates better performance. 
The target dimension $s$ is taken to be $k$; and, for each method, the sketch size $c$ is varied from $5s$ to $50s$.
We record the time cost of the two methods, excluding the time spent on the $k$-means clustering required in both algorithms.\footnote{For both
	SC and KK with Nystr\"om approximation, the extracted feature matrices have dimension $n\times k$, so the $k$-means clusterings required by
	both SC and KK have identical cost.}
We repeat this procedure $100$ times and report the averaged NMI and average elapsed time.

We note that, at small sketch sizes $c$, exceptions often arise during approximate spectral clustering due to negative entries in the degree matrix.
(This is an example, as discussed in Section~\ref{sec:sc:discussion}, of when approximate spectral clustering heuristics do not perform well.)
We discard the trials where such exceptions occur.

Our results are summarized in Figure~\ref{fig:sc_vs_kk}.
Figure~\ref{fig:sc_vs_kk} illustrates the NMI of SC and KK as a function of the sketch size $c$ and as a function of elapsed time for both algorithms.
While there are quantitative differences between the results on the three data sets, the plots all show that KK is more accurate as a function of the sketch size or elapsed time than SC.

\section{Single-Machine Medium-Scale Experiments}
\label{sec:medium}

In this section, we empirically compare the Nystr\"om method and random feature maps \citep{rahimi2007random} for kernel $k$-means clustering.
We conduct experiments on the data listed in Table~\ref{tab:data}.
For the Mushrooms and PenDigits data, we are able to evaluate the objective function value of kernel $k$-means clustering.

\begin{algorithm}[t]  
	\caption{Approximate Kernel $k$-Means Clustering using Nystr\"om Approximation.}
	\label{alg:kkmeans}
	\begin{small}
		\begin{algorithmic}[1]
			\STATE {\bf Input}: data set $\a_1, \ldots, \a_n \in \RB^d$,
			number of clusters $k$, target dimension $s$ ($\geq k$), 
			arbitrary integer $\ell$ $(> s)$, sketch size $c$ ($> \ell$),
			kernel function $\kappa$.
			\STATE \textcolor{Brown}{\bf // Step 1: The Nystr\"om Method}
			\STATE Form sketches $\C = \K \PP \in \RB^{n\times c}$ and $\W = \PP^T \C \in \RB^{c\times c}$,
			where $\K = [\kappa (\a_i , \a_j)]_{ij} \in \RB^{n\times n}$ is the kernel matrix
			and $\PP \in \RB^{n\times c}$ is some sketching matrix, e.g., uniform sampling;
			\STATE Compute a matrix $\R \in \RB^{n\times \ell}$ such that $\R \R^T = \C \W_\ell^{-1} \C^T$;
			\STATE \textcolor{OliveGreen}{\bf // Step 2: Dimensionality Reduction}
			\STATE Compute the rank-$s$ truncated SVD $\R_s = \tilde{\U}_s \tilde{\Si}_s \tilde{\V}_s^T $;
			\STATE Let $\B = \tilde{\U}_s \tilde{\Si}_s (= \R \tilde{\V}_s) \in \RB^{n\times s}$;
			\STATE \textcolor{Blue}{\bf // Step 3: Linear $k$-Means Clustering}
			\STATE Perform $k$-means clustering over the rows of $\B$;
			\RETURN the clustering results.
		\end{algorithmic}
	\end{small}
\end{algorithm}

\subsection{Single-Machine Implementation of Approximate Kernel $k$-Means}
\label{sec:medium:algs}

Our algorithm for approximate kernel $k$-means clustering comprises three steps:
\textcolor{Brown}{\bf Nystr\"om approximation}, 
\textcolor{OliveGreen}{\bf dimensionality reduction}, and 
\textcolor{Blue}{\bf linear $k$-means clustering}.
Both the single-machine as well as the distributed variants of the algorithm are governed by three parameters:
$s$, the number of features used in the clustering; $\ell$, a regularization parameter; and $c$, the sketch size. These 
parameters satisfy $k \leq s < \ell \leq c \ll n$.
\begin{enumerate}
	\item 
	\textcolor{Brown}{\bf Nystr\"om approximation.}
	Let $c$ be the sketch size and $\PP \in \RB^{n\times c}$ be a sketching matrix.
	Let $\C = \K \PP$ and $\W = \PP^T \K \PP = \PP^T \C $. 
	The standard Nystr\"om approximation is $ \C \W^\dag \C^T$; small singular values in $\W$ can lead to instability in the Moore-Penrose inverse,
	so a widely used heuristic is to choose $\ell < c$ and use $\C \W_\ell^{-1} \C^T$ instead of the standard Nystr\"om approximation.\footnote{The Nystr\"om approximation $\C \W^{\dag} \C^T$ is correct in theory,
		but the Moore-Penrose inverse often causes numerical errors in practice.
		The Moore-Penrose inverse drops all the zero singular values, however, 
		due to the finite numerical precision, it is difficult to determine whether a singular value, say $10^{-12}$, should be zero or not, and this makes the computation unstable:
		if such a small singular value is believed to be zero, it will be dropped;
		otherwise, the Moore-Penrose inverse will invert it to obtain a singular value of $10^{12}$.
		Dropping some portion of the smallest singular values is a simple heuristic that avoids this instability.
		This is why we heuristically use $\C \W_\ell^{-1} \C^T$ instead of $\C \W^{\dagger} \C^T$.
		Currently we do not have theory for this heuristic. \citet{chiu2013sublinear} considers
		the theoretical implications of this regularization heuristic, but their results do not apply to our problem.}
	We set $\ell = \lceil c / 2 \rceil$ (arbitrarily).
	Let $\W_\ell = \U_{\W,\ell} \Lam_{\W,\ell} \U_{\W,\ell}^T$ be the truncated SVD of $\W$ and 
	return $\R = \C \U_{\W,\ell} \Lam_{\W,\ell}^{-1/2} \in \RB^{n\times \ell}$ as the output of the Nystr\"om method.
	\item
	\textcolor{OliveGreen}{\bf Dimensionality reduction.}
	Let $\tilde{\V}_s \in \RB^{\ell\times s}$ contain the dominant $s$ right singular vectors of $\R$.
	Let $\B = \R \tilde{\V}_s \in \RB^{n\times s}$.
	It can be verified that $\B \B^T = (\C \W_\ell^{-1} \C^T)_s$, 
	which is our desired rank-restricted Nystr\"om approximation.
	\item
	\textcolor{Blue}{\bf Linear $k$-means clustering.}
	With $\B \in \RB^{n\times s}$ at hand, use an arbitrary off-the-shelf linear $k$-means clustering algorithm to cluster the rows of $\B$.
\end{enumerate}

\begin{figure}
	\vspace{-5mm}
	\begin{center}
		\centering
		\includegraphics[width=0.99\textwidth]{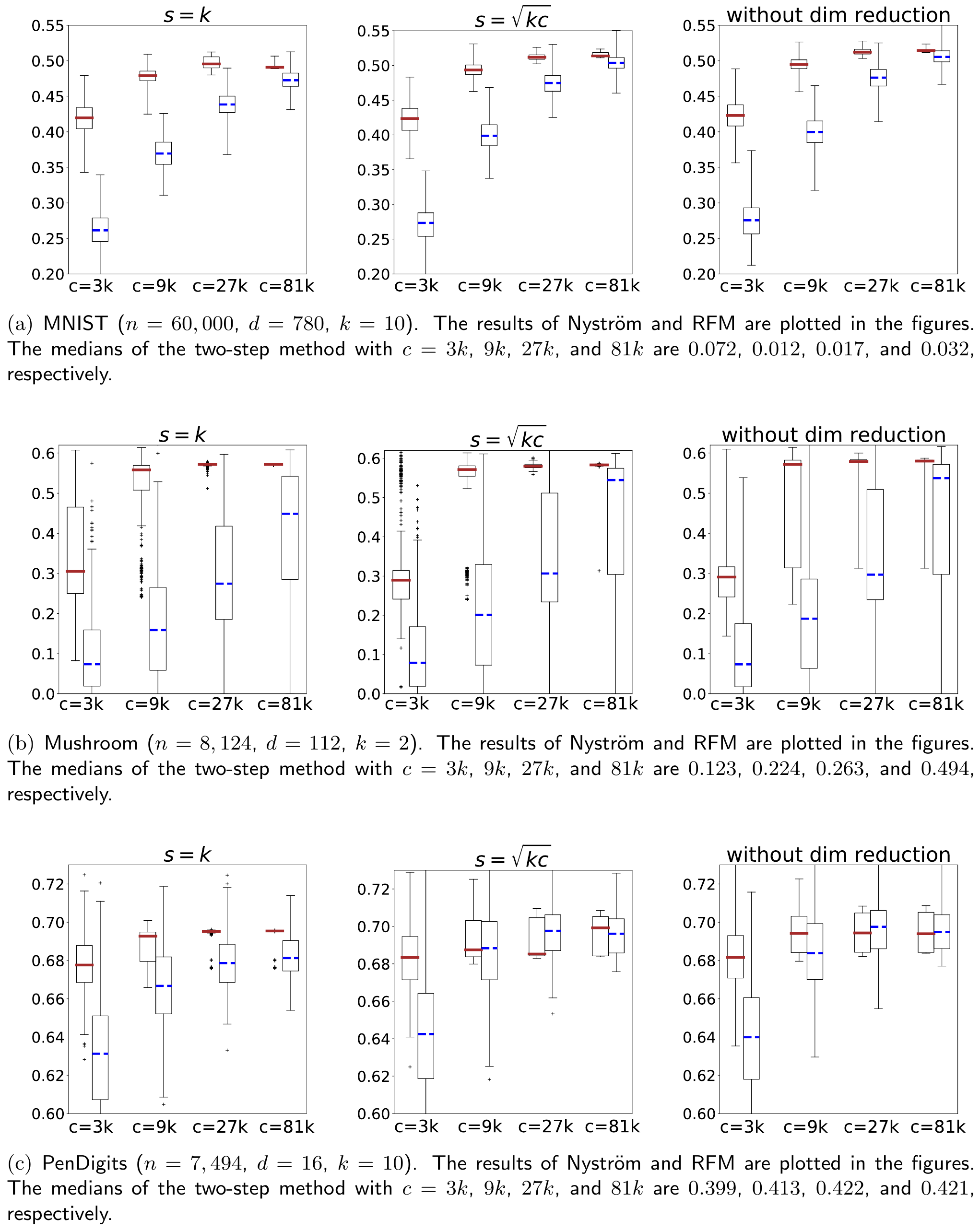}
		\vspace{3mm}
		\includegraphics[width=0.6\textwidth]{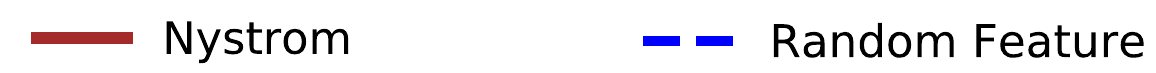}
		\vspace{-8mm}
	\end{center}
	\caption{Quality of approximate kernel $k$-means clusterings using Nystr\"om, 
		random feature maps \citep{rahimi2007random}, and the two-step method \citep{chitta2011approximate}. 
		(Dimensionality reduction is not applicable to the two-step method). 
		The $y$-axis reports the normalized mutual
		information (NMI).}
	\label{fig:nys_rfm_nmi}
\end{figure}

See Algorithm~\ref{alg:kkmeans} for the single-machine version of this approximate kernel $k$-means clustering algorithm.
Observe that we can use uniform sampling to form $\C$ and $\W $, and thereby avoid computing most of $\K$.

Let $\R \in \RB^{n\times c}$ be the feature matrix computed by random feature maps (RFM).
To make the comparison fair, we perform \textcolor{OliveGreen}{\bf
	dimensionality reduction} for RFM \citep{rahimi2007random} in the same way as described in Algorithm~\ref{alg:kkmeans} to compute $\B = \R \tilde{\V}_s$,
and apply linear $k$-means clustering on the rows of $\B$.

\subsection{Comparing Nystr\"om, Random Feature Maps, and Two-Step Method}

We empirically compare the clustering performances of kernel approximations
formed using Nystr\"om, random feature map (RFM)~\citep{rahimi2007random}, and
the two-step method~\citep{chitta2011approximate} on the data sets detailed in Table~\ref{tab:data}.

We use the RBF kernel with width parameter given by \eqref{eq:rbf_param};
Figure~\ref{fig:sc_vs_kk} indicates that $\beta = 1$ is a good choice for these data sets.
We conduct dimensionality reduction for both Nystr\"om and RFM to obtain $s$-dimensional features, and consider
three choices: $s= k$, $s = \big\lceil \sqrt{c k} \big\rceil $, and without dimensionality reduction (equivalently, $s = c$).

\begin{figure}
	\vspace{-5mm}
	\begin{center}
		\centering
		\includegraphics[width=0.99\textwidth]{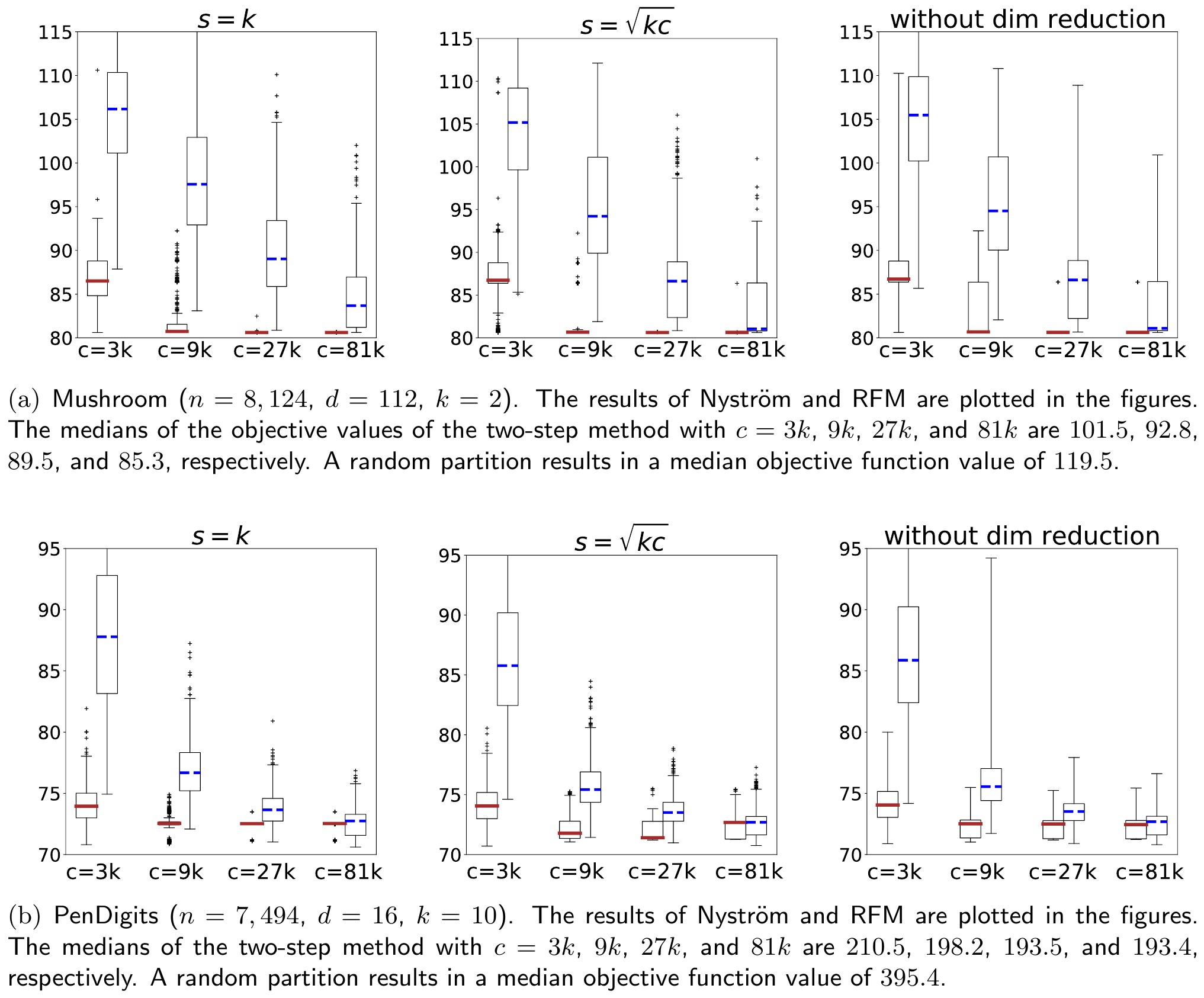}
		\vspace{3mm}
		\includegraphics[width=0.6\textwidth]{figure/nys_rfm_legend.pdf}
		\vspace{-5mm}
	\end{center}
	\caption{Quality of approximate kernel $k$-means clusterings obtained using
		rank-restricted Nystr\"om, RFMs~\citep{rahimi2007random}, and the two-step
		method \citep{chitta2011approximate}. 
		(Dimensionality reduction is not applicable to the two-step method). 
		The $y$-axis reports the kernel $k$-means clustering objective function value.}
	\label{fig:nys_rfm_obj}
\end{figure}

The quality of the clusterings is quantified using both normalized mutual information (NMI) \citep{strehl2002cluster}
and the  objective function value:
\begin{small}
	\begin{eqnarray}
	\label{eqn:kernel_kmeans_normalized}
	\frac{1}{ n} \: \sum_{i=1}^k \sum_{j\in \JM_i} 
	\bigg\| \k_j \: - \: \frac{1}{|\JM_i |} \sum_{l \in \JM_i} \k_l \bigg\|_2^2 ,
	\end{eqnarray}
\end{small}%
where $\k_1 , \cdots , \k_n \in \RB^n$ are the columns of the kernel matrix $\K$,
and the disjoint sets $\JM_1 , \cdots , \JM_k$ reflect the clustering.

We repeat the experiments $500$ times and report the results in Figures~\ref{fig:nys_rfm_nmi} and \ref{fig:nys_rfm_obj}.
The experiments show that as measured by both NMIs and objective values, the Nystr\"om method outperforms RFM in most cases.
Both the Nystr\"om method and RFM are consistently superior to the two-step method of \citep{chitta2011approximate}, which requires a large sketch size.
All the compared methods improve as the sketch size $c$ increases.

Judging from these medium-scale experiments, the target rank $s$ has little impact on the NMI and clustering objective value.
This phenomenon is not general; in the large-scale experiments of the next section we see that setting $s$ properly 
allows one to obtain a better NMI than an over-small or over-large $s$.

\section{Large-Scale Experiments using Distributed Computing}
\label{sxn:implementation}

In this section, we empirically study our approximate kernel $k$-means clustering algorithm on large-scale data.
We state a distributed version of the algorithm, implement it in Apache Spark\footnote{ This implementation is available at
	\url{https://github.com/wangshusen/SparkKernelKMeans.git}. }, and evaluate its performance on NERSC's Cori supercomputer.
We investigate the effect of increased parallelism, sketch size $c$, and target dimension $s$.

Algorithm~\ref{alg:map_reduce} is a distributed version of our method described in Section~\ref{sec:medium:algs}.
Again, we use uniform sampling to form $\C$ and $\W $ to avoid computing most of $\K$.
We mainly focus on the \textcolor{Brown}{\bf Nystr\"om approximation} step, as the other two steps are well supported by distributed computing systems such as Apache Spark.

\begin{algorithm}[t]  
	\caption{Distributed Approximate Kernel $k$-Means Clustering using Nystr\"om Approximation.}
	\label{alg:map_reduce}
	\begin{small}
		\begin{algorithmic}[1]
			\STATE {\bf Input}: data set $\a_1, \ldots, \a_n \in \RB^d$ distributed among $m$ machines,
			number of clusters $k$, target dimension $s$ ($\geq k$), 
			arbitrary integer $\ell > s$, sketch size $c$ ($> \ell$),
			kernel function $\kappa$.
			\STATE \textcolor{Brown}{\bf // Step 1: The Nystr\"om Method}
			\STATE {Sample} $c$ vectors from $\a_1, \ldots, \a_n$ to form $\a_1', \ldots , \a_c'$ and send to the driver;
			\STATE {Driver} computes $\W = [\kappa (a_i', a_j')]_{ij} \in \RB^{c\times c}$
			and a matrix $\Z \in \RB^{c\times \ell}$ satisfying $\Z \Z^T = \W_\ell^{-1}$;
			\STATE {Broadcast} $\Z $ and $\a_1', \ldots , \a_c'$ to all executors;
			\FOR{all $j \in [n]$}
			\STATE Each executor locally computes $\c_j = [ \kappa (\a_j , \a_1') ; \ldots ; \kappa (\a_j , \a_c') ] \in \RB^{c}$
			and $\r_j = \Z^T \c_j \in \RB^{\ell}$;
			\ENDFOR
			\STATE \textcolor{OliveGreen}{\bf // Step 2: Dimensionality Reduction}
			\STATE Let $\tilde{\V}_{s}$ contain the dominant $s$ right singular vectors of
			the distributed row matrix $\R = [\r_1 , \ldots , \r_n]^T \in \RB^{n\times \ell}$
			(computed using a distributed truncated SVD algorithm);
			\STATE Broadcast $\tilde{\V}_s \in \RB^{\ell\times s}$ to all executors;
			\FOR{all $j \in [n]$}
			\STATE Each executor locally computes $\bb_j = \tilde{\V}_s^T \r_j \in \RB^s$;
			\ENDFOR
			\STATE \textcolor{Blue}{\bf // Step 3: Linear $k$-Means Clustering}
			\STATE Perform $k$-means clustering over the Nystr\"om features $\bb_1 , \ldots , \bb_n \in \RB^s$
			(using a distributed linear $k$-means clustering algorithm).
			\RETURN the clustering results.
		\end{algorithmic}
	\end{small}
\end{algorithm}

\begin{table}[t]\setlength{\tabcolsep}{0.3pt}  
	\caption{The mean and standard deviation of the runtimes of Algorithm~\ref{alg:map_reduce}, in seconds, as a function of the number of
		compute nodes. We report the total runtime as well as the runtimes of the three stages:
		\textcolor{Brown}{\bf Nystr\"om approximation}, 
		\textcolor{OliveGreen}{\bf dimensionality reduction (DR)}, and 
		\textcolor{Blue}{\bf linear $k$-means clustering}.
		The total time is the sum of the three stages and all the Spark overheads, 
		e.g., Spark initialization.
		We fix $k=10$ and $c=400$.
		In the upper table, we set $s=20$, and the NMI is $0.400 \pm 0.009$;
		in the lower table, we set $s=80$, and the NMI is $0.410 \pm 0.010$.}
	\label{tab:parallel}
	\vspace{-4mm}
	\begin{center}
		\begin{footnotesize}
			\begin{tabular}{c c c c c c}
				\hline
				& 8 Nodes & 16 Nodes & 32 Nodes & 64 Nodes & 128 Nodes \\
				\hline
				~~\textcolor{Brown}{\bf Nystr\"om}~~
				& ~~$2998.6 \pm 263.9$~~ & ~~$1268.8 \pm 106.4$~~
				& ~~$665.2 \pm 39.6$~~ & ~~$369.7 \pm 126.1$~~ & ~~$181.4 \pm 22.2$~~ \\
				~~\textcolor{OliveGreen}{\bf DR}~~
				& ~~$37.4 \pm 58.3$~~ & ~~$63.8 \pm 72.5$~~
				& ~~$38.9 \pm 18.2$~~ & ~~$87.3 \pm 37.8$~~ & ~~$183.9 \pm 89.1$~~ \\
				~~\textcolor{Blue}{\bf $k$-means}~~
				& ~~$117.2 \pm 132.3$~~ & ~~$153.1 \pm 98.8$~~
				& ~~$119.9 \pm 76.9$~~ & ~~$223.2 \pm 88.4$~~ & ~~$391.0 \pm 155.9$~~ \\
				~~{\bf Total}~~
				& ~~$3201.5 \pm 344.0$~~ & ~~$1532.0 \pm 146.9$~~
				& ~~$867.9 \pm 97.5$~~ & ~~$734.7 \pm 152.0$~~ & ~~$828.2 \pm 210.4$~~ \\
				\hline
				\vspace{2mm}
			\end{tabular}
			\begin{tabular}{c c c c c c}
				\hline
				& 8 Nodes & 16 Nodes & 32 Nodes & 64 Nodes & 128 Nodes \\
				\hline
				~~\textcolor{Brown}{\bf Nystr\"om}~~
				& ~~$3008.4 \pm 385.6$~~  &  ~~$1312.9 \pm 141.2$~~ 
				& ~~$696.0 \pm 111.8$~~  &  ~~$342.5 \pm 16.9$~~ & ~~$197.0 \pm 24.1$~~ \\
				~~\textcolor{OliveGreen}{\bf DR}~~
				& $53.0 \pm 41.9$   &  $58.7 \pm 22.7$ 
				& $94.6 \pm 31.8$   &  $179.4 \pm 55.8$ & $470.6 \pm 118.5$ \\
				~~\textcolor{Blue}{\bf $k$-means}~~
				& $58.2 \pm 24.7$   &  $80.9 \pm 31.8$ 
				& $104.6 \pm 28.7$   &  $211.3 \pm 103.1$ & $501.7 \pm 162.9$ \\
				~~{\bf Total}~~
				& ~~$3168.1 \pm 434.1$~~  &  ~~$1492.6 \pm 135.3$~~ 
				& ~~$940.6 \pm 139.3$~~  &  ~~$775.9 \pm 151.6$~~ & ~~$1232.4 \pm 232.1$~~ \\
				\hline
			\end{tabular}
		\end{footnotesize}
	\end{center}
\end{table}

\subsection{Experimental Setup}

We implemented Algorithm~\ref{alg:map_reduce} in the Apache Spark framework \citep{zaharia2010spark,zaharia2012rdd}, using the Scala 
API. We computed the Nystr\"om approximation using the matrix operations provided by Spark, and invoked the MLlib library for machine learning
in Spark~\citep{meng2016mllib} to perform the dimensionality reduction and linear $k$-means clustering steps.
For the linear $k$-means clustering, we set the maximum number of iterations to $100$.

We ran our experiments on Cori Phase I, a NERSC supercomputer, located at Lawrence Berkeley National Laboratory. 
Cori Phase I is a Cray XC40 system with 1632 compute nodes, each of which has two 2.3GHz 16-core Haswell processors and 128GB of DRAM. 
The Cray Aries high-speed interconnect linking the compute nodes is configured in a dragonfly topology. 

We used the MNIST8M data set to conduct our empirical evaluations; this data set
has $n=8.1\times 10^6$ instances, $d=784$ features, and $k=10$ clusters.
We vary $c$ and $s$ and set $\ell = \frac{c}{2}$.
We chose the RBF kernel width parameter according to~\eqref{eq:rbf_param} with $\beta = 1.0$.
We use Cori's default setting of Spark configurations.
For each setting of parameters, we repeated the experiments $10$ times and recorded the NMIs and elapsed time.

\subsection{Effect of Increased Parallelism}

We varied the number of nodes to test the impact of increased parallelism on each of the three steps in Algorithm~\ref{alg:map_reduce}.
We set the sketch size to $c=400$ and the target dimension to $s=20$ or $80$.
Table~\ref{tab:parallel} reports the mean and standard deviation of the elapsed times.
As a reference point, using $32$ nodes, our algorithm takes $15$ minutes on average to group the $8.1$ million input instances into $10$ clusters.
We also plot the elapsed times in Figure~\ref{fig:spark_nodes}.

\begin{figure}[t]
	\begin{center}
		\centering
		\subfigure[\textsf{$k=10$, $s=20$, $c=400$.}]{
			\includegraphics[width=0.45\textwidth]{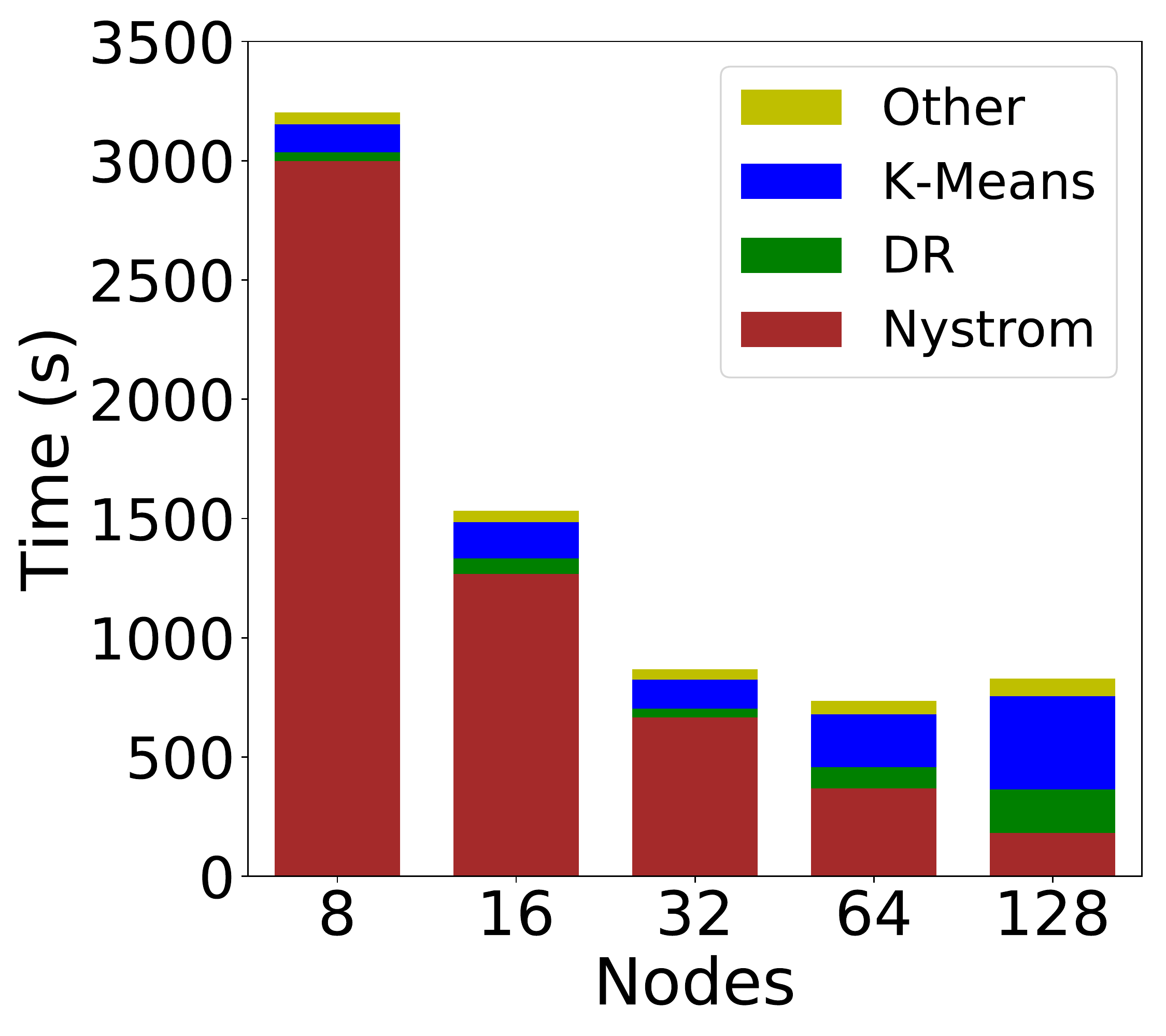}}~~
		\subfigure[\textsf{$k=10$, $s=80$, $c=400$.}]{
			\includegraphics[width=0.45\textwidth]{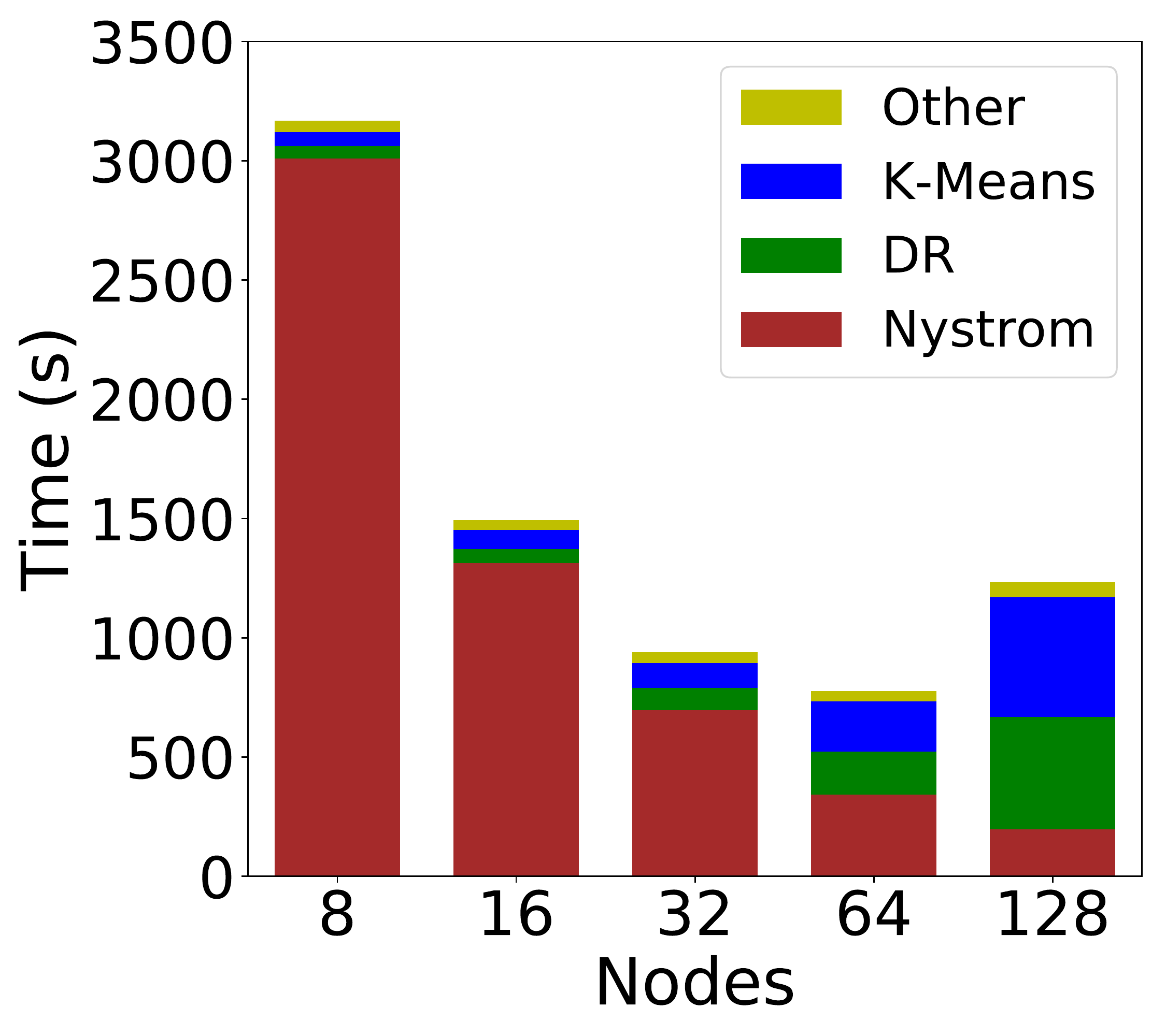}}
	\end{center}
	\vspace{-5mm}
	\caption{The mean of the runtimes of Algorithm~\ref{alg:map_reduce}, in seconds, as a function of the number of compute nodes.}
	\label{fig:spark_nodes}
\end{figure}

The \textcolor{Brown}{\bf Nystr\"om method} scales very well with the increase of nodes:
its elapsed time is inversely proportional to the number of nodes.
This is because the \textcolor{Brown}{\bf Nystr\"om method} requires only 2 rounds of communications;
the elapsed time is spent mostly on computation on the executors.

\textcolor{OliveGreen}{\bf Dimensionality reduction (DR)} and \textcolor{Blue}{\bf linear $k$-means clustering} are highly iterative
and thus have high latency, incur straggler delays, and have large scheduler and communication overheads.
See~\citep{gittens2016multi} for an in-depth discussion of the performance concerns when using Spark for distributed matrix computations.
As the number of nodes increases from $32$ to $64$ and higher, 
\textcolor{OliveGreen}{\bf DR} and \textcolor{Blue}{\bf $k$-means} exhibit anti-scaling behavior: as the number of nodes goes up, their 
runtimes increase rather than decrease. 
For \textcolor{OliveGreen}{\bf DR}, the input dimension is $\ell=\frac{c}{2}=200$, 
and the target dimension is $s=20$ or $80$; 
while for \textcolor{Blue}{\bf $k$-means}, the input dimension is $s=20$ or $80$. 
Clearly, the issue here is not that \textcolor{OliveGreen}{\bf DR} and \textcolor{Blue}{\bf $k$-means} are computationally intensive.
Instead, as the number of nodes increase, 
although the per-node computational time decreases, 
the Spark communications overheads are increasing,
making \textcolor{OliveGreen}{\bf DR} and \textcolor{Blue}{\bf $k$-means} less performant.
Such behavior has been characterized and studied
in other Spark implementations of linear algebra algorithms~\citep{gittens2016multi}.

\begin{table}[h]\setlength{\tabcolsep}{0.3pt}  
	\caption{The NMI and elapsed time (seconds) for varying sketch sizes $c$.
		Here ``T'' denotes the elapsed time.
		We use $32$ nodes.
		In the upper table, we set $s=20$;
		in the lower table, we set $s=80$.}
	\label{tab:spark_c}
	\begin{center}
		\begin{footnotesize}
			\begin{tabular}{c c c c c}
				\hline
				& $c=100$ & $c=400$ & $c=1,600$  \\
				\hline
				~~{NMI}~~
				& ~~$0.3833 \pm 0.0117$~~ & ~~$0.3975 \pm 0.0112$~~ & ~~$0.4069 \pm 0.0001$~~ \\
				~~{T(\textcolor{Brown}{Nystr\"om})}~~
				& ~~$138.5 \pm 62.1$~~ & ~~$665.2 \pm 39.5$~~ & ~~$2634.5 \pm 94.4$~~ \\
				~~{T(\textcolor{OliveGreen}{DR})}~~
				& ~~$1.2 \pm 0.5$~~ & ~~$38.9 \pm 18.2$~~ & ~~$31.3 \pm 6.6$~~ \\
				~~{T(\textcolor{Blue}{$k$-Means})}~~
				& ~~$60.9 \pm 28.7$~~ & ~~$119.9 \pm 76.9$~~ & ~~$86.9 \pm 27.3$~~ \\
				~~{T(Total)}~~
				& ~~$227.4 \pm 101.3$~~ & ~~$867.9 \pm 97.5$~~ & ~~$2785.4 \pm 83.2$~~ \\
				\hline
				\vspace{2mm}
			\end{tabular}
			\begin{tabular}{c c c c}
				\hline
				& $c=400$ & $c=1600$ \\
				\hline
				~~{NMI}~~
				& ~~$0.4101 \pm 0.0101$~~ & ~~$0.4233 \pm 0.0131$~~  \\
				~~{T(\textcolor{Brown}{Nystr\"om})}~~
				& ~~$696.0 \pm 111.8$~~ & ~~$2728.1 \pm 190.3$~~ \\
				~~{T(\textcolor{OliveGreen}{DR})}~~
				& ~~$94.6 \pm 31.8$~~ & ~~$97.9 \pm 20.6$~~ \\
				~~{T(\textcolor{Blue}{$k$-Means})}~~
				& ~~$104.6 \pm 28.7$~~ & ~~$88.0 \pm 23.9$~~ \\
				~~{T(Total)}~~
				& ~~$940.6 \pm 139.3$~~ & ~~$2952.2 \pm 174.5$~~ \\
				\hline
			\end{tabular}
		\end{footnotesize}
	\end{center}
\end{table}

\subsection{Effect of Sketch Size $c$}

We executed our Spark implementation using $32$ compute nodes, setting $s=20$ or $80$, and varying the target dimension $c$.
Table~\ref{tab:spark_c} reports the observed normalized mutual information (NMI)~\citep{strehl2002cluster} and elapsed times. 
As predicted by our theory, for fixed $k$ and $s$, larger $c$ always leads to better performance.

\begin{figure}[t]
	\begin{center}
		\centering
		\subfigure[\textsf{$k=10$ and $s=20$.}]{
			~~\includegraphics[width=0.4\textwidth]{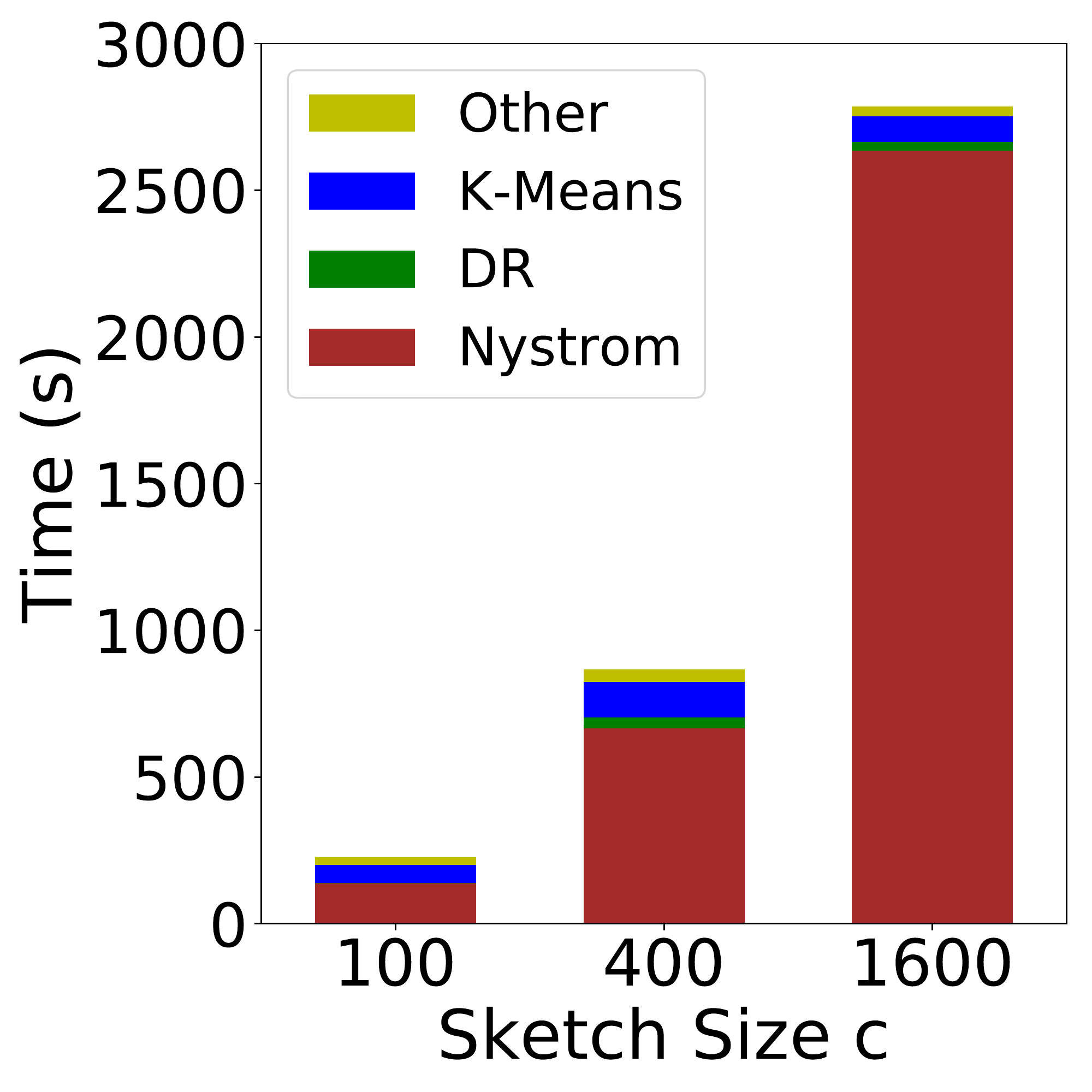}~~}
		\subfigure[\textsf{$k=10$ and $s=80$.}]{
			~~~~~~~\includegraphics[width=0.3\textwidth]{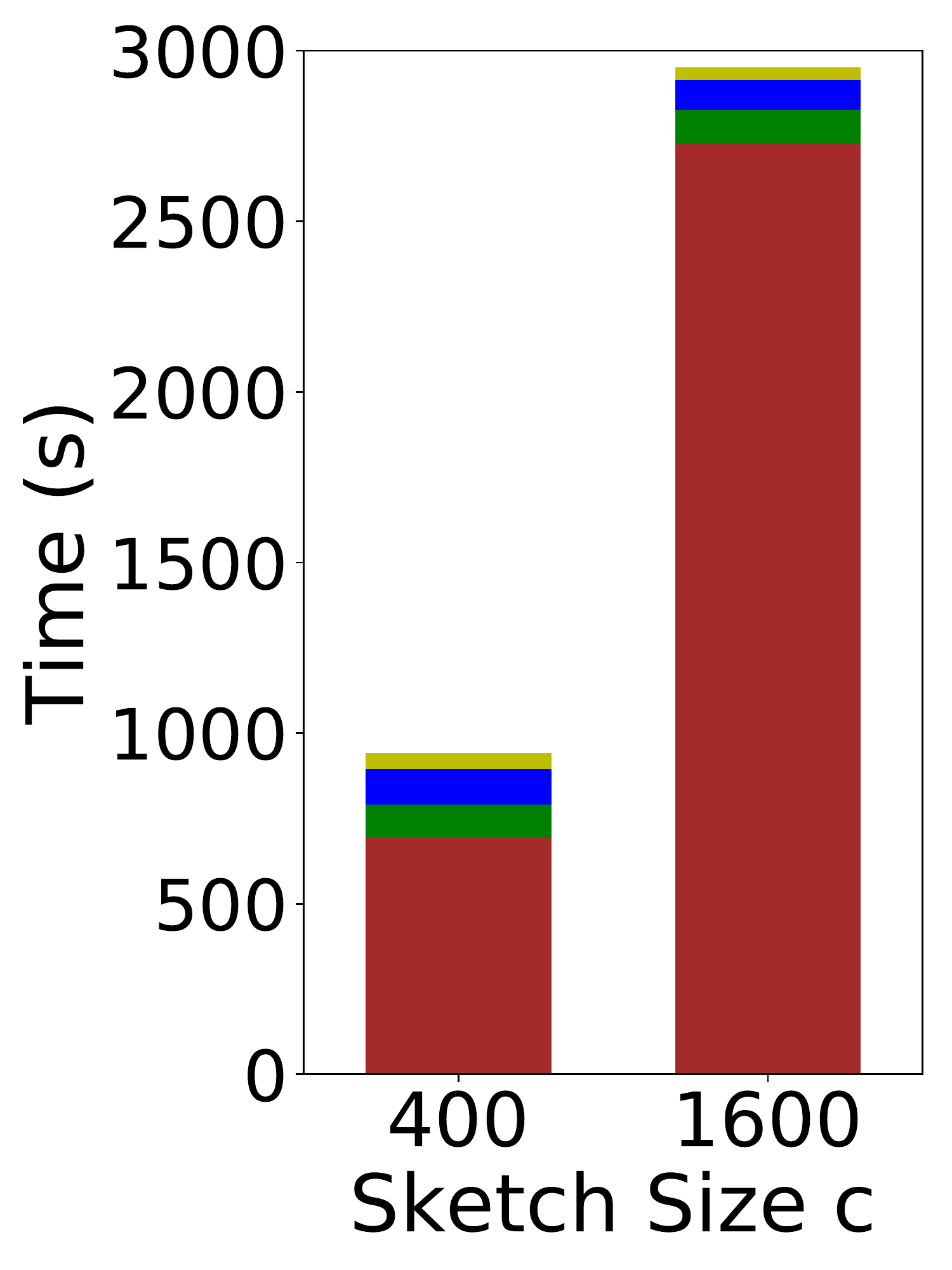}~~~~~~~}
	\end{center}
	\vspace{-5mm}
	\caption{The mean of the runtimes of Algorithm~\ref{alg:map_reduce}, in seconds, as a function of the sketch size $c$.}
	\label{fig:spark_c}
\end{figure}

In Figure~\ref{fig:spark_c} we plot the runtime as a function of the sketch size $c$.
The plots shows the significant influence of $c$ on the running time.
In Figure~\ref{fig:spark_c}, the runtime of \textcolor{Brown}{Nystr\"om} grows superlinearly with $c$.
According to our analysis, the computational and communication costs of \textcolor{Brown}{Nystr\"om} are both superlinear in $c$.
Therefore, the user should keep in mind this trade-off between the clustering performance and the computational cost.

\subsection{Effect of Target Dimension $s$}

We executed our Spark implementation using $32$ computational nodes, fixing $k=10$ and $c=1600$ and varying the target dimension $s$. 
Table~\ref{tab:spark_s} reports the observed NMIs and elapsed times.
For fixed $k$ and $c$, a moderately large $s$ leads to better clustering performance.
However, as $s$ grows, $\tfrac{k}{s}$ decreases but $\tfrac{s}{c}$ increases, so as predicted in 
Remark~\ref{remark:twoerrors}, the clustering performance is nonmonotonic in $s$.
Indeed, as $s$ grows from $80$ to $160$, the NMI deteriorates.
In practice, for fixed $k$ and $c$,
one should set $s$ moderately large, but not over-large.

\begin{table}[h]\setlength{\tabcolsep}{0.3pt}  
	\caption{The NMI and elapsed time (seconds) for varying target dimensions $s$.
		Here ``T'' denotes the elapsed time.
		We use $32$ nodes and fix $c=1600$.}
	\label{tab:spark_s}
	\begin{center}
		\begin{scriptsize}
			\begin{tabular}{c c c c c c c}
				\hline
				& $s=10$ & $s=20$ & $s=40$ & $s=80$ & $s=160$ & $s=320$  \\
				\hline
				~~{NMI}~~
				& ~$0.3852 \pm 0.0003$~ & ~$0.4069 \pm 0.0001$~
				& ~$0.4130 \pm 0.0062$~ & ~$0.4233 \pm 0.0131$~ 
				& ~$0.4086 \pm 0.0125$~ & ~$0.4099 \pm 0.0098$~\\
				~{T(\textcolor{Brown}{Nystr\"om})}~
				& ~$2702.2 \pm 131.6$~ & ~$2634.5 \pm 94.4$~
				& ~$2815.3 \pm 403.6$~ & ~$2728.1 \pm 190.3$~ 
				& ~$2849.9 \pm 511.2$~ & ~$2757.1 \pm 271.9$~ \\
				~{T(\textcolor{OliveGreen}{DR})}~
				& ~$21.1 \pm 3.2$~ & ~$31.3 \pm 6.6$~
				& ~$57.5 \pm 9.2$~ & ~$97.9 \pm 20.6$~
				& ~$168.2 \pm 29.0$~ & ~$570.5 \pm 151.0$~  \\
				~{T(\textcolor{Blue}{$k$-Means})}~
				& ~$84.0 \pm 6.4$~ & ~$86.9 \pm 27.3$~
				& ~$90.6 \pm 9.0$~ & ~$88.0 \pm 23.9$~ 
				& ~$99.8 \pm 12.2$~ & ~$169.3 \pm 64.2$~ \\
				~{T(Total)}~
				& ~$2840.9 \pm 132.0$~ & ~$2785.4 \pm 83.2$~
				& ~$3003.4 \pm 413.7$~ & ~$2952.2 \pm 174.5$~ 
				& ~$3157.5 \pm 549.2$~ & ~$3537.5 \pm 256.3$~ \\
				\hline
			\end{tabular}
		\end{scriptsize}
	\end{center}
\end{table}

In Figure~\ref{fig:spark_s} we plot the runtime as a function of the target dimension $s$.
Note that $s$ does not affect the time cost of \textcolor{Brown}{the Nystr\"om method}.
As $s$ grows, the time costs of \textcolor{OliveGreen}{dimensionality reduction}
and \textcolor{Blue}{linear $k$-means clustering} both increase.
This implies that a moderate $s$ is good for computational purpose.
Previously, in Figure~\ref{fig:spark}, we plot the NMI as a function of $s$ while fixing $k$ and $c$.
Figure~\ref{fig:spark} shows that a proper $s$ leads to better NMI than an over-large or over-small $s$.
In sum, setting $s$ properly has both computational and accuracy benefits, which corroborates our theory.

\begin{figure}[h]
	\begin{center}
		\centering
		\includegraphics[width=0.6\textwidth]{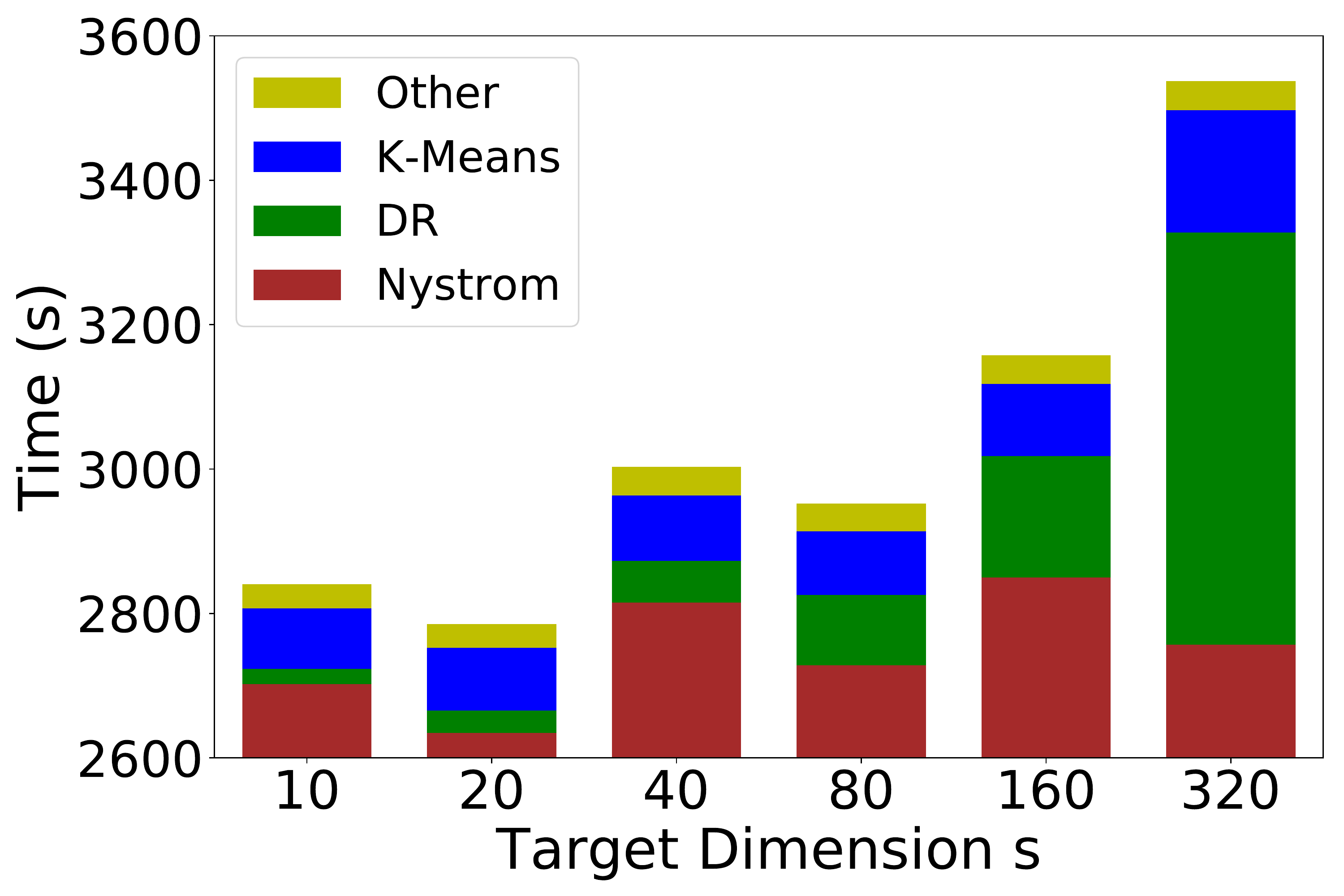}
	\end{center}
	\vspace{-5mm}
	\caption{The mean of the runtimes of Algorithm~\ref{alg:map_reduce}, in seconds, as a function of the target dimension $s$. We use $32$ nodes and fix $k=10$ and $c=1600$.}
	\label{fig:spark_s}
\end{figure}

\section{Conclusion}
\label{snx:conclusion}

We provided principled algorithms for computing approximate kernel $k$-means clusterings. In particular, we showed that 
the combination of linear $k$-means with rank-restricted Nystr\"om approximation is theoretically sound, practically useful, and
scalable to large data sets. This should be contrasted with approximate spectral
clustering using Nystr\"om approximation. Although the latter is a
widely-used approach to scalable non-linear clustering, it has
theoretical deficiencies and practical limitations.
Experiments demonstrated that approximate
kernel $k$-means clustering using the rank-restricted Nystr\"om approximation
consistently outperforms approximate spectral clustering.

Our analysis uses the concept of projection-cost preservation and builds upon the existing theory
of randomized linear algebra. 
Our main result is a $1+\epsilon$ approximation ratio guarantee for kernel
$k$-means clustering: when $s=\frac{k}{\epsilon} $ Nystr\"om features
are used, a $1+ \OM (\epsilon )$
approximation ratio is guaranteed with high probability.
As an intermediate theoretical result of independent interest, we introduced 
a novel rank-restricted Nystr\"om approximation and proved that it gives a $1+\epsilon$
relative-error low-rank approximation guarantee in the trace norm.

To complement our theory, we investigated the performance of the rank-restricted Nystr\"om approximation
when applied to kernel $k$-means clustering. We produced an Apache Spark
implementation of a distributed version of our approximate kernel $k$-means algorithm, and applied it to
cluster 8.1 million vectors; the results demonstrate the usefulness and simplicity of kernel $k$-means with rank-restricted Nystr\"om approximation for clustering
moderately large-scale data sets. Combined with recent work on 
user-friendly frameworks for distributed computing
systems~\citep{zaharia2010spark,zaharia2012rdd}, large-scale machine
learning~\citep{meng2016mllib}, and large-scale randomized linear
algebra~\citep{gittens2016multi}, our results suggest the use of kernel
$k$-means with rank-restricted Nystr\"om approximations for simple, computationally efficient, and theoretically principled clustering of 
large-scale data sets.

\acks{We thank the authors of~\cite{musco2016recursive} for bringing to our attention their results
  on approximate kernel $k$-means clustering through the use of ridge leverage score Nystr\"om approximations, the
  authors of~\cite{tropp2017fixed} for informing us of their contemporaneous guarantees on the approximation error of rank-restricted Nystr\"om approximations, and the anonymous reviewers for their helpful suggestions.
  We would like to acknowledge ARO, DARPA, NSF, and ONR for providing partial support of this work.
  }

\appendix
\section{Proof of Theorem~\ref{thm:nystrom}}
\label{sec:proof:nystrom}

In this section, we will provide a proof of Theorem~\ref{thm:nystrom}, our main quality-of-approximation result for rank-restricted Nystr\"om approximation.
We will start, in Section~\ref{sec:proof:nystrom:lemmas}, by establishing three technical lemmas; 
then, in Section~\ref{sec:proof:nystrom:structural}, we will establish an important structural result on the rank-restricted Nystr\"om approximation as well as key approximate matrix multiplication properties for low-rank matrix approximation; 
and finally, in Section~\ref{sec:proof:nystrom:theorem}, we will use these results to prove Theorem~\ref{thm:nystrom}.

\subsection{Technical Lemmas} \label{sec:proof:nystrom:lemmas}

Lemma~\ref{lem:proof:opt_fro} is a very well known result.
See \citep{boutsidis2011near,zhang2015svd}.
Here we offer a simplified proof.

\begin{lemma}\label{lem:proof:opt_fro}
	Let $\A \in \RB^{n\times m}$ be any matrix 
	and $\U \in \RB^{n\times c}$ have orthonormal columns.
	Let $s$ be any positive integer no greater than $c$.
	Then 
	\[
	(\U^T \A)_s \; = \; \argmin_{\rk (\Z) \leq s} \big\| \A - \U \Z \big\|_F^2 .
	\]
\end{lemma}

\begin{proof}
	Let $\U^\perp \in \RB^{n\times (n-c)}$ be the orthogonal complement of $\U$.
	It holds that
	\begin{eqnarray*}
		\Big\| \A - \U \Z  \Big\|_F^2
		& = & \Bigg\|
		\A
		\, - \,
		\left[
		\begin{array}{cc}
			\U & \U^\perp \\
		\end{array}
		\right]
		\left[
		\begin{array}{c}
			\Z \\
			\0 \\
		\end{array}
		\right]
		\Bigg\|_F^2 \nonumber \\
		& = & \Bigg\|
		\left[
		\begin{array}{c}
			\U^T \\
			(\U^\perp)^T \\
		\end{array}
		\right]
		\, \A
		\, - \,
		\left[
		\begin{array}{c}
			\Z \\
			\0 \\
		\end{array}
		\right]
		\Bigg\|_F^2 \nonumber \\
		& = & \Bigg\|
		\left[
		\begin{array}{cc}
			\U^T \A - \Z \\
			(\U^\perp)^T \A  \\
		\end{array}
		\right]
		\Bigg\|_F^2 \\
		& = & \big\| \U^T \A - \Z \big\|_F^2 + \big\| (\U^\perp)^T \A \big\|_F^2 ,
	\end{eqnarray*}
	where the second equality follows from the unitary invariance of the Frobenius norm.
	Then 
	\[
	\argmin_{\rk (\Z) \leq s} \big\| \A - \U \Z \big\|_F^2 
	\; = \; \argmin_{\rk (\Z) \leq s} \big\| \U^T \A - \Z \big\|_F^2
	\; = \; (\U^T \A)_s ,
	\]
	by which the lemma follows.
\end{proof}

Lemma~\ref{lem:proof:rank_restrict_fro} is a new result established in this work.
The lemma plays an important role in analyzing the rank-restricted Nystr\"om method.

\begin{lemma} \label{lem:proof:rank_restrict_fro}
	Let $\A \in \RB^{n\times m}$ be any matrix and $\U \in \RB^{n\times c}$ has orthonormal columns.
	Let $s \leq c$ be arbitrary integer.
	Let $\Q \in \RB^{n\times s}$ be the orthonormal bases of the rank $s$ matrix 
	$\U (\U^T \A)_s \in \RB^{n\times m}$.
	Then $\Q \Q^T \A = \U (\U^T \A)_s$.
\end{lemma}

\begin{proof}
	Let $\Z^\star = \Q^T \U (\U^T \A)_s \in \RB^{s\times m}$.
	By the definition of $\Q$, we have 
	\begin{eqnarray}\label{eq:proof:rank_restrict_fro:0}
	\Q \Q^T \U (\U^T \A)_s 
	\; = \; \U (\U^T \A)_s .
	\end{eqnarray}
	It holds that
	\begin{align*}
	\big\| \U (\U^T \A )_s - \A \big\|_F^2
	\; = \; \min_{\rk (\X ) \leq s} \big\| \U \X - \A \big\|_F^2
	\; \leq \; \min_{\rk (\Z ) \leq s} \big\| \Q \Z - \A \big\|_F^2
	\; \leq \; \big\| \Q \Z^\star - \A \big\|_F^2 ,
	\end{align*}
	where the equality follows from Lemma~\ref{lem:proof:opt_fro};
	the former inequality follows from that $\range (\Q ) \subset \range (\U )$.
	Because $\U (\U^T \A )_s = \Q \Z^\star$ by \eqref{eq:proof:rank_restrict_fro:0}, 
	the lefthand and righthand sides are equal, and thereby
	\begin{align*}
	\min_{\rk (\Z ) \leq s} \big\| \Q \Z - \A \big\|_F^2
	\; = \; \big\| \Q \Z^\star - \A \big\|_F^2 .
	\end{align*}
	Since $\Z$ is $s\times m$, it follows that
	\begin{align}\label{eq:proof:rank_restrict_fro:1}
	\min_{\Z} \big\| \Q \Z - \A \big\|_F^2
	\; = \; \min_{\rk (\Z ) \leq s} \big\| \Q \Z - \A \big\|_F^2
	\; = \; \big\| \Q \Z^\star - \A \big\|_F^2 .
	\end{align}
	Let $\Q^\perp \in \RB^{n\times (n-s)}$ be the orthogonal complement of $\Q$.
	It holds that
	\begin{eqnarray*}
		\Big\| \A - \Q \Z  \Big\|_F^2
		& = & \Bigg\|
		\A
		\, - \,
		\left[
		\begin{array}{cc}
			\Q & \Q^\perp \\
		\end{array}
		\right]
		\left[
		\begin{array}{c}
			\Z \\
			\0 \\
		\end{array}
		\right]
		\Bigg\|_F^2 \nonumber \\
		& = & \Bigg\|
		\left[
		\begin{array}{c}
			\Q^T \\
			(\Q^\perp)^T \\
		\end{array}
		\right]
		\, \A
		\, - \,
		\left[
		\begin{array}{c}
			\Z \\
			\0 \\
		\end{array}
		\right]
		\Bigg\|_F^2 \nonumber \\
		& = & \Bigg\|
		\left[
		\begin{array}{cc}
			\Q^T \A - \Z \\
			(\Q^\perp)^T \A  \\
		\end{array}
		\right]
		\Bigg\|_F^2 \\
		& = & \big\| \Q^T \A - \Z \big\|_F^2 + \big\| (\Q^\perp)^T \A \big\|_F^2 ,
	\end{eqnarray*}
	where the second equality follows from the unitary invariance of the Frobenius norm.
	Obviously, $\Z = \Q^T \A$ is the unique minimizer of $\min_\Z \| \Q \Z - \A \|_F^2$;
	otherwise $\Q^T \A - \Z$ in the righthand side is non-zero, making the objective function increase.
	
	On the one hand, Eqn.\ \eqref{eq:proof:rank_restrict_fro:1} shows that 
	$\Z^\star = \Q^T \U (\U^T \A)_s \in \RB^{s\times m}$ is one minimizer
	of $\min_{\Z} \| \Q \Z - \A \|_F^2$. 
	On the other hand, we have that $\Z=\Q^T \A$ is the unique minimizer of $\min_\Z \| \Q \Z - \A \|_F^2$.	
	By the uniqueness, we have 
	\[
	\Z^\star
	\; = \; \Q^T \U (\U^T \A)_s  
	\; = \; \Q^T \A .
	\]
	It follows that
	\[
	\Q \Q^T \U (\U^T \A)_s  
	\; = \; \Q \Q^T \A .
	\]
	The lemma follows from  \eqref{eq:proof:rank_restrict_fro:0} and the above equality.
\end{proof}

%
%
%

Lemma~\ref{lem:proof:power} is a new result established by this work.
Let $\A$ be any symmetric matrix.
Lemma~\ref{lem:proof:power} analyzes the power scheme:
using $\A^t \PP$, instead of $\A \PP$, as a sketch of $\A$.
The lemma shows that the power scheme leads to an improvement of 
$\big(\frac{\sigma_{s+1}^2 }{\sigma_{s}^2 } \big)^{t-1}$,
where $\sigma_i$ is the $i$-th biggest singular value of $\A$.
Lemma~\ref{lem:proof:power} with $t=1$ is identical to \cite[Lemma 9]{boutsidis2011near}.
The power scheme has been studied by \citet{gittens2013revisiting,halko2011ramdom,woodruff2014sketching};
their results do not allow the rank restriction $\rk (\X ) \leq s$.
To prove Theorem~\ref{thm:nystrom}, we only need $t=1$;
we will use Lemma~\ref{lem:proof:power} to extend Theorem~\ref{thm:nystrom} to the power method.

\begin{lemma} \label{lem:proof:power}
	Let $\A \in \RB^{n\times n}$ be any SPSD matrix
	and $\A_s = \V_s \Si_s \V_s^T$ be the truncated SVD.
	Let $\PP \in \RB^{n\times c}$ satisfy that $\PP^T \V_s \in \RB^{c\times s}$ has full column rank.
	Let $t \geq 1$ be any integer and $\C = \A^t \PP$. Then for $\xi = 2$ or $F$,
	\begin{eqnarray*}
		\min_{\rk(\X ) \leq s } \big\| \A - \C \X \big\|_\xi^2
		& \leq & \big\| \A - \A_s \big\|_\xi^2 
		+ \big(\tfrac{\sigma_{s+1}^2 }{\sigma_{s}^2 } \big)^{t-1}  
		\big\|  (\A - \A_s) \PP (\V_s^T \PP)^\dag  \big\|_\xi^2 ,
	\end{eqnarray*}
	where $\sigma_i$ is the $i$-th largest singular value of $\A$.
\end{lemma}

\begin{proof}
	We construct the rank $s$ matrix $\tilde{\A}_s = \C (\V_s^T \PP)^\dag \Si_s^{1-t} \V_{s}^T$
	and use it to facilitate our proof.
	Since $\tilde{\A}_s$ has rank $s$, and its column space is in $\range (\C)$, it holds that
	\begin{align*}
	& \min_{\rk(\X ) \leq s } \big\| \A - \C \X \big\|_\xi^2
	\; \leq \; \big\| \A - \tilde{\A}_s \big\|_\xi^2 \\
	& = \; \big\| (\A - \A_s ) + (\A_s - \tilde{\A}_s ) \big\|_\xi^2 
	\; \leq \; \big\| \A - \A_s \big\|_\xi^2
	+ \big\| \A_s - \tilde{\A}_s \big\|_\xi^2 .
	\end{align*}
	The latter inequality follows from 
	that the row space of $\A_s - \tilde{\A}_s$ is in $\range (\V_s)$,
	that $(\A - \A_s) \V_s = \0$,
	and the matrix Pythagorean theorem.
	By the assumption $\rk (\V_s^T \PP) = s$, it holds that $(\V_s^T  \PP) (\V_s^T  \PP)^\dag = \I_s$.
	We have
	\begin{align*}
	&\big\| \A_s - \tilde{\A}_s \big\|_\xi^2
	\; = \; \big\| \A_s - (\A_s^t + \A^t - \A_s^t ) \PP (\V_s^T \PP)^\dag \Si_s^{1-t} \V_{s}^T \big\|_\xi^2 \\
	& = \; \big\| \A_s - \V_s \Si_s^t \underbrace{(\V_s^T  \PP)}_{s\times c} 
	\underbrace{(\V_s^T \PP)^\dag}_{c\times s}  \Si_s^{1-t} \V_s^T
	- (\A^t - \A_s^t)  \PP (\V_s^T \PP)^\dag \Si_s^{1-t} \V_{s}^T \big\|_\xi^2 \\
	& = \; \big\| \A_s - \V_s \Si_s\V_s^T 
	- (\A^t - \A_s^t) \PP (\V_s^T \PP)^\dag \Si_s^{1-t} \V_s^T \big\|_\xi^2 \\
	& = \; \big\|  (\A^t - \A_s^t) \PP (\V_s^T \PP)^\dag \Si_s^{1-t} \V_s^T \big\|_\xi^2\\
	& \leq \; \big\|  (\A - \A_s)^{t-1} \big\|_2^2 \big\| \Si_s^{1-t} \big\|_2^2
	\big\|  (\A - \A_s) \PP (\V_s^T \PP)^\dag  \big\|_\xi^2 \\
	& = \; \big(\tfrac{\sigma_{s+1}^2 }{\sigma_{s}^2 } \big)^{t-1}  
	\big\|  (\A - \A_s) \PP (\V_s^T \PP)^\dag  \big\|_\xi^2,
	\end{align*}
	by which the lemma follows.
\end{proof}


\subsection{Structural Result and Approximate Matrix Multiplication Properties}
\label{sec:proof:nystrom:structural}

Lemma~\ref{lem:proof:nystrom_decompose} is a structural result of independent interest that is important for establishing Theorem~\ref{thm:nystrom}.
By \emph{structural result}, we mean that it is a linear algebraic result that holds for any matrix $\PP$, and in particular for any matrix $\PP$ that is a randomized sketching matrix.
In particular, given this structural result, randomness in an algorithm enters only via $\PP$.
The importance of establishing such structural results within randomized linear algebra has been highlighted previously~\citep{drineas2011faster,mahoney2011ramdomized,gittens2013revisiting,MD16_chapter,wang2017sketched}.

\begin{lemma} \label{lem:proof:nystrom_decompose}
	Let $\K \in \RB^{n\times n}$ be any SPSD matrix,
	$\PP \in \RB^{n\times c}$ be any matrix,
	and $s \leq c$ be any positive integer.
	Let $\D = \K^{1/2} \PP  \in \RB^{n\times c}$, $\rho = \rk (\D ) \leq c$,
	and $\U\in \RB^{n\times \rho}$ be the left singular vectors of $\D $.
	Let $\Q \in \RB^{n\times s}$ be any orthonormal bases of $\U (\U^T \K^{1/2} )_s$.
	It holds that
	\begin{eqnarray*} 
		\big( \C  \W^\dag \C^T \big)_s
		& = & \K^{1/2}\Q \Q^T  \K^{1/2} , \\
		\big\| \K - \big( \C  \W^\dag \C^T \big)_s \big\|_*
		& = & \min_{\rk (\Z) \leq s} \big\|  \K^{1/2} - (\K^{1/2} \PP) \Z \big\|_F^2 .
	\end{eqnarray*}
\end{lemma}

\begin{proof}
	Let $\D = \K^{1/2} \PP = \U \Si \V^T \in \RB^{n\times c}$ be the SVD of $\D $.
	It holds that 
	$\C = \K \PP = \K^{1/2} \D$ and $\W = \PP^T \K \PP = \D^T \D$.
	It follows the SVD of $\D$ that
	\begin{eqnarray*}\label{eq:proof:nystrom:0}
		\C (\W^\dag)^{1/2}
		\; = \; \K^{1/2} \D \big((\D^T \D)^\dag \big)^{1/2}
		\; = \; \K^{1/2} (\U \Si \V^T) \big(\V \Si^{-2} \V^T \big)^{1/2}
		\; = \; \K^{1/2} \U \V^T .
	\end{eqnarray*}
	
	We need the following lemma to prove Lemma~\ref{lem:proof:nystrom_decompose}.
	The following lemma is not hard to prove.
	
	\begin{lemma}
		Let the integers $n, c, \rho, s$ satisfy $n \geq c \geq \rho \geq s$.
		Let $\A \in \RB^{n\times \rho}$ be any matrix with rank at least $s$ and
		$\V \in \RB^{c\times \rho}$ has orthonormal columns.
		Then $(\A \V^T)_s = \A_s \V^T$.
	\end{lemma}
	

	If $\rk (\K^{1/2} \U) < s$, then obviously $\K^{1/2} \U \V^T = (\K^{1/2} \U \V^T)_s = (\K^{1/2} \U)_s \V^T$.
	We thus assume $\rk (\K^{1/2} \U \V^T) \geq s$ and can apply the above lemma to show $(\K^{1/2} \U \V^T )_s  = (\K^{1/2} \U )_s  \V^T$.
	In either case, we have
	\begin{eqnarray*}
		(\K^{1/2} \U   \V^T)_s \; = \; (\K^{1/2} \U )_s  \V^T .
	\end{eqnarray*}
	It follows from \eqref{eq:proof:nystrom:0} that 
	$\big( \C (\W^\dag)^{1/2} \big)_s = (\K^{1/2} \U )_s  \V^T$.
	Thus
	\begin{eqnarray*}
		\big( \C  \W^\dag \C^T \big)_s
		\; = \; (\K^{1/2} \U )_s (\U^T \K^{1/2} )_s 
		\; = \; (\K^{1/2} \U )_s \U^T \U (\U^T \K^{1/2} )_s 
	\end{eqnarray*}
	Let $\Q \in \RB^{n\times s}$ be the orthonormal bases of the rank $s$ matrix $\U (\U^T \K^{1/2} )_s$.
	Lemma~\ref{lem:proof:rank_restrict_fro} ensures that
	\[
	\Q \Q^T \K^{1/2}
	\; = \; \U (\U^T \K^{1/2} )_s.
	\]
	It follows that 
	\begin{eqnarray}\label{eq:proof:nystrom:1}
	\big( \C  \W^\dag \C^T \big)_s
	\; = \; \K^{1/2} \Q \Q^T \K^{1/2} ,
	\end{eqnarray}
	by which the former conclusion of the lemma follows.
	
	It is well known that $\|\Y \|_* = \tr (\Y)$ for any SPSD matrix $\Y$
	and that $\tr (\A^T \A ) = \|\A\|_F^2$ for any matrix $\A$.
	It follows from \eqref{eq:proof:nystrom:1} that
	\begin{align*}
	& \big\| \K - \big( \C  \W^\dag \C^T \big)_s \big\|_*
	\; = \; \big\| \K^{1/2} \big( \I_n - \Q \Q^T \big) \K^{1/2} \big\|_* \\
	& = \;  \tr \Big(  \K^{1/2} \big( \I_n - \Q \Q^T \big) \big( \I_n - \Q \Q^T \big) \K^{1/2} \Big) \\
	& = \; \big\| \big( \I_n - \Q \Q^T \big) \K^{1/2} \big\|_F^2
	\; = \; \big\|  \K^{1/2} - \U (\U^T \K^{1/2} )_s \big\|_F^2 ;
	\end{align*}
	here the second equality holds because $\K^{1/2} ( \I_n - \Q \Q^T ) \K^{1/2}$ is SPSD
	and $(\I_n - \Q \Q^T)$ is orthogonal projection matrix;
	the last equality follows from Lemma~\ref{lem:proof:rank_restrict_fro}.
	It follows from Lemma~\ref{lem:proof:opt_fro} that
	\begin{align*}
	& \big\| \K - \big( \C  \W^\dag \C^T \big)_s \big\|_*
	\; = \; \big\|  \K^{1/2} - \U (\U^T \K^{1/2} )_s \big\|_F^2 \\
	& = \; \min_{\rk (\X) \leq s} \big\|  \K^{1/2} - \U \X \big\|_F^2
	\; = \; \min_{\rk (\Z) \leq s} \big\|  \K^{1/2} - \D \Z \big\|_F^2 ,
	\end{align*}
	by which the latter conclusion of the lemma follows.
\end{proof}

Lemma~\ref{lem:property} formally defines the key \emph{approximate matrix multiplication properties} (subspace embedding property and approximate orthogonality property) which---when sketching dimensions are chosen appropriately (see Table~\ref{tab:property})---are shared by the uniform sampling, leverage score sampling, Gaussian projection, SRHT, and CountSketch sketching methods.
Establishing approximate matrix multiplication bounds is the key step in proving many approximate regression and low-rank approximation results~\citep{mahoney2011ramdomized}, and this lemma will be used crucually in the proof of Theorem~\ref{thm:nystrom}.

\begin{lemma} [Key Approximate Matrix Multiplication Properties for a Sketch] \label{lem:property}
	Let $\eta , \epsilon , \delta_1, \delta_2 \in (0, 1)$ be fixed parameters.
	Fix $\Y \in \RB^{n \times d}$ and let $\V \in \RB^{n\times s}$ have orthonormal columns.
	Let $\PP \in \RB^{n\times c}$ be one of the sketching methods listed in Table~\ref{tab:property}.
	When $c$ is larger than the quantity in the middle column of the corresponding row of Table~\ref{tab:property},
	the subspace embedding property
	\begin{eqnarray*}
		\big\| \V^T \PP \PP^T \V - \I_s \big\|_2
		\; \leq \; \eta 
	\end{eqnarray*}
	holds with probability at least $1-\delta_1$.
	When $c$ is larger than the quantity in the right-hand column of the corresponding row of Table~\ref{tab:property},
	the matrix multiplication property
	\begin{eqnarray*}
		\big\| \V^T \PP \PP^T \Y - \V^T \Y \big\|_F^2
		\; \leq \; \epsilon \|\Y \|_F^2  
	\end{eqnarray*}
	holds with probability at least $1-\delta_2$.
\end{lemma}

\begin{table}[h]\setlength{\tabcolsep}{0.3pt}
	\caption{Sufficient sketching dimensions for the sketching properties of Lemma~\ref{lem:property} to hold. 
		The leverage score sampling referred to here is with respect to the leverage scores of the matrix $\V$ described in Lemma~\ref{lem:property}.
		Similarly, $\mu$ is the row coherence of $\V$.}
	\label{tab:property}
	\begin{center}
		\begin{small}
			\begin{tabular}{c c c c}
				\hline
				{\bf Sketching}
				&~~{\bf Subspace Embedding}~~
				&~~{\bf Matrix Multiplication}~~\\
				\hline
				Leverage Score Sampling
				&~~~$\frac{ s }{\eta^2 } \log \frac{s}{\delta_1} $~~~
				& ~~~$ \frac{ s}{\epsilon \delta_2} $~~~ \\
				Uniform Sampling
				&~~~$  \frac{\mu  s }{\eta^2 } \log \frac{s}{\delta_1}$~~~
				& ~~~$\frac{\mu s}{\epsilon \delta_2} $~~~ \\
				SRHT
				&~~~$ \frac{ s + \log n }{\eta^2 }  \log \frac{s}{\delta_1 } $~~~
				& ~~~$ \frac{s + \log n }{\epsilon \delta_2 } $~~~ \\
				~Gaussian Projection~
				&~~~$ \frac{s +  \log (1/\delta_1)}{\eta^2} $~~~
				& ~~~$  \frac{s}{\epsilon \delta_2}$~~~ \\
				CountSketch
				&~~~$ \frac{s^2}{\delta_1 \eta^2}  $~~~
				&  ~~~$  \frac{s}{\epsilon \delta_2}  $~~~ \\
				\hline
			\end{tabular}
		\end{small}
	\end{center}
\end{table}

\begin{proof}
	This lemma is a reproduction of \cite[Lemma 2]{wang2015towards}, and is a
	collation of heterogenously stated results from the literature.
	The random sampling estimates were established by~\citet{drineas2008cur,wang2016spsd,woodruff2014sketching}.
	The Gaussian projection was firstly analyzed by~\citet{johnson1984extensions};
	see~\citet{woodruff2014sketching} for a proof of the stated sufficient sketching dimensions.
	The SRHT estimates are from~\citet{drineas2011faster,lu2013faster,tropp2011improved}.
	The CountSketch estimates are from~\citet{meng2013low,nelson2013osnap}.
\end{proof}

\subsection{Completing the Proof of Theorem~\ref{thm:nystrom}}
\label{sec:proof:nystrom:theorem}

\begin{proof}
	Let $\K_s = \V_s \Lam_s \V_s^T \in \RB^{n\times n}$ be the truncated SVD of $\K$.
	It follows from Lemma~\ref{lem:property} and the sketch sizes defined in Table~\ref{tab:sketch} that
	\begin{eqnarray*}
		\big\| \V_s^T \PP \PP^T \V_s - \I_s \big\|_2 & \leq & \eta ,\\
		\big\| (\K^{1/2} - \K_s^{1/2}) \PP \PP^T \V_s \big\|_F^2 
		& \leq & \epsilon \big\| \K^{1/2} - \K_s^{1/2} \|_F^2
	\end{eqnarray*}
	hold simultaneously with probability at least $0.9$.
	
	It follows from Lemma~\ref{lem:proof:nystrom_decompose} that
	\begin{align*}
	&\big\| \K - \big( \C  \W^\dag \C^T \big)_s \big\|_*
	\; = \; \min_{\rk (\Z) \leq s} \big\|  \K^{1/2} - (\K^{1/2} \PP) \Z \big\|_F^2 \\
	& \leq \; \big\| \K^{1/2} - \K_s^{1/2} \big\|_F^2
	+  \big\| (\K^{1/2} - \K_s^{1/2}) \PP (\V_s^T \PP)^\dag  \big\|_F^2 \\
	& = \; \big\| \K - \K_s  \big\|_*
	+  \big\| (\K^{1/2} - \K_s^{1/2}) \PP \PP^T \V_s (\V_s^T \PP \PP^T \V_s)^\dag  \big\|_F^2 \\
	& \leq \; \big\| \K - \K_s  \big\|_*
	+ \sigma_{s}^{-2} \big( \V_s^T \PP \PP^T \V_s  \big) 
	\, \big\| (\K^{1/2} - \K_s^{1/2}) \PP \PP^T \V_s \big\|_F^2  ,
	\end{align*}
	where the former inequality follows from Lemma~\ref{lem:proof:power} (with $t=1$).
	It follows that
	\begin{eqnarray*}
		\big\| \K - \big( \C  \W^\dag \C^T \big)_s \big\|_*
		& \leq & \big\| \K - \K_s  \big\|_*
		+ \tfrac{\epsilon}{(1-\eta)^2}\big\| \K - \K_s  \big\|_*.
	\end{eqnarray*}
	We let $\eta$ be a constant and obtain the former claim of Theorem~\ref{thm:nystrom}.
	The latter claim of Theorem~\ref{thm:nystrom} follows from Lemma~\ref{lem:proof:nystrom_decompose}.
\end{proof}

\section{Proof of Theorem~\ref{thm:kkmeans:nystrom2}}

In this section, we will provide a proof of Theorem~\ref{thm:kkmeans:nystrom2}, our main result for kernel $k$-means approximation.
We will start, in Sections~\ref{sec:proof:proj}, \ref{sec:proof:kmeans}, and \ref{sec:proof:kernel}, by presenting several technical tools of independent interest (Lemmas~\ref{lem:proj:nuclear}, \ref{lem:PCP_kmeans}, and~\ref{lem:kernel_kmeans}, respectively).
Then, in Section~\ref{sec:proof:main}, we will use use these tools to prove Theorem~\ref{thm:kkmeans:nystrom2}.

The proof of Theorem~\ref{thm:kkmeans:nystrom2} will proceed by combining Theorem~\ref{thm:nystrom} with Lemmas~\ref{lem:proj:nuclear}, \ref{lem:PCP_kmeans}, and~\ref{lem:kernel_kmeans}.
For convenience, the structure of proof is given in Figure~\ref{fig:dependence}.

\begin{figure}[h]
	\begin{center}
		\centering
		{\includegraphics[width=0.8\textwidth]{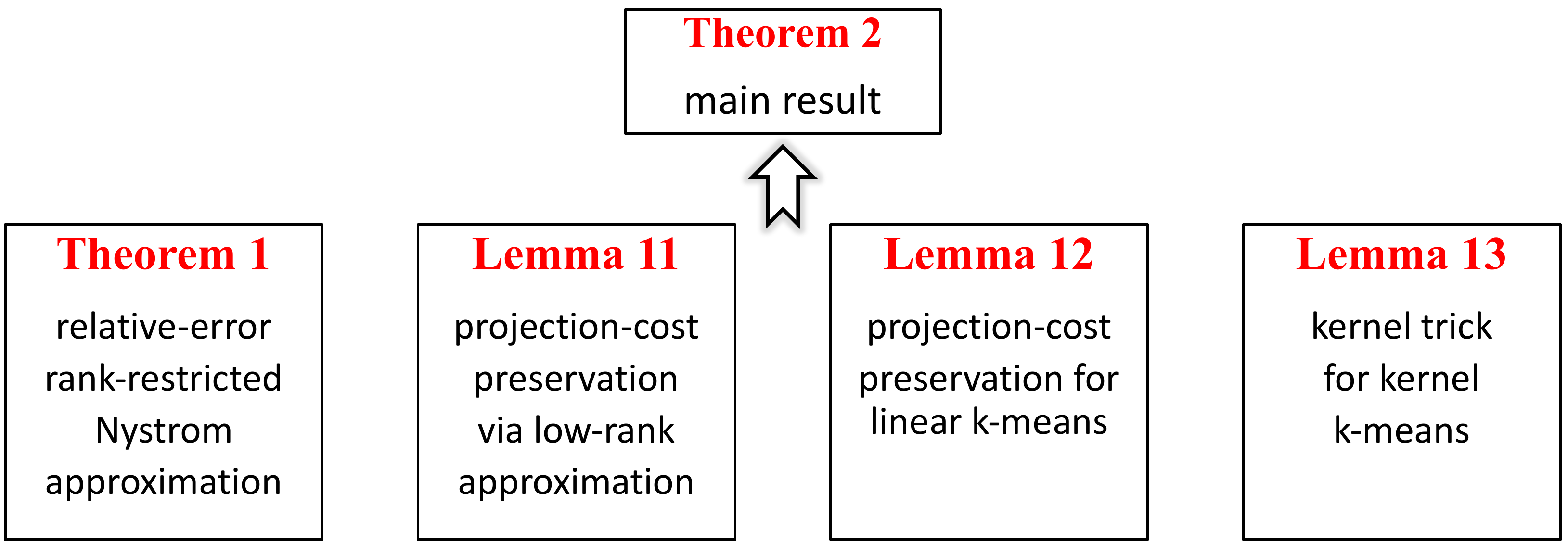}}
	\end{center}
	\caption{The structure of the proof of Theorem~\ref{thm:kkmeans:nystrom2}.}
	\label{fig:dependence}
\end{figure}

We use the notation of {\it orthogonal projection matrix} throughout this section.
A matrix $\M \in \RB^{n\times n}$ is an orthogonal projection matrix if and only if 
there exists a matrix $\Q$ with orthonormal columns such that
$\M = \Q \Q^T $.
It follows that $\M^T = \M $, $\M \M = \M $, $\M \preceq \I_n$
and that $\I_n - \M $ is also an orthogonal projection matrix.

\subsection{Projection-Cost Preservation via Low-Rank Approximation} \label{sec:proof:proj}

We first define projection-cost preservation, then show that certain low-rank approximations are projection-cost preserving. 
Projection-cost preservation was named and systematized by~\citet{cohen2015dimensionality}, but the idea
has been used in earlier works \citep{boutsidis2010random,boutsidis2015randomized,feldman2013turning,liang2014improved}.

\begin{defn} [Rank $k$ Projection-Cost Preservation] \label{def:pcp}
	Fix a matrix $\A \in \RB^{n \times d}$ and error parameters $\epsilon_1, \epsilon_2 > 0.$ The matrix $\B \in \RB^{n \times d}$ has the rank $k$ projection-cost preserving property with respect to $\A$ if
	for any rank $k$ orthogonal projection matrices $\Pii \in \RB^{n \times n}$,
	\[
	(1-\epsilon_1) \big\| \A - \Pii \A \big\|_F^2
	\; \leq \; \big\| \B - \Pii \B \big\|_F^2 + \alpha
	\; \leq \; (1+\epsilon_2) \big\| \A - \Pii \A \big\|_F^2 ,
	\]
	for some fixed non-negative $\alpha$ that can depend on $\A$ and $\B$ but is independent of $\Pii.$
\end{defn}

The next lemma provides a way to construct a projection-cost preserving sketch of $\A.$ It states that if a
rank-$s$ matrix $\tilde{\A}_s$ in the span of $\A$ is constructed so that $\tilde{\A}_s \tilde{\A}_s^T$ approximates $\A\A^T$ sufficiently well
in terms of the trace norm, then $\tilde{\A}_s$ has the projection-cost preserving property.
Observe that this lemma explicitly relates projection-cost preserving sketches to approximate matrix multiplication~\citep{drineas06fastmonte1,mahoney2011ramdomized}, where the latter is measures with respect to the trace norm~\citep{gittens2013revisiting}.
We provide a proof of Lemma~\ref{lem:proj:nuclear} in Appendix~\ref{sec:proof:proj:nuclear}.

\begin{lemma} [Projection-Cost Preservation via Low-Rank Approximation]
	\label{lem:proj:nuclear}
	Fix an error parameter $\epsilon \in (0, 1)$. Let $\A \in \RB^{n\times d}$ 
	and choose a rank-$s$ matrix $\tilde{\A}_s$ that 
	satisfies the following conditions:
	\begin{enumerate}[nolistsep,label=({\roman*})]
		\item 
		there is an orthogonal projection matrix $\M$ such that
		$\tilde\A_s \tilde{\A}_s^T = \A \M \A^T$, and
		\item
		$ \| \A \A^T - \tilde\A_s \tilde{\A}_s^T \|_*
		\leq (1+\epsilon ) \big\| \A \A^T - \A_s \A_s^T \big\|_*$.
	\end{enumerate}
	Then there exists a fixed $\alpha \geq 0$ such that, for any rank $k$ projection matrices $\Pii \in \RB^{n \times n}$,
	\[
	\big\| (\I_n - \Pii ) \A \big\|_F^2
	\; \leq \; \big\| (\I_n - \Pii )  \tilde{\A}_s \big\|_F^2 + \alpha
	\; \leq \; \big(1+ \epsilon + \tfrac{k}{s} \big) \big\| (\I_n - \Pii ) \A \big\|_F^2 .
	\]
\end{lemma}

\subsection{Linear $k$-Means Clustering and Projection-Cost Preservation} \label{sec:proof:kmeans}

Linear $k$-means clustering is formally defined by the optimization problem \eqref{eq:kmeans}.
To relate the linear $k$-means clustering problem to the projection-cost preservation property, we use an equivalent formulation of \eqref{eq:kmeans} as an optimization over the set of rank-$k$ projection matrices of the form $\X\X^T,$ where $\X$ is a \emph{cluster indicator matrix}. 
This approach was adopted 
by~\citet{boutisdis2009unsupervised,boutsidis2010random,boutsidis2015randomized,cohen2015dimensionality,ding2005equivalence}.
Following them, we define the cluster indicator matrix in the following and give an example in Figure~\ref{fig:kmeans_X}.

\begin{defn} [Cluster Indicator Matrices] \label{def:cluster}
	Let $n$ and $k$ be given, and let $\JM_1, \ldots , \JM_k$ be a $k$-partition of the set $[n]$.
	The cluster indicator matrix corresponding to $\JM_1, \ldots, \JM_k$ is the $n \times k$ matrix $\X$
	defined by
	\begin{eqnarray*}
		x_{ij} & = &
		\left\{ 
		\begin{array}{cc}
			\frac{1}{\sqrt{ | \JM_j |}} & \textrm{if } \, i \in \JM_j ; \\
			0 & \textrm{otherwise}.
		\end{array} 
		\right.
	\end{eqnarray*}
	We take $\XM_{n,k}$ to be the collection of cluster indicator matrices corresponding to all the $k$-partitions of $[n]$.
\end{defn}

\begin{figure}[h]
	\begin{center}
		\includegraphics[width=0.9\textwidth]{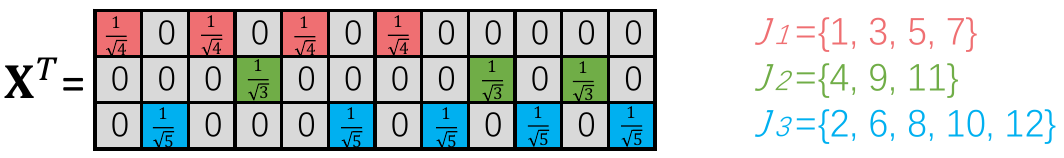}
	\end{center}
	\caption{An example of a cluster indicator matrix $\X \in \XM_{n,k}$.
		In this example, $n=12$, $k=3$, and $\X$ corresponds to the indicated 3-partition
		\textcolor{OrangeRed}{$\JM_1$}, \textcolor{ForestGreen}{$\JM_2$}, \textcolor{ProcessBlue}{$\JM_3$}.}
	\label{fig:kmeans_X}
\end{figure}

A cluster indicator matrix $\X$ has exactly one non-zero entry in each row and has orthonormal columns.
That is, because of the normalization of the non-zero entries, a cluster indicator matrix $\X$ is an orthogonal matrix, but one with the additional constraint that there is only one non-zero entry in each row.
(Indeed, the insight of~\citet{boutisdis2009unsupervised,boutsidis2010random,boutsidis2015randomized} was that the rank-constrained optimization over all orthonormal matrices, which provides the usual PCA/SVD-based low-rank approximation, is a relaxation of the optimization over this smaller set of orthonormal matrices.)
Among other things, it follows that $\X\X^T$ is a rank-$k$ projection matrix.

Given this, assume $i \in \JM_j$;
it can be verified that the $i$-th row of $\X \X^T \A \in \RB^{n\times d}$ is the centroid of $\JM_j$:
\[
(\X \X^T \A)_{i:}
\; = \; \x_{i:} \X^T \A
\; = \; \frac{1}{|\JM_j |} \sum_{l \in \JM_j} \a_{l:}.
\]
Consider the objective function of~\eqref{eq:kmeans}. When $\X$ is the cluster indicator matrix 
corresponding to the input $k$-partition, this objective can be rewritten in terms of $\X\X^T$ as 
\begin{align*}
&\sum_{j=1}^k \sum_{i\in \JM_j} \bigg\| \a_{i:} \, - \, \frac{1}{|\JM_j |} \sum_{l \in \JM_j} \a_{l:}  \bigg\|_2^2 
\; = \; \sum_{j=1}^k \sum_{i\in \JM_j} \big\| \a_{i:} - (\X \X^T \A)_{i:} \big\|_2^2 \\
& = \; \sum_{i=1}^n \big\| \a_{i:} - (\X \X^T \A)_{i:} \big\|_2^2
\; = \; \big\| \A - \X \X^T \A \big\|_F^2  .
\end{align*}
Therefore, the linear $k$-means clustering problem \eqref{eq:kmeans} is equivalent to the optimization problem
\begin{eqnarray} \label{eq:kmeans2}
\argmin_{\X \in \XM_{n,k}} \frac{1}{n} \big\| \A - \X \X^T \A \big\|_F^2.
\end{eqnarray}
In the remainder of this section,
we use this equivalence between $k$-partitionings of $[n]$ and cluster indicator matrices without
comment; e.g., we refer to the outputs of linear $k$-means clustering algorithms as cluster indicator matrices.
Because the linear $k$-means clustering problem \eqref{eq:kmeans2} is NP-hard,
in practice it is solved by variants of Lloyd's algorithm.

Definition~\ref{def:gamma} defines an appropriate notion of approximation for linear $k$-means clustering algorithms.
Definition~\ref{def:gamma} is an equivalent statement of Definition~\ref{def:gamma0}.

\begin{defn} [$\gamma$-Approximate Algorithms] \label{def:gamma}
	We call an algorithm $\AM_\gamma$ ($\gamma \geq 1$) 
	a {$\gamma$-approximate linear $k$-means algorithm} if it produces a cluster indicator matrix $\tilde\X$ such that
	\[
	\big\| \A - \tilde\X \tilde\X^T \A \big\|_F^2
	\; \leq \; \gamma \cdot \min_{\X \in \XM_{n,k}} \big\| \A - \X \X^T \A \big\|_F^2
	\]
	for any conformal matrix $\A$.
\end{defn}

Given this, we now state Lemma~\ref{lem:PCP_kmeans}, which relates projection-cost preservation to linear $k$-means clustering.
Let $\A \in \RB^{n\times d}$ be the input matrix
and $\B \in \RB^{n\times s}$ be a smaller matrix ($s < d$) satisfying the projection-cost preservation property.
The projection-cost preserving property ensures that the $k$-partition obtained by applying linear $k$-means clustering to the rows of $\B$, instead of $\A$,
is a good $k$-partition for the rows of $\A.$
This lemma was established by \citet{cohen2015dimensionality}.
To make this paper self-contained, we provide a proof of Lemma~\ref{lem:PCP_kmeans} in Appendix~\ref{sec:proof:PCP_kmeans}.

\begin{lemma} [Projection-Cost Preservation for Linear $k$-Means] 
	\label{lem:PCP_kmeans}
	Fix $\A \in \RB^{n\times d}$ and assume $\B \in \RB^{n\times s}$ satisfies the 
	rank $k$ projection-cost preservation property in Definition~\ref{def:pcp}
	with probability at least $1-\delta$.
	Let the error parameters $\epsilon_1 , \epsilon_2 \in (0, 1)$ be as in Definition~\ref{def:pcp}.
	If $\tilde\X_\B$ is the result of applying a $\gamma$-approximate linear $k$-means clustering algorithm to the rows of $\B,$
	then with probability at least $1-2 \delta$,
	\[
	\big\| \A - \tilde\X_\B \tilde\X_\B^T \A \big\|_F^2
	\; \leq \; \gamma \cdot 
	\tfrac{1+\epsilon_2}{1-\epsilon_1} \cdot \min_{\X \in \XM_{n,k}} \big\| \A - \X \X^T \A \big\|_F^2 .
	\]
\end{lemma}


\subsection{Kernel $k$-Means Clustering} \label{sec:proof:kernel}

Given input vectors $\a_1 , \ldots , \a_n$, the kernel $k$-means clustering algorithm applies 
linear $k$-means clustering to feature vectors $\ph (\a_1 ) , \ldots , \ph (\a_n )$.
For convenience, we let $\ph (\a_1 )^T , \ldots , \ph (\a_n )^T$ constitute the rows of the matrix $\Ph$.
Lemma~\ref{lem:kernel_kmeans} argues that all the information in $\Ph$ relevant to the kernel
$k$-means clustering problem is present in the kernel matrix $\K = \Ph \Ph^T.$
Variants of this lemma are well-known~\citep{scholkopf2002learning}.

\begin{lemma} [Kernel Trick for Kernel $k$-Means]
	\label{lem:kernel_kmeans}
	Let $\Ph$ be a matrix with $n$ rows and let $\K = \Ph \Ph^T \in \RB^{n\times n}$.
	Let $\K = \V \Lam \V^T$ be the EVD of $\K$.
	Then for any matrix $\X$ with orthonormal columns,
	\begin{eqnarray*}
		\big\| \Ph - \X \X^T \Ph \big\|_F^2 
		& = &\big\| (\V \Lam^{1/2}) - \X \X^T (\V \Lam^{1/2}) \big\|_F^2 .
	\end{eqnarray*}
\end{lemma}

\begin{proof}
	Because $\X$ has orthonormal columns, $\X \X^T$ is an orthogonal projection matrix,
	and $(\I_n - \X \X^T)$ is the orthogonal projection onto the complementary space.
	It follows that, as claimed,
	\begin{align*} 
	&\big\| \Ph - \X \X^T \Ph \big\|_F^2 
	\; = \;  \tr \Big( \Ph^T (\I_n - \X \X^T) \Ph \Big)
	\; = \; \tr \Big( (\I_n - \X \X^T) \K \Big) \\
	& = \;\tr \Big( (\I_n - \X \X^T) \V \Lam^{1/2} \Lam^{1/2} \V (\I_n - \X \X^T) \Big)
	\; = \; \big\| (\V \Lam^{1/2}) - \X \X^T (\V \Lam^{1/2}) \big\|_F^2.
	\end{align*}
	The equalities are justified by, respectively, the definition of the squared Frobenius norm and the idempotence of projections;
	the cyclicity of the trace and the fact that $\K = \Ph \Ph^T$; the cylicity of the trace and the EVD of $\mathbf{K}$; and the definition of the squared Frobenius norm.
\end{proof}

\subsection{Completing the Proof of Theorem~\ref{thm:kkmeans:nystrom2}}
\label{sec:proof:main}

To complete the proof of Theorem~\ref{thm:kkmeans:nystrom2}, we first establish Lemma~\ref{lem:proof:main}, and then use it to prove Theorem~\ref{thm:proof:kkmeans:nystrom}, which is equivalent to Theorem~\ref{thm:kkmeans:nystrom2}.
Let $\Ph$ be a matrix with $n$ rows, and $\K = \Ph \Ph^T \in \RB^{n \times n}$ be the corresponding kernel matrix.

\begin{lemma} \label{lem:proof:main}
	Assume that $\B \in \RB^{n\times s}$ satisfies 
	\begin{eqnarray} \label{lem:proof:main:1}
	\tr \big( \K - \B \B^T \big)
	\; \leq \; (1+\epsilon ) \big\| \K - \K_s \big\|_* 
	\end{eqnarray}
	and that there is an orthogonal projection matrix $\M$ such that 
	\begin{eqnarray} \label{lem:proof:main:2}
	\B \B^T 
	\; = \; \K^{1/2} \M \K^{1/2} .
	\end{eqnarray}
	Use the rows of $\B$ as input to a $\gamma$-approximation algorithm (see Definition~\ref{def:gamma})
	and let $\tilde{\X}_\B$ be the output cluster indicator matrix (see Definition~\ref{def:cluster}), then
	\begin{eqnarray*}
		\big\| \Ph - \tilde\X_\B \tilde\X_\B^T \Ph \big\|_F^2
		\; \leq \; \gamma \big(1+ \epsilon + \tfrac{k}{s} \big) \cdot \min_{\X \in \XM_{n,k}} 
		\big\| \Ph - \X \X^T \Ph \big\|_F^2 .
	\end{eqnarray*}
\end{lemma}

\begin{proof}
	Because assumptions \eqref{lem:proof:main:1} and \eqref{lem:proof:main:2} are satisfied,
	Lemma~\ref{lem:proj:nuclear} ensures that $\B$ is a rank $k$ projection-cost preserving sketch of $\K$. That is, there exists 
	a constant $\alpha \geq 0$ independent of $\Pii$ such that
	\begin{eqnarray*}
		\big\| (\I_n - \Pii ) \K^{1/2} \big\|_F^2
		\; \leq \; \big\| (\I_n - \Pii )  \B \big\|_F^2 + \alpha
		\; \leq \; \big(1+ \epsilon + \tfrac{k}{s} \big) \big\| (\I_n - \Pii ) \K^{1/2} \big\|_F^2 
	\end{eqnarray*}
	holds for any rank-$k$ orthogonal projection matrix $\Pii$.
	It follows by Lemma~\ref{lem:PCP_kmeans} that $\tilde\X_\B$ gives an almost optimal linear $k$-means clustering over the rows of $\K^{1/2}$:
	\begin{eqnarray*}
		\big\| \K^{1/2} - \tilde\X_\B \tilde\X_\B^T \K^{1/2} \big\|_F^2
		\; \leq \; \gamma \big(1+ \epsilon + \tfrac{k}{s} \big) \cdot \min_{\X \in \XM_{n,k}} 
		\big\| \K^{1/2} - \X \X^T \K^{1/2} \big\|_F^2 .
	\end{eqnarray*}
	Finally, Lemma~\ref{lem:kernel_kmeans} states that clustering the rows of $\K^{1/2}$ is equivalent to clustering the rows of $\Ph$, so we reach
	the desired conclusion that
	\begin{eqnarray*}
		\big\| \Ph - \tilde\X_\B \tilde\X_\B^T \Ph \big\|_F^2
		\; \leq \; \gamma \big(1+ \epsilon + \tfrac{k}{s} \big) \cdot \min_{\X \in \XM_{n,k}} 
		\big\| \Ph - \X \X^T \Ph \big\|_F^2.
	\end{eqnarray*}
\end{proof}

Finally, we state and prove Theorem~\ref{thm:proof:kkmeans:nystrom}. 
Theorem~\ref{thm:proof:kkmeans:nystrom} shows that approximate kernel $k$-means clustering using the Nystr\"om method exhibits a $1+\epsilon+ \tfrac{k}{s}$ approximation ratio.
Observe that Theorem~\ref{thm:proof:kkmeans:nystrom} is equivalent to Theorem~\ref{thm:kkmeans:nystrom2}, and thus establishing it also establishes Theorem~\ref{thm:kkmeans:nystrom2}.
Recall the Nystr\"om approximation is $\K \approx \C \W^\dag \C^T$,
where $\C = \K \PP$, $\W = \PP^T \K \PP$, and $\PP$ is a sketching matrix.

\begin{theorem} \label{thm:proof:kkmeans:nystrom}
	Choose a sketching sketching matrix $\PP \in \RB^{n\times c}$ and sketch size $c$ consistent with Table~\ref{tab:sketch}.
	Let $\B \in \RB^{n\times s}$ be any matrix satisfying $\B \B^T = (\C \W^\dag \C^T )_s$.
	Let the cluster indicator matrix $\X_\B$ be the output of any $\gamma$-approximate $k$-means clustering algorithm applied to the rows of $\B$. 
	With probability at least $0.9$, 
	\[
	\big\| \Ph - \X_\B \X_\B^T \Ph \big\|_F^2
	\; \leq \; \gamma \big( 1 + \epsilon + \tfrac{k}{s} \big) 
	\min_{\X \in \XM_{n,k}} \big\| \Ph - \X \X^T \Ph \big\|_F^2.
	\]
\end{theorem}

\begin{proof}
	Theorem~\ref{thm:nystrom} shows that \eqref{lem:proof:main:1} holds with probability at least $0.9$.
	Theorem~\ref{thm:nystrom} also shows that $\B \B^T = \K^{1/2} \Q \Q^T \K^{1/2}$ 
	where $\Q$ is a $n\times s$ matrix with orthonormal columns, so~\eqref{lem:proof:main:2} holds surely.
	The theorem now follows from Lemma~\ref{lem:proof:main}.
\end{proof}

\section{Proof of Theorem~\ref{thm:power}} \label{sec:proof:power}

Section~\ref{sec:proof:power:1} analyzes rank-restricted Nyst\"om with power method.
Section~\ref{sec:proof:power:2} completes the proof of Theorem~\ref{thm:power}.
Here we re-state the algorithm of Theorem~\ref{thm:power}:
first, set any $c \geq s$ and draw a Gaussian projection matrix $\PP \in \RB^{n\times c}$;
second, run the power iteration $\PP \longleftarrow \K \PP$ for $t$ times;
third, orthogonalize $\PP \in \RB^{n\times c}$ to obtain $\U \in \RB^{n\times c}$;
fourth, compute $\C = \K \U $ and $\W = \U^T \K \U $;
last, find any $\B \in \RB^{n\times s}$ satisfying $\B \B^T = (\C \W^\dag \C^T )_s$.

\subsection{Rank-Restricted Nystr\"om with Power Method} \label{sec:proof:power:1}

Theorem~\ref{thm:nystrom:power} analyzes the rank-restricted Nystr\"om with power method.
The theorem will be used to prove Theorem~\ref{thm:power}.

\begin{theorem} \label{thm:nystrom:power}
	Let $\K \in \RB^{n\times n}$ be an SPSD matrix, $s$ be the target rank, 
	$c $ ($\geq s$) be arbitrary,
	and $\epsilon \in (0,1)$ be an approximation parameter.
	Let $\C \in \RB^{n\times c}$ and $\W \in \RB^{c\times c}$ be computed by the above algorithm with 
	$t = \OM (\frac{ \log (n / \epsilon ) }{ \log ( \sigma_{s} / \sigma_{s+1} ) })$
	power iterations, where $\sigma_i$ is the $i$-th singular value of $\K$.
	If $c = s + \OM (\log \frac{1}{\delta})$, then 
	\begin{eqnarray*}
		\big\| \K - (\C \W^\dag \C^T)_s \big\|_*
		&  \leq & (1+ \epsilon)  \,
		\big\| \K - \K_s \big\|_* 
	\end{eqnarray*}
	with probability at least $1-\delta$.
	If $c=s$, then this bound holds with a constant probability (that depends on $s$).
\end{theorem}

\begin{proof}
	Let $\V_{-s} \in \RB^{n\times (n-s)}$ be the orthogonal complement of $\V_s$.
	Let $\PP $ be an $n\times c$ standard Gaussian matrix.
	Then $\V_{-s}^T \PP$ is $(n-s)\times c$ standard Gaussian matrix,
	and $\V_{s}^T \PP$ is $s\times c$ standard Gaussian matrix.

	It is well known that the largest singular value of an $(n-s)\times c$ standard Gaussian matrix
	is at most $ \sqrt{n-s} + \sqrt{c} + \zeta$ with probability $1-\exp (-\zeta^2/2)$.
	See \citep{vershynin2010introduction}.
	
	Consider $c > s$.
	The least singular value of an  $s\times c$ standard Gaussian matrix is at least
	$\sqrt{c} - \sqrt{s} - \zeta$ with probability $1-\exp (-\zeta^2/2)$.
	Combining the above results, we have that
	\begin{eqnarray*}
		\frac{ \sigma_1^2 ( \V_{-s}^T \PP ) }{ \sigma_{s}^2 ( \V_s^T \PP ) }
		\; \leq \; \frac{ \sqrt{n-s} + \sqrt{c} + \sqrt{2 \log (1/\delta)} }{ \sqrt{c} - \sqrt{s} - \sqrt{2 \log (1/\delta)}  }
	\end{eqnarray*}
	holds with probability $1 - 2\delta$.
	
	Consider $c = s$.
	\citet{rudelson2008littlewood,tao2010random} showed that
	the least singular value of an $s\times s$ standard Gaussian matrix $\G$ satisfy
	\[
	\PB \big\{ \sigma_s (\G) \leq \tfrac{\delta_1}{\sqrt{s}} \big\}
	\; = \; \delta_1 + \OM (s^{- \tau}) ,
	\]
	where $\tau > 0$ is a constant.
	It follows that
	\begin{eqnarray*}
		\frac{ \sigma_1^2 ( \V_{-s}^T \PP ) }{ \sigma_{s}^2 ( \V_s^T \PP ) }
		\; \leq \; \frac{ \sqrt{n-s} + \sqrt{s} + \sqrt{2 \log (1/\delta_2)} }{ \delta_1 / \sqrt{s} }
	\end{eqnarray*}
	holds with probability $1 - \delta_1- \delta_2 - \OM (s^{- \tau})$.
	
	For either $c > s$ or $c = s$,
	it follows from Lemma~\ref{lem:nystrom:power:proof}, which will be proved subsequently, that 
	\begin{eqnarray*}
		\big\| \K - ( \C  \W^\dag \C^T )_s \big\|_* 
		& \leq & \big\| \K - \K_s \big\|_*
		+ \tfrac{ \sigma_1^2 ( \V_{-s}^T \PP ) }{ \sigma_{s}^2 ( \V_s^T \PP ) }
		\, \Big(\tfrac{\sigma_{s+1} (\K ) }{\sigma_{s} (\K ) } \Big)^{2t}  
		\,\big\| \K - \K_s \big\|_* \\
		& \leq & \big\| \K - \K_s \big\|_*
		+ \OM \big( n s \big) \, \Big(\tfrac{\sigma_{s+1} (\K ) }{\sigma_{s} (\K ) } \Big)^{2t}  
		\,\big\| \K - \K_s \big\|_* 
	\end{eqnarray*}
	holds with constant probability, by which the theorem follows.
\end{proof}

\begin{lemma} \label{lem:nystrom:power:proof}
	Let $\K \in \RB^{n\times n}$ be any fixed SPSD matrix,
	$s$ and $t$ be any positive integer,
	and $\V_s \in \RB^{n\times s}$ be the top $s$ singular vectors of $\K$.
	Let $\PP \in \RB^{n\times c}$ satisfy that $\PP^T \V_s \in \RB^{c\times s}$ has full column rank.
	Let $\U \in \RB^{n\times c}$ be the orthonormal bases of $\K^t \PP$.
	Let $\C = \K \U$ and $\W = \U^T \K \U$.
	Then
	\begin{align*}
	\big\| \K - ( \C  \W^\dag \C^T )_s \big\|_* 
	\; \leq \; \big\| \K - \K_s \big\|_*
	+ \Big(\tfrac{\sigma_{s+1} (\K ) }{\sigma_{s} (\K ) } \Big)^{2t}  
	\, \tfrac{ \sigma_1^2 ( \V_{-s}^T \PP ) }{ \sigma_{s}^2 ( \V_s^T \PP ) }
	\,\big\| \K - \K_s \big\|_* .
	\end{align*}
	Here $\V_{-s} \in \RB^{n\times (n-s)}$ is the orthogonal complement of $\V_s$.
\end{lemma}

\begin{proof}
	It follows from Lemma~\ref{lem:proof:nystrom_decompose} that
	\begin{eqnarray*} 
		\big\| \K - \big( \C  \W^\dag \C^T \big)_s \big\|_*
		& = & \min_{\rk (\Z) \leq s} \big\|  \K^{1/2} - (\K^{1/2} \U) \Z \big\|_F^2 \\
		& = & \min_{\rk (\Z) \leq s} \big\|  \A - (\A \U) \Z \big\|_F^2
	\end{eqnarray*}
	where we let $\A = \K^{1/2}$. 
	By definition, $\U $ contains the orthonormal bases of $\K^t \PP =\A^{2t} \PP$,
	and thus there exists an $s\times s$ non-singular matrix $\R$ such that
	$\A^{2t} \PP = \U \R$.
	It follows that
	\begin{eqnarray*} 
		\big\| \K - \big( \C  \W^\dag \C^T \big)_s \big\|_*
		& = & \min_{\rk (\Z) \leq s} \big\|  \A - (\A \A^{2t} \PP \R^{-1}) \Z \big\|_F^2 \\
		& = & \min_{\rk (\Y) \leq s} \big\|  \A - (\A^{2t+1} \PP) \Y \big\|_F^2 ,
	\end{eqnarray*}
	where the latter equality follows from that the column spaces of
	$\A^{2t+1} \PP$ and $\A^{2t+1} \PP \R^{-1}$ are the same.
	It follows from Lemma~\ref{lem:proof:power} that
	\begin{align*}
	& \big\| \K - \big( \C  \W^\dag \C^T \big)_s \big\|_*
	\; = \; \min_{\rk (\Y) \leq s} \big\| \A - (\A^{2t+1} \PP) \Y \big\|_F^2 \\
	& \leq \; \big\| \A - \A_s \big\|_F^2
	+ \Big(\tfrac{\sigma_{s+1}^2 (\A) }{\sigma_{s}^2 (\A) } \Big)^{2t}  
	\, \Big\| (\A - \A_s) \PP ( \V_s^T \PP )^\dag   \Big\|_F^2  \\
	& \leq \; \big\| \A - \A_s \big\|_F^2
	+ \Big(\tfrac{\sigma_{s+1}^2 (\A) }{\sigma_{s}^2 (\A) } \Big)^{2t}  
	\, \big\| \A - \A_s \big\|_F^2 \big\| \V_{-s}^T \PP \big\|_2^2 \big\| ( \V_s^T \PP )^\dag \big\|_2^2 ,
	\end{align*}
	where we define $\V_{-s} \in \RB^{n\times (n-s)}$ as the orthogonal complement of $\V_s$.
	By our definition $\A = \K^{1/2}$, it follows that
	\begin{align*}
	& \big\| \K - \big( \C  \W^\dag \C^T \big)_s \big\|_* \\
	& \leq \; \big\| (\K - \K_s)^{1/2} \big\|_F^2
	+ \Big(\tfrac{\sigma_{s+1} (\K ) }{\sigma_{s} (\K ) } \Big)^{2t} 
	\, \tfrac{ \sigma_1^2 ( \V_{-s}^T \PP ) }{ \sigma_{s}^2 ( \V_s^T \PP ) }
	\,\big\| (\K - \K_s)^{1/2}\big\|_F^2 \\
	& = \; \big\| \K - \K_s \big\|_*
	+ \Big(\tfrac{\sigma_{s+1} (\K ) }{\sigma_{s} (\K ) } \Big)^{2t}  
	\, \tfrac{ \sigma_1^2 ( \V_{-s}^T \PP ) }{ \sigma_{s}^2 ( \V_s^T \PP ) }
	\,\big\| \K - \K_s \big\|_* ,
	\end{align*}
	by which the lemma follows.
\end{proof}

\subsection{Completing the Proof of Theorem~\ref{thm:power}} \label{sec:proof:power:2}

Finally, we state and prove Theorem~\ref{thm:proof:kkmeans:power}. 
Observe that Theorem~\ref{thm:proof:kkmeans:power} is equivalent to Theorem~\ref{thm:power}, 
and thus establishing it also establishes Theorem~\ref{thm:power}.

\begin{theorem} \label{thm:proof:kkmeans:power}
	Let $c$, $t$, and $\B \in \RB^{n\times s}$ be defined in the begining of this section.
	Set $t = \OM (\frac{ \log (n / \epsilon ) }{ \log ( \sigma_{s} / \sigma_{s+1} ) })$.
	Let the cluster indicator matrix $\X_\B$ be the output of any $\gamma$-approximate $k$-means clustering algorithm applied to the rows of $\B$. 
	If $c = s + \OM (\log \frac{1}{\delta}) $, then
	\[
	\big\| \Ph - \X_\B \X_\B^T \Ph \big\|_F^2
	\; \leq \; \gamma \big( 1 + \epsilon + \tfrac{k}{s} \big) 
	\min_{\X \in \XM_{n,k}} \big\| \Ph - \X \X^T \Ph \big\|_F^2
	\]
	holds with probability at least $1-\delta$.
	If $c = s$, then the above inequality holds with probability
	$0.9 - \OM (s^{- \tau})$, where $\tau$ is a constant.
\end{theorem}

\begin{proof}
	Since $\B \B^T = (\C \W^\dag \C^T)_s$ by definition,
	Theorem~\ref{thm:nystrom:power} ensures that \eqref{lem:proof:main:1} holds with some constant probability.
	Theorem~\ref{thm:nystrom} shows that $\B \B^T = \K^{1/2} \Q \Q^T \K^{1/2}$ 
	where $\Q$ is a $n\times s$ matrix with orthonormal columns, so~\eqref{lem:proof:main:2} holds surely.
	The theorem now follows from Lemma~\ref{lem:proof:main}.
\end{proof}

\section{Proof of Lemma~\ref{lem:proj:nuclear}} \label{sec:proof:proj:nuclear}

To prove Lemma~\ref{lem:proj:nuclear}, we first establish a key lemma, Lemma~\ref{lem:proj:nuclear2}.
This lemma will make use of the following two assumptions.

\begin{assumption} \label{assumption:proj:nuclear1}
	Assume that $\tilde{\A}_s$ satisfies
	$ \tr \big( \A \A^T - \tilde\A_s \tilde{\A}_s^T \big)
	\leq (1+\epsilon ) \big\| \A \A^T - \A_s \A_s^T \big\|_*$.
\end{assumption}

\begin{assumption} \label{assumption:proj:nuclear2}
	Assume there exists such an orthogonal projection matrix $\M$ that
	$\tilde\A_s \tilde{\A}_s^T = \A \M \A^T$.
\end{assumption}

\begin{lemma}\label{lem:proj:nuclear2}
	Let $\A \in \RB^{n\times d}$ be any fixed matrix.
	Fix an error parameter $\epsilon \in (0, 1)$.
	Let the rank $s$ matrix $\tilde{\A}_s$ satisfy 
	Assumptions~\ref{assumption:proj:nuclear1} and \ref{assumption:proj:nuclear2}.
	Then for any rank $k$ orthogonal projection matrix $\Pii \in \RB^{n\times n}$,
	\begin{align*} 
	& \tr \Big( \Pii \big(\A \A^T - \tilde\A_s \tilde\A_s^T \big) \Pii \Big)
	\; \leq \; \big( \epsilon + \tfrac{k}{s} \big)
	\big\| \A - \Pii \A \big\|_F^2 .
	\end{align*}
\end{lemma}

\begin{proof}
	It holds that
	\begin{align*}
	& \tr \Big( (\I_n - \Pii ) \big( \A \A^T - \tilde\A_s \tilde\A_s^T \big) \Big) \\
	& = \; \tr \Big(  \big( \A \A^T - \tilde\A_s \tilde\A_s^T \big)
	- \Pii \big(\A \A^T - \tilde\A_s \tilde\A_s^T \big)    \Big)  \\
	& = \; \tr \big( \A \A^T - \tilde\A_s \tilde\A_s^T \big)
	- \tr \big( \Pii \big(\A \A^T - \tilde\A_s \tilde\A_s^T \big) \Pii \big) \\
	& \leq \; (1+\epsilon ) \big\| \A \A^T - \A_s \A_s^T \big\|_*
	- \tr \big( \Pii \big(\A \A^T - \tilde\A_s \tilde\A_s^T \big) \Pii \big) ,
	\end{align*}
	where the inequality follows from Assumption~\ref{assumption:proj:nuclear1}.
	It can be equivalently written as
	\begin{align*} 
	& \tr \Big( \Pii \big(\A \A^T - \tilde\A_s \tilde\A_s^T \big) \Pii \Big) \nonumber \\
	& \leq \; (1+\epsilon ) \big\| \A \A^T - \A_s \A_s^T \big\|_*
	-  \tr \Big( (\I_n - \Pii ) \big( \A \A^T - \tilde\A_s \tilde\A_s^T \big) (\I_n - \Pii ) \Big) .
	\end{align*}
	Let $\M$ be an orthogonal projection matrix defined in Assumption~\ref{assumption:proj:nuclear2};
	by this assumption, it holds that $\tilde\A_s \tilde\A_s^T = \A \M \A^T$
	and $\rk (\M \A) = s$.
	It follows that
	\begin{align*}
	&\tr \Big( (\I_n - \Pii ) \big( \A \A^T - \tilde\A_s \tilde\A_s^T \big) (\I_n - \Pii ) \Big)
	\; = \; \tr \Big(  (\I_n - \Pii ) \A (\I - \M ) \A^T (\I_n - \Pii )  \Big) \\
	& = \;  \big\| (\I_n - \Pii ) \A (\I - \M ) \big\|_F^2 
	\; = \; \Big\| \A - \underbrace{\Pii \A}_{\textrm{rank } k} 
	- \underbrace{(\I_n - \Pii ) \A \M}_{\textrm{rank } s} \Big\|_F^2 \\
	& \geq \; \min_{\rk (\Y ) \leq s+k}
	\big\| \A - \Y \big\|_F^2
	\; = \; \big\| \A  - \A_{s+k} \big\|_F^2 ,
	\end{align*}
	where the inequality follows from that matrix rank is subadditive function.
	It follows that
	\begin{align*} 
	& \tr \Big( \Pii \big(\A \A^T - \tilde\A_s \tilde\A_s^T \big) \Pii \Big) \nonumber \\
	& \leq \; (1+\epsilon ) \big\| \A \A^T - \A_s \A_s^T \big\|_*
	-  \big\| \A  - \A_{s+k} \big\|_F^2 \\
	& = \; (1+\epsilon ) \big\| \A - \A_s  \big\|_F^2
	-  \big\| \A  - \A_{s+k} \big\|_F^2 \\
	& = \; \epsilon \big\| \A - \A_s  \big\|_F^2
	+ \sum_{i=s+1}^n \sigma_i^2(\A ) - \sum_{i=s+k+1}^n \sigma_i^2 (\A ) \\
	& = \; \epsilon \big\| \A - \A_s \big\|_F^2
	+ \sum_{i=s+1}^{s+k} \sigma_i^2(\A ) \\
	& \leq \; \epsilon \big\| \A - \A_s \big\|_F^2
	+ \frac{k}{s} \sum_{i=k+1}^{s+k} \sigma_i^2(\A )  ,
	\end{align*}
	where the last inequality follows by the decrease of the singular values.
	Finally, we obtain that
	\begin{align*} 
	& \tr \Big( \Pii \big(\A \A^T - \tilde\A_s \tilde\A_s^T \big) \Pii \Big)
	\; \leq \; \epsilon \big\| \A - \A_k \big\|_F^2
	+ \tfrac{k}{s} \big\| \A - \A_k \big\|_F^2 
	\; \leq \; \big( \epsilon + \tfrac{k}{s} \big)
	\big\| \A - \Pii \A \big\|_F^2 ,
	\end{align*}
	by which the lemma follows.
\end{proof}

\noindent
Given Lemma~\ref{lem:proj:nuclear2}, we now provide the proof of Lemma~\ref{lem:proj:nuclear}.
\\

\begin{proof}
	For any orthogonal projection matrix $\M$, it holds that $\M = \M \M$. It follows that
	\begin{align*}
	& \big\| (\I_n - \Pii ) \A \big\|_F^2 - \big\| (\I_n - \Pii )  \tilde\A_s \big\|_F^2
	\; = \;\tr \big( (\I_n - \Pii ) \A \A^T \big)
	- \tr \big( (\I_n - \Pii ) \tilde\A_s \tilde\A_s^T   \big) \\
	& = \; \tr \big( (\I_n - \Pii ) (\A \A^T - \tilde\A_s \tilde\A_s^T )  \big)
	\; = \; \tr \big( \A \A^T - \tilde\A_s \tilde\A_s^T  \big)
	- \tr \big( \Pii  (\A \A^T - \tilde\A_s \tilde\A_s^T )  \big) \\
	& \; = \; \alpha - \tr \big( \Pii  (\A \A^T - \tilde\A_s \tilde\A_s^T )  \big) ,
	\end{align*}
	where the last equality follows by letting 
	$\alpha = \tr \big( \A \A^T - \tilde\A_s \tilde\A_s^T  \big) \geq 0$ 
	which is independent of $\Pii$.
	The above equality can be equivalently written as
	\begin{eqnarray} \label{thm:projection_cost_rsvd:1}
	\big\| (\I_n - \Pii )  \tilde\A_s \big\|_F^2 + \alpha
	& = & \big\| (\I_n - \Pii ) \A \big\|_F^2
	+ \tr \big( \Pii  (\A \A^T - \tilde\A_s \tilde\A_s^T )  \Pii \big)  .
	\end{eqnarray}
	Under Assumption~\ref{assumption:proj:nuclear2} that $\tilde\A_s \tilde\A_s^T \preceq \A \A^T$,
	it follows from \eqref{thm:projection_cost_rsvd:1} that
	\begin{align*}
	&\big\| (\I_n - \Pii ) \A \big\|_F^2 \nonumber \\
	& \leq \; \big\| (\I_n - \Pii )  \tilde\A_s \big\|_F^2 + \alpha \\
	& = \; \big\| (\I_n - \Pii ) \A \big\|_F^2
	+ \tr \big( \Pii  (\A \A^T - \tilde\A_s \tilde\A_s^T )  \Pii \big)  .
	\end{align*}
	Under Assumptions~\ref{assumption:proj:nuclear1} and \ref{assumption:proj:nuclear2},
	we can apply Lemma~\ref{lem:proj:nuclear2} to bound the right-hand side:
	\begin{align*} 
	& \tr \Big( \Pii \big(\A \A^T - \tilde\A_s \tilde\A_s^T \big) \Pii \Big)
	\; \leq \; \big( \epsilon + \tfrac{k}{s} \big)
	\big\| \A - \Pii \A \big\|_F^2 ,
	\end{align*}
	by which the lemma follows.
\end{proof}

\section{Proof of Lemma~\ref{lem:PCP_kmeans}} \label{sec:proof:PCP_kmeans}

Because $\B $ enjoys the projection-cost preservation property,
there exists a constant $\alpha \geq 0$ such that
for any rank $k$ orthogonal projection matrices $\Pii_1$ and $\Pii_2$, the two inequalities
\begin{align}
& (1-\epsilon_1 ) \big\| \A - \Pii_1 \A \big\|_F^2
\; \leq \; \big\| \B - \Pii_1 \B \big\|_F^2 + \alpha , 
\label{eq:lem:PCP_kmeans:1} \\
& \big\| \B - \Pii_2 \B \big\|_F^2 + \alpha
\; \leq \; (1+\epsilon_2 ) \big\| \A - \Pii_2 \A \big\|_F^2 
\label{eq:lem:PCP_kmeans:2}
\end{align}
hold simultaneously with probability at least $1 - 2 \delta$.
Let $\Pii_1 = \tilde\X_\B \tilde\X_\B^T$
and $\X_\A^\star = \argmin_{\X \in \XM_{n,k}} \| \A - \X \X^T \A \|_F^2$.
It follows from \eqref{eq:lem:PCP_kmeans:1} and 
the definition of $\gamma$-approximate $k$-means algorithm that
\begin{align*} 
& (1-\epsilon_1) \big\| \A - \tilde\X_\B \tilde\X_\B^T \A \big\|_F^2
\; \leq \; \big\| \B - \tilde\X_\B \tilde\X_\B^T \B \big\|_F^2 + \alpha \\
& \leq \; \gamma \cdot \min_{\X \in \XM_{n,k}} \big\| \B - \X \X^T \B \big\|_F^2  + \alpha
\; \leq \; \gamma \big\| \B - \X_\A^\star {\X_\A^\star}^T \B \big\|_F^2  + \alpha .
\end{align*}
Let $\Pii_2 = \X_\A^\star {\X_\A^\star}^T$.
It follows from \eqref{eq:lem:PCP_kmeans:2} that
\[
\big\| \B - \X_\A^\star {\X_\A^\star}^T \B \big\|_F^2 + \alpha
\; \leq \; (1+\epsilon_2) \big\| \A - \X_\A^\star {\X_\A^\star}^T \A \big\|_F^2 .
\]
Combining the above results, we have that
\begin{align*}
&(1-\epsilon_1) \big\| \A - \tilde\X_\B \tilde\X_\B^T \A \big\|_F^2
\; \leq \; \gamma \big\| \B - \X_\A^\star {\X_\A^\star}^T \B \big\|_F^2  + \alpha \\
& \leq \; \gamma \big( \big\| \B - \X_\A^\star {\X_\A^\star}^T \B \big\|_F^2  + \alpha \big)
\; \leq \; \gamma (1+\epsilon_2 ) \big\| \A - \X_\A^\star {\X_\A^\star}^T \A \big\|_F^2 ,
\end{align*}
where the second inequality follows from that $\gamma \geq 1$ and $\alpha \geq 0$.
It follows that
\begin{align*}
& \big\| \A - \tilde\X_\B \tilde\X_\B^T \A \big\|_F^2
\; \leq \;  \gamma \tfrac{1+\epsilon_2 }{1-\epsilon_1} \big\| \A - \X_\A^\star {\X_\A^\star}^T \A \big\|_F^2 
\; = \;  \gamma \tfrac{1+\epsilon_2 }{1-\epsilon_1} 
\min_{\X \in \XM_{n,k}} \big\| \A - \X {\X}^T \A \big\|_F^2
\end{align*}
This concludes our proof.

\vskip 0.2in

\bibliography{matrix}

\end{document}